\documentclass[conference]{IEEEtran}
\usepackage{times}

\usepackage[numbers,compress]{natbib}
\usepackage{multicol}
\usepackage[bookmarks=true]{hyperref}
\usepackage{cuted}

\usepackage{titlesec}
\usepackage[all=normal,paragraphs=tight]{savetrees}

\usepackage[utf8]{inputenc} 
\usepackage[T1]{fontenc}    
\usepackage{hyperref}       
\usepackage{nameref}
\usepackage{tocloft}
\usepackage{url}            
\usepackage{booktabs}       
\usepackage{amsfonts}       
\usepackage{nicefrac}       
\usepackage{microtype}      
\usepackage{xcolor}
\usepackage{xcolor-material}
\usepackage{colortbl}                   
\usepackage{multicol}

\usepackage{graphicx} 

\definecolor{commentblue}{HTML}{598be7}

\usepackage{algorithm}
\usepackage[endLComment={},commentColor=commentblue]{algpseudocodex}

\usepackage{amsmath,amssymb}
\usepackage{amsthm}
\usepackage[font=small,skip=0pt]{caption}

\usepackage{dsfont}
\usepackage{soul}
\usepackage{enumitem}
\usepackage{subcaption}
\usepackage{microtype,wrapfig}

\usepackage{dblfloatfix}

\usepackage{mathtools}
\usepackage{framed}
\usepackage{xspace}

\usepackage[short]{optidef}       

\usepackage{crossreftools}

\usepackage{placeins}

\usepackage{multirow}
\selectcolormodel{natural}
\usepackage{ninecolors}
\selectcolormodel{rgb}

\usepackage{etoolbox}

\usepackage{enumitem}
\usepackage{braket}                 

\usepackage{bm}

\usepackage{siunitx}
\usepackage{thm-restate}
\usepackage[framemethod=TikZ]{mdframed}
\usepackage[most]{tcolorbox}
\usepackage{thm-restate}

\usepackage[scale=0.95]{roboto-mono}

\declaretheorem[name=Lemma]{lemma}
\declaretheorem[name=Definition]{defi}

\declaretheorem[name=Corollary]{coroll}
\declaretheorem[name=Remark,style=normalstyle]{rmk}

\definecolor{MaterialBlueGray10}{HTML}{CFD8DC}
\definecolor{MaterialBlueGray900}{HTML}{263238}
\definecolor{MaterialBlue10}{HTML}{E3F2FD}
\definecolor{MaterialBlue900}{HTML}{0D47A1}
\definecolor{SteelMist}{HTML}{F0EFF6}
\definecolor{SteelSlate}{HTML}{6F7582}

\mdfdefinestyle{ThmFrame}{%
    linecolor=white,
    outerlinewidth=0pt,
    roundcorner=2pt,
    leftmargin=-6pt,
    rightmargin=-2pt,
    innertopmargin=1pt,
    innerbottommargin=1pt,
    innerrightmargin=1pt,
    innerleftmargin=0pt,
    backgroundcolor=MaterialBlue10!40
}
\mdfdefinestyle{DefFrame}{%
    linecolor=white,
    outerlinewidth=0pt,
    roundcorner=2pt,
    leftmargin=-6pt,
    rightmargin=-2pt,
    innertopmargin=1pt, %
    innerbottommargin=1pt, %
    innerrightmargin=1pt,
    innerleftmargin=1pt,
    backgroundcolor=SteelMist!40
}
\mdfdefinestyle{ProofFrame}{%
    linecolor=white,
    outerlinewidth=0pt,
    roundcorner=2pt,
    leftmargin=-3pt,
    rightmargin=-2pt,
    innertopmargin=1pt, %
    innerbottommargin=1pt, %
    innerrightmargin=1pt,
    innerleftmargin=1pt,
    backgroundcolor=MaterialBlue10!40
}

\tcbset{
  ThmFrame/.style={
    enhanced,
    breakable,
    colback=MaterialBlue10!20,
    colframe=MaterialBlue900!35,
    boxrule=0pt,
    boxsep=7pt,
    left=4pt,
    top=1pt,
    arc=12pt,
    outer arc=12pt,
    bottom=3pt,
    overlay unbroken={
      \begin{scope}
        \clip[rounded corners=12pt] (frame.south west) rectangle (frame.north east);
        \fill[MaterialBlue900!35]
          (frame.south west) rectangle ([xshift=3pt]frame.north west);
      \end{scope}
    },
    overlay first={
      \begin{scope}
        \clip[rounded corners=12pt] (frame.south west) rectangle (frame.north east);
        \fill[MaterialBlue900!35]
          (frame.south west) rectangle ([xshift=3pt]frame.north west);
      \end{scope}
    },
    overlay middle={
      \fill[MaterialBlue900!35]
        (frame.south west) rectangle ([xshift=3pt]frame.north west);
    },
    overlay last={
      \begin{scope}
        \clip[rounded corners=12pt] (frame.south west) rectangle (frame.north east);
        \fill[MaterialBlue900!35]
          (frame.south west) rectangle ([xshift=3pt]frame.north west);
      \end{scope}
    },
  }
}
\tcbset{
  DefFrame/.style={
    enhanced,
    breakable,
    colback=SteelMist!35,
    colframe=SteelSlate!35,
    boxrule=0pt,
    boxsep=7pt,
    left=4pt,
    top=1pt,
    right=3pt,
    arc=12pt,
    outer arc=12pt,
    bottom=3pt,
    overlay unbroken={
      \begin{scope}
        \clip[rounded corners=12pt] (frame.south west) rectangle (frame.north east);
        \fill[SteelSlate!35]
          (frame.south west) rectangle ([xshift=3pt]frame.north west);
      \end{scope}
    },
    overlay first={
      \begin{scope}
        \clip[rounded corners=12pt] (frame.south west) rectangle (frame.north east);
        \fill[SteelSlate!35]
          (frame.south west) rectangle ([xshift=3pt]frame.north west);
      \end{scope}
    },
    overlay middle={
      \fill[SteelSlate!35]
        (frame.south west) rectangle ([xshift=3pt]frame.north west);
    },
    overlay last={
      \begin{scope}
        \clip[rounded corners=12pt] (frame.south west) rectangle (frame.north east);
        \fill[SteelSlate!35]
          (frame.south west) rectangle ([xshift=3pt]frame.north west);
      \end{scope}
    },
  }
}
\tcbset{
  ProofFrame/.style={
    enhanced,
    breakable,
    colback=MaterialBlue10!20,
    colframe=MaterialBlue900!55,
    boxrule=0pt,
    boxsep=7pt,
    top=1pt,
    arc=1pt,
    outer arc=1pt,
    bottom=3pt,
    sharp corners=all
  }
}

\DeclareMathOperator*{\argmax}{arg\,max}

\newcommand{\lf}{\left}
\newcommand{\rg}{\right}

\definecolor{colorPredA}{rgb}{0.365,0.494,0.71}
\definecolor{colorPredB}{rgb}{0.333,0.6588,0.4078}
\definecolor{colorPredV}{HTML}{FFAB40}

\newcommand{\predA}[1]{\textcolor[rgb]{0.365,0.494,0.71}{\mathsf{#1}}}
\newcommand{\predB}[1]{\textcolor[rgb]{0.333,0.6588,0.4078}{\mathsf{#1}}}
\newcommand{\predV}[1]{\textcolor[HTML]{FFAB40}{\mathsf{#1}}}

\usepackage{cleveref}

\crefname{thm}{thm.}{thms.}
\Crefname{thm}{Thm.}{Thms.}
\crefname{figure}{fig.}{figs.}
\Crefname{figure}{Fig.}{Figs.}

\title{Bellman Value Decomposition for Task Logic \\ in Safe Optimal Control}

\author{%
  \begin{tabular}[t]{c}
    {William~Sharpless$^{*,1}$,\: Oswin~So$^{*,2}$,\: Dylan~Hirsch$^1$,\: Sylvia~Herbert$^1$,\: Chuchu~Fan$^2$}\\
    \normalfont{$^{1}$UCSD,\;\; $^{2}$MIT,\;\; $^*$Equal contribution,\;\; $\texttt{wsharpless@ucsd.edu}$}
  \end{tabular}
}

\date{November 2025}

\begin{document}
\maketitle

\thispagestyle{plain}
\pagestyle{plain}

\begin{strip}
\begin{center}
    \vspace{-4.2em}
    \includegraphics[width=\linewidth, trim=0pt 7pt 0pt -5pt
    ]{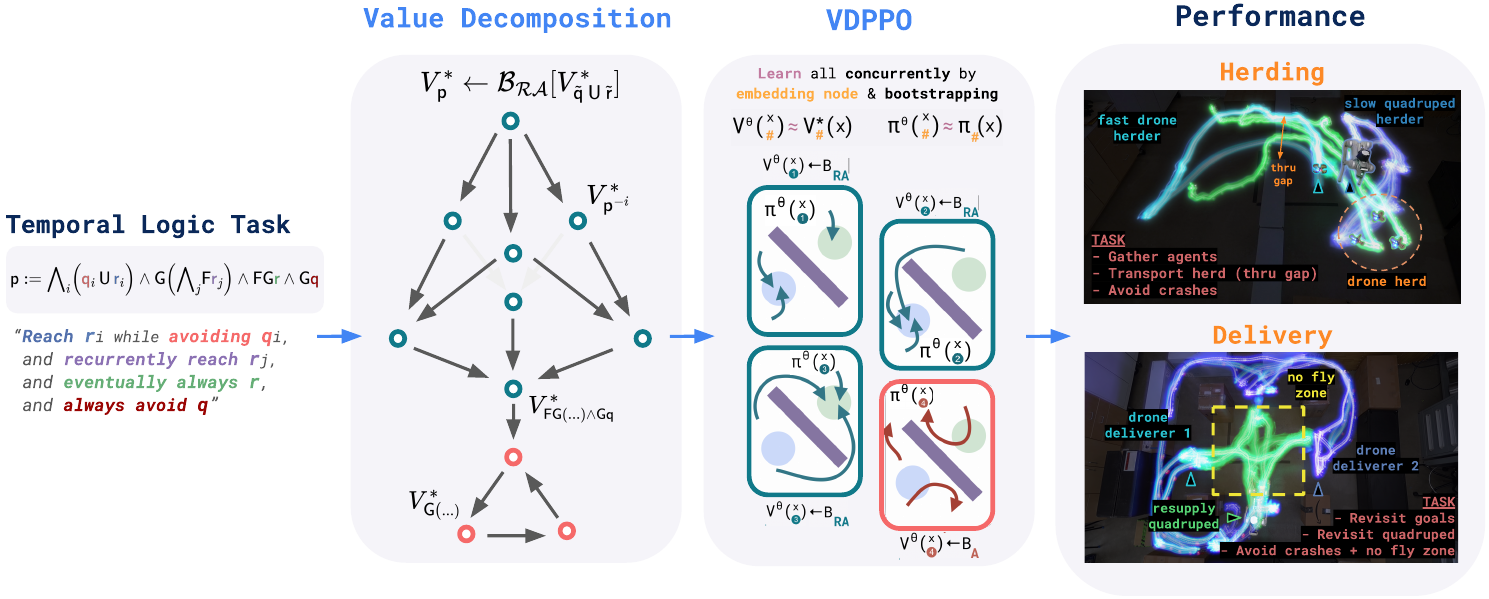}
    \vspace{0.1ex}
    \captionof{figure}{
    \textbf{Value-Decomposition and \texttt{VDPPO}.} The Bellman Value for a range of temporal logic (e.g., multi-goal, recurrence, staying, safety) decomposes into a Value graph connected by atomic Bellman equations (Thms. 1–4).
    We propose \texttt{VDPPO}, an algorithm that exploits this structure to learn policies for complex, high-dimensional tasks. Our approach is validated on hardware with \texttt{Herding} and \texttt{Delivery}, two complex tasks involving a heterogeneous team of drones and a quadruped.
    }
    \vspace{-1.5em}
    \label{fig:front}
\end{center}
\end{strip}

\begin{abstract}
Real-world tasks involve nuanced combinations of goal and safety specifications. 
In high dimensions, the challenge is exacerbated: formal automata become cumbersome, and the combination of sparse rewards tends to require laborious tuning.
In this work, we consider the innate structure of the Bellman Value as a means to naturally organize the problem for improved automatic performance.
Namely, \textbf{we prove the Value for a complex task defined in temporal logic can be decomposed into a graph of Values}, connected by a set of well-known Bellman equations (BEs): the \texttt{Reach}-\texttt{Avoid} BE, the \texttt{Avoid} BE, and a novel type, the \texttt{Reach}-\texttt{Avoid}-\texttt{Loop} BE.
To solve the Value and optimal policy, we propose \texttt{\textbf{VDPPO}}, which embeds the decomposed Value graph into a \textbf{two-layer neural net}, bootstrapping the implicit dependencies.
We conduct a variety of simulated and hardware experiments to test our method on complex, high-dimensional tasks involving heterogeneous teams and nonlinear dynamics. Ultimately, we find this approach greatly improves performance over existing baselines, balancing safety and liveness automatically.  \url{https://willsharpless.github.io/valdec-site/}
\end{abstract}

\section{Introduction and Related Work}

Popular modern approaches to robotic learning, like Reinforcement Learning (RL), typically optimize expected cumulative reward \cite{SuttonRL}, making them ill-suited for safety-critical or temporally structured tasks that require worst  or best-case satisfaction over long horizons. Such objectives are naturally expressed using Temporal Logic (TL) \cite{LTL-and-beyond}, but TL itself does not prescribe how to act. Moreover, efforts to combine RL \& TL tend to face a trade-off between sparse binary rewards that slow learning and hand-crafted dense rewards that can misalign objectives.

Hamilton–Jacobi Reachability (HJR) \cite{mitchell2005time, fisac2015reach} provides optimal controllers for basic safety and liveness tasks via max–min Bellman equations, yielding dense and informative learning signals. Recent work showed that certain TL tasks can be solved exactly by decomposing their Value functions into sequences of simple HJR problems \cite{sharpless2025dual}. As depicted in Fig.~\ref{fig:front} we generalize this idea to a broad class of TL specifications by introducing a Value function decomposition algebra and a corresponding algorithm \texttt{VDPPO}, and demonstrate effectiveness in simulation and real-world drone and quadruped experiments.

\textbf{Constrained, Multi-Objective, and Goal-Conditioned RL.}
Constrained Markov Decision Processes (CMDPs) maximize discounted rewards subject to constraints, typically via Lagrangian relaxation \cite{Altman-CMDPs,Abbeel-Constrained-Policy-Optimization,Safe-RL-CMDPs,Abbeel-Lagrangian-CMDP,CMDP3,CMDP4,CMDP5,CMDP6,CMDP7,CMDP11,CMDP12,CMDP13,pmlr-v168-castellano22a,pmlr-v238-mcmahan24a}, but require careful tuning and are ill-suited to general TL objectives. Multi-objective RL instead Pareto-optimizes multiple reward sums \cite{Model-based-multi-objective-2014,Pareto-Dominating-Policies-2014,Distributional-Multi-Objective,Multi-Objective-2016,MORL2,Generalized-Algorithm-Multi-Objective-2019,Liu_2025}, yet does not naturally encode TL structure. Goal-conditioned RL learns policies over a family of goals \cite{Goal-Conditioned-Problems-and-Solutions,Multi-Goal-Reinforcement-Learning,Exploration-via-Hindsight-GCRL,GCRL1,GCRL2,GCRL6,GCRL4,GCRL1,GCRL2,GCRL3}, but differs fundamentally from TL settings, where all specifications must be jointly satisfied.

  \textbf{RL with TL specifications.} A large body of works studies RL with TL specifications \cite{STL1,STL2,STL3,STL13,Rabin-Automaton-Sastry,Bozkurt_2020}, including approaches based on Non-Markovian Reward Decision Processes \cite{rewarding-behaviors,Decision-theoretic-planning,Guiding-search-LTL,reward-machines,LTL-and-beyond}, approximated quantitative semantics \cite{STL5,STL14,Cai_2021}, modified Bellman equations \cite{cdby1,cdby2,cdby3}, or multiple discounted rewards \cite{cdby4,cdby5,cdby6}. In contrast, our method exactly decomposes TL Value functions into simpler Values for which explicit Bellman equations exist, avoiding semantic approximation and long-horizon reward sparsity. See \Cref{apx:related} and \cite{sharpless2025dual} for additional discussion.

\textbf{Hamilton–Jacobi Reachability.}
HJR was originally developed to compute Value functions for \texttt{Reach}, \texttt{Avoid}, and \texttt{Reach}-\texttt{Avoid} problems in continuous time and space \cite{mitchell2005time,fisac2015reach}, corresponding to the quantitative semantics of \texttt{Eventually}, \texttt{Always}, and \texttt{Until} predicates \cite{chen2018signal}. Recent work has successfully integrated HJR into RL frameworks \cite{so2024solving,hsu2021safety,fisac2019bridging,Ganai2023,yu2022reachability,Zhu2024-ck}. Our work builds on these results by decomposing Value functions for complex TL objectives into sequences of simpler HJR problems.

\vspace{0.1em}
\section{Contributions}

\begin{enumerate}[leftmargin=*, itemsep=8pt, topsep=6pt]
\item We establish a formal connection between Temporal Logic and the Bellman Value by deriving algebraic rules for admissible operations, dubbed the \textcolor[HTML]{A64D79}{\texttt{Value Decomposition Rules}} (\textcolor[HTML]{A64D79}{\texttt{VDR}}) (\Cref{lem:equiv}).
\item We use the \textcolor[HTML]{A64D79}{\texttt{VDR}} to prove that, for a broad class of predicates, the Value decomposes into a graph of atomic BE (Thms.~\ref{thm:n-ra}, \ref{thm:n-ra-loop}): the \texttt{Avoid} BE, \texttt{Reach–Avoid} BE, and a novel \texttt{Reach–Avoid–Loop} BE for recurrence specifications (Def.~\ref{def:raloopbe}, Thm.~\ref{lem:raloopbe}).
\item We introduce \texttt{\textbf{VDPPO}}, an algorithm that embeds the decomposed Value graph (DVG) in two \textit{two-layer neural nets}, and demonstrate its success through extensive simulation and real-world hardware experiments, improving speed and achievement over baselines.

\end{enumerate}

\section{Preliminaries}

Given a discrete-time system $x_{t+1} = f(x_t, a_t)$ with finite states $x_t \in \mathcal{X} \subseteq \mathbb{R}^n$ and action $a_t \in \mathcal{A} \subseteq \mathbb{R}^m$, a trajectory beginning at $x$ is a sequence of states $\xi_{x}^{\alpha} := (x, ...) \in \mathbb{X} := \mathcal{X}^\mathbb{N}$ arising from actions $\alpha = (a, ...) \in \mathbb{A} := \mathcal{A}^\mathbb{N}$. 
We let $\xi_x(t)$ and $\alpha(t)$ be the state and action at time $t$. 

To specify desired properties of a trajectory, let an atomic proposition (AP) $\mathsf{\textcolor[rgb]{0.365,0.494,0.713}{r}}: \mathbb{R}^n \to \{\text{true}, \text{false}\}$ be defined by a bounded predicate function $\textcolor[rgb]{0.365,0.494,0.713}{r}: \mathbb{R}^n \to \mathbb{R}$, also known as a target or reward function in HJR or RL. 
Given a trajectory and time $(\xi_{x}, t)$, $\mathsf{\textcolor[rgb]{0.365,0.494,0.713}{r}}$ is satisfied, written $(\xi_{x}, t) \models \mathsf{\textcolor[rgb]{0.365,0.494,0.713}{r}}$, iff $\textcolor[rgb]{0.365,0.494,0.713}{r}(\xi_{x}(t)) \ge 0$, and thus, $\mathsf{\textcolor[rgb]{0.365,0.494,0.713}{r}}$ represents the arrival of a trajectory at a region (defined by the $0$-level-set of $\textcolor[rgb]{0.365,0.494,0.713}{r}$). Let $\mathsf{\textcolor[rgb]{0.8,0.4,0.4}{q}}$ with $\textcolor[rgb]{0.8,0.4,0.4}{q}$ denote an AP to maintain, s.t. the complement represents an unsafe region. 

To represent complex tasks, TL defines a language for the composition of predicates \cite{maler2004monitoring}. Predicates may be composed via negation (\texttt{NOT}, $\lnot$), conjunction (\texttt{AND}, $\land$), the  \texttt{Until} operator ($\mathsf{U}$) and the next operator (\texttt{NEXT}, $\mathsf{X}$).
With these operations, one may also define disjunction (\texttt{OR},$\lor$), finally/eventually ($\mathsf{F}$), and globally/always ($\mathsf{G}$). 
All operators may be defined via a robustness function $\rho: \mathbb{R}^n \to \mathbb{R}$ \cite{donze2010robust}, a scalar measure of satisfaction, equivalent to the payoff in HJB optimal control \cite{mitchell2005time,fisac2015reach}.
    \begin{defi} \label{def:tlpayoff}
        For any predicate $\mathsf{p}$, let the robustness score $\rho[\mathsf{p}]: \mathbb{X} \times \mathbb{N} \to \mathbb{R}$ be defined inductively with the following rules.

        \vspace{1.5ex}
        {\centering
        \includegraphics[width=\linewidth]
        {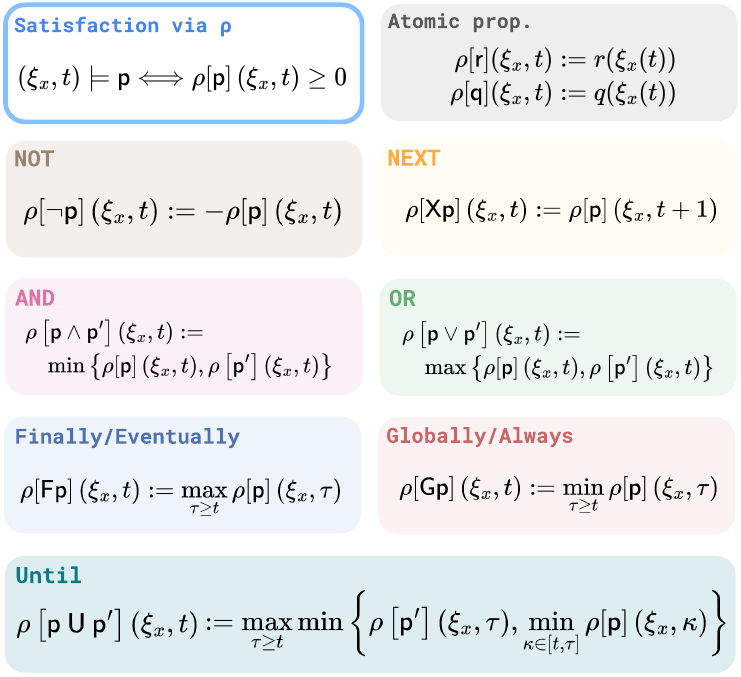}
        }
    \end{defi}
With this syntax, one may express the satisfaction of complex specifications over trajectories formally and succinctly\footnote{Note that $\mathsf{F}\mathsf{\textcolor[rgb]{0.365,0.494,0.713}{r}} = \top \, \mathsf{U} \, \mathsf{\textcolor[rgb]{0.365,0.494,0.713}{r}}$, where $\top$ is true, and thus it often suffices to consider only $\mathsf{U}$ and $\mathsf{G}$ in analysis. Similarly, $\mathsf{G}\mathsf{\textcolor[rgb]{0.8,0.4,0.4}{q}} \land \mathsf{G}\mathsf{\textcolor[rgb]{0.8,0.4,0.4}{q}}' = \mathsf{G}(\mathsf{\textcolor[rgb]{0.8,0.4,0.4}{q}} \land \mathsf{\textcolor[rgb]{0.8,0.4,0.4}{q}}') = \mathsf{G}\mathsf{\textcolor[rgb]{0.8,0.4,0.4}{q}}''$ so we write always specifications succinctly. Also note, $\mathsf{U}$ is defined per HJR convention to include the arrival time s.t. $\kappa \in [t, \tau]$, not $\kappa \in [t, \tau)$.}. Importantly, we use the the robustness score $\rho$ because this is the payoff used in the corresponding HJR optimal control problem \cite{mitchell2005time,fisac2015reach}
(see \Cref{apx:tlbg} for details).






\section{Problem Formulation}

In this work, we consider the problem of synthesizing optimal actions $\alpha$, such that for any initial state $x$ the resulting trajectory $\xi_{x}^\alpha$ maximizes the payoff $\rho$ for a given predicate. We later extend this to a policy $\pi:\mathcal{X} \to \mathcal{A}$ (\Cref{apx:policy}). We assume the system begins at $t=0$ and evolves indefinitely. For brevity, we let $\rho[\mathsf{p}](\xi) := \rho[\mathsf{p}](\xi, 0)$. This leads to the following infinite-horizon Optimal Control Problem (OCP),
\begin{align*}
    \begin{array}{cl}
    {\operatorname{maximize}}_\alpha & \rho[\mathsf{p}](\xi_{x}^\alpha), \\
    \text { s.t. } &  \xi_x^\alpha(t+1) = f \lf( \xi_x^\alpha(t), \alpha(t) \rg).
    \end{array}
\end{align*}
Note, because $\rho$ is defined by temporal extrema ($\max$/$\min$ over time), this program induces behavior characterized by its \textit{outlying performance}, in contrast with a sum-based OCP (in canonical RL \cite{SuttonRL}) which selects for average behavior. This objective is explicitly captured by the Bellman Value function $V^*$, the ``high score'' function for the given infinite-horizon OCP.
    \begin{defi} \label{def:bellmanvalue}
        For a predicate $\mathsf{p}$, we aim to solve
        \begin{equation} \label{eq:bellmanvalue}
            V^*[\mathsf{p}](x) := \max_\alpha \rho[\mathsf{p}](\xi_{x}^\alpha),
        \end{equation}
        which gives the optimal $\rho$ over infinite action sequences from any given initial state.
    \end{defi}

We define the Value for a general TL predicate $\mathsf{p}$ ($V^*_\mathsf{p}$ for brevity), but for basic operations, this object has been studied in-depth \cite{mitchell2005time,fisac2015reach,bansal2017hamilton}. Namely, the \texttt{Reach} ($\mathcal{R}$), \texttt{Avoid} ($\mathcal{A}$) and \texttt{Reach-Avoid} ($\mathcal{R}\mathcal{A}$) Values -- for respectively (1) reaching a target, (2) avoiding an obstacle, and (3) avoiding an obstacle until a target is reached -- correspond to the Values for $\mathsf{F}$, $\mathsf{G}$, and $\mathsf{U}$.
\vspace{1.5ex}
{\centering
\includegraphics[width=\linewidth, trim=0 4mm 0 -3mm]
{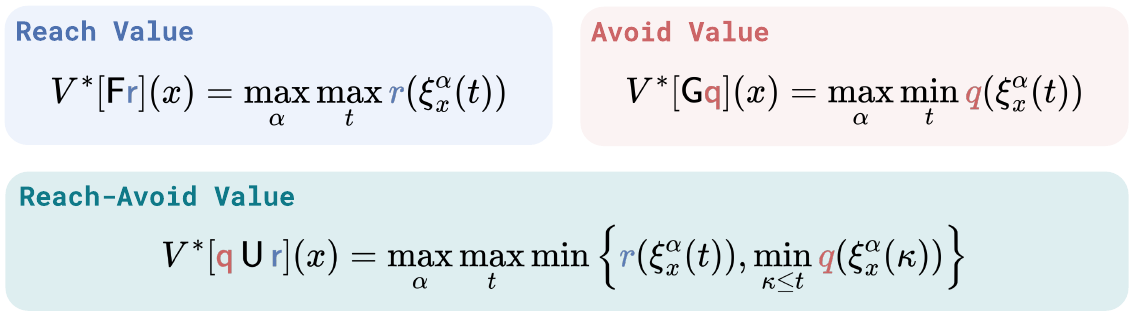}
}
In this context, the following Bellman operations for these extrema-based Values have been derived \cite{fisac2019bridging, hsu2021safety} (note, $\mathcal{R}$ is a special case of $\mathcal{R}\mathcal{A}$).

    \begin{defi} \label{def:atomicbellmaneqn}
        Let the $\mathcal{A}$ and $\mathcal{R}\mathcal{A}$ Bellman operators be
        \begin{align*}
            \mathcal{B}_{\mathcal{A}}^\gamma[V] &:= (1-\gamma)\textcolor[rgb]{0.8,0.4,0.4}{q} + \gamma \min\{V^{\scriptscriptstyle +}, \textcolor[rgb]{0.8,0.4,0.4}{q}\}, \\
            \mathcal{B}_{\mathcal{R}\mathcal{A}}^\gamma[V] &:= (1-\gamma)\min \{\textcolor[rgb]{0.365,0.494,0.713}{r}, \textcolor[rgb]{0.8,0.4,0.4}{q}\} + \gamma \min \{\max\{\textcolor[rgb]{0.365,0.494,0.713}{r}, V^{\scriptscriptstyle +}\}, \textcolor[rgb]{0.8,0.4,0.4}{q}\},
        \end{align*}
        where $V^{\scriptscriptstyle +}(x):=\max_a V(f(x,a))$.
    \end{defi}
These operators are of crucial importance as they are contractive and yield the namesake Value in the limit of discounting \cite{fisac2019bridging, hsu2021safety}. In other words, for the $\mathcal{A}$-Value $V^*[\mathsf{G}\mathsf{\textcolor[rgb]{0.8,0.4,0.4}{q}}]$ and the $\mathcal{R}\mathcal{A}$-Value $V^*[\mathsf{\mathsf{\textcolor[rgb]{0.8,0.4,0.4}{q}}\,U\,\mathsf{\textcolor[rgb]{0.365,0.494,0.713}{r}}}]$ defined by \eqref{eq:bellmanvalue}, we have that
\begin{align*}
    V^\gamma[\mathsf{G}\mathsf{\textcolor[rgb]{0.8,0.4,0.4}{q}}] \leftarrow \mathcal{B}^\gamma_\mathcal{A}[V^\gamma[\mathsf{G}\mathsf{\textcolor[rgb]{0.8,0.4,0.4}{q}}]], \quad V^\gamma[\mathsf{\textcolor[rgb]{0.8,0.4,0.4}{q}}\,\mathsf{U}\,\mathsf{\textcolor[rgb]{0.365,0.494,0.713}{r}}] \leftarrow \mathcal{B}^\gamma_{\mathcal{R}\mathcal{A}}[V^\gamma[\mathsf{\textcolor[rgb]{0.8,0.4,0.4}{q}}\,\mathsf{U}\,\mathsf{\textcolor[rgb]{0.365,0.494,0.713}{r}}]],
\end{align*}
are unique fixed points, satisfying $\lim_{\gamma \to 1} V^\gamma = V^*$ \cite{fisac2019bridging, hsu2021safety}. Unlike canonical RL \cite{SuttonRL}, these operators propagate extrema of $\textcolor[rgb]{0.365,0.494,0.713}{r}$ and $\textcolor[rgb]{0.8,0.4,0.4}{q}$, encouraging behavior defined by outlying performance. 

In a recent work \cite{sharpless2025dual}, it was demonstrated that for simple conjunctions $\mathsf{F}\mathsf{\textcolor[rgb]{0.365,0.494,0.713}{r}} \land \mathsf{G}\mathsf{\textcolor[rgb]{0.8,0.4,0.4}{q}}$ (reach a goal and avoid obstacles) and $\mathsf{F}\textcolor[rgb]{0.365,0.494,0.713}{\mathsf{r}_1} \land \mathsf{F}\textcolor[rgb]{0.365,0.494,0.713}{\mathsf{r}_2}$ (reach two goals), one may decompose the corresponding Values into the ``atomic'' BE, which in some ways resembles the base case for what follows. In this work, we generalize this principle, demonstrating that the $\mathcal{A}$-BE and $\mathcal{R}\mathcal{A}$-BE, along with the novel \texttt{Reach-Avoid-Loop} BE (Lem. \ref{lem:raloopbe}), serve as a set of ``atomic'' building blocks to decompose the Value of complex TL predicates.

\section{Optimality versus Satisfaction}

The decomposition of formal logic is well-studied in several contexts, including formal verification \cite{baier2008principles}, automata theory \cite{clarke1999model}, and temporal logic trees (TLT) \cite{bombara2016decision}. This body of work has established a rich framework for understanding the structure and properties of temporal logic formulas, and has led to performant decompositional learning methods for complex tasks \cite{meng2025tgpo}. However, the algebra of TL, which is equivalent to the algebra over the robustness score, is fundamentally distinct from that of the Value function due to the presence of the maximum over action sequences or control policies in \eqref{def:bellmanvalue}. This distinction is not only relevant to theoretical analysis but can lead to safety failures and sub-optimality in real world applications.
We illustrate this with the following remark and offer concrete counterexamples in \Cref{apx:counter}.

\begin{rmk} \label{rmk:counter}

The following TL identity always holds: $$\rho[\mathsf{ F}\mathsf{\textcolor[rgb]{0.365,0.494,0.713}{r}} \land \mathsf{G}\mathsf{\textcolor[rgb]{0.8,0.4,0.4}{q}}](\xi_x^\alpha) = \min \{\rho[\mathsf{ F}\mathsf{\textcolor[rgb]{0.365,0.494,0.713}{r}}](\xi_x^\alpha), \rho[\mathsf{ G}\mathsf{\textcolor[rgb]{0.8,0.4,0.4}{q}}](\xi_x^\alpha)\}.$$
By contrast, for the corresponding Value, we have $$V^*[\mathsf{ F}\mathsf{\textcolor[rgb]{0.365,0.494,0.713}{r}} \land \mathsf{G}\mathsf{\textcolor[rgb]{0.8,0.4,0.4}{q}}](x) \le \min \{V^*[\mathsf{ F}\mathsf{\textcolor[rgb]{0.365,0.494,0.713}{r}}](x), V^*[\mathsf{ G}\mathsf{\textcolor[rgb]{0.8,0.4,0.4}{q}}](x)\},$$ where the inequality is indeed strict when no single choice of action sequence can both reach $\mathsf{\textcolor[rgb]{0.365,0.494,0.713}{r}}$ and avoid $\lnot\mathsf{\textcolor[rgb]{0.8,0.4,0.4}{q}}$.
\end{rmk}

\section{RESULTS} \label{sec:results}

\begin{figure*}[t]
\begin{center}
    \includegraphics[
    width=\textwidth, 
    trim=0pt -8pt 0pt 0pt
    ]{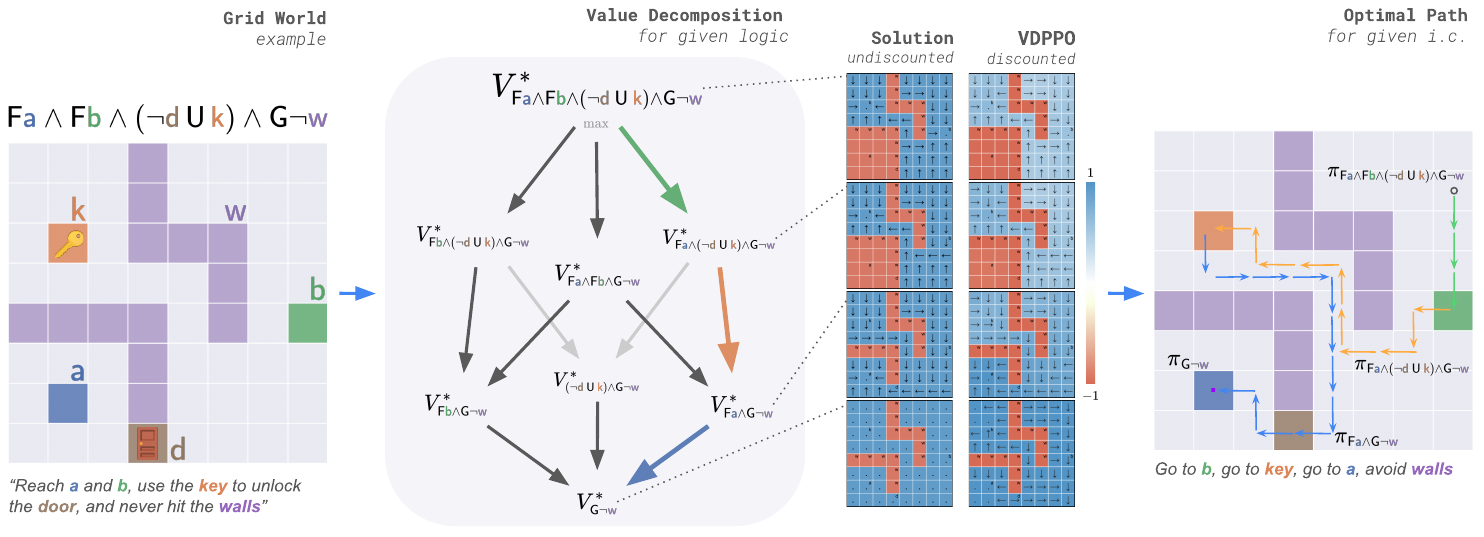}
    \captionof{figure}{\textbf{Multi-\texttt{Reach}-\texttt{Avoid} Value Decomposition.} Here we illustrate the primary decomposition result (Thm.~\ref{thm:n-ra}), with a \texttt{GridWorld} example (left) for a given specification. The corresponding DVG is shown (center left) with each node representing a decomposed Value, and edges representing dependencies. In the center right, a subset of decomposed Values solved with dynamic programming are shown, along with the discounted solution produced by \texttt{VDPPO}. On the right, the optimal path for a given initial condition is shown.
    }
    \label{fig:nRAdemofig}
    \vspace{-2em}
\end{center}
\end{figure*}

We present our main results regarding the decomposition of the Value for complex TL predicates. We begin by discussing the relationship between the Value and TL algebras, and then present the decompositional results. In general, we seek to express the Value for a complex predicate in terms of simpler components that are themselves composed with the fundamental Bellman equations.
All proofs of the results are in \Cref{apx:proofs}.

\subsection{Agreeable Algebra}

We begin by clarifying the similarity between the decomposition of the Value and TL algebra. Although not immediately apparent, these rules will serve as the foundation of complex decompositions presented later. The presence of the $\max_\alpha$ in \eqref{eq:bellmanvalue} does not always distinguish the Value algebra from that of TL, for example, when the TL is also defined by maxima, as with \texttt{OR} and a ``right''-$\mathsf{U}$ (for which, $\mathsf{F}$ is a special case). In any case, when $\max_\alpha$ commutes over the given $\rho$, a decomposition follows that mirrors that of the TL. To state these rules, it is first important consider the atomic proposition associated with a given Value.
\begin{defi} \label{def:atomicbellmaneqn}
    Let $\textcolor[HTML]{FFAB40}{\mathsf{v}}[\mathsf{p}]$ be the AP for $V^*[\mathsf{p}]$, s.t.
    \begin{equation*}
        \rho[\textcolor[HTML]{FFAB40}{\mathsf{v}}[\mathsf{p}]](\xi_{x}, t) := V^*[\mathsf{p}](\xi_{x}(t)).
    \end{equation*}
\end{defi}
\noindent We now give the algebraic properties of Value decomposition.
\begin{tcolorbox}[DefFrame]
    \begin{restatable}{thm}{lemmaequiv} \label{lem:equiv}
        \textnormal{\textcolor[HTML]{A64D79}{\texttt{Value Decomposition Rules}} (\textcolor[HTML]{A64D79}{\texttt{VDR}})}
        \\
        For general $\predA{p}$ and $\predB{p'}$, the following holds:
        \begin{enumerate}[itemsep=8pt, topsep=6pt]
            \item $V^*[\predA{p} \, \mathsf{U} \, \predB{p'}](x) = V^*[\predA{p} \, \mathsf{U} \, {\predV{v}[{\predB{p'}}}]](x)$ \hfill \textnormal{(\texttt{right-U})} \label{vdr:ru}
            \item $V^*[\predA{p} \lor \predB{p'}](x) = V^*[{\predV{v}[{\predA{p}}]} \lor \,{\predV{v}[{\predB{p'}}]}](x)$ \hfill \textnormal{(\texttt{OR})} \label{vdr:or}
            \item if $\forall\alpha,\beta\in \mathcal{A}^\mathbb{N},\,\rho[\predA{p}](\xi_x^\alpha)=\rho[\predA{p}](\xi_x^\beta)$,\\
            $V^*[\predA{p} \land \predB{p'}](x) = V^* [{\predV{v}[{\predA{p}}]} \land \predB{p'}](x)$  \hfill \textnormal{(\texttt{weak-AND})} \label{vdr:wand}
            \item $V^*[\mathsf{X}\predA{p}](x) = V^*[\mathsf{X}{\predV{v}[{\predA{p}}]}](x)$ \hfill \textnormal{(\texttt{NEXT})} \label{vdr:next}
        \end{enumerate}
    \end{restatable}
\end{tcolorbox}
\noindent Intuitively, the \textcolor[HTML]{A64D79}{{\texttt{VDR}}} offer four ways to decompose the Value by ``pushing it'' through the logic. In fact, each symbolic manipulation corresponds to distributing the infinite action sequence argument to the logical components of the total predicate. Then by definition, the residual objects are the component Values (see \Cref{apx:aa-proof} for proof). 

Together, this set of operations makes some compositions simple to consider. For example, one may see the Value for $\mathsf{FG}$, i.e. the ``Reach-Stay'' task, readily decomposes by \textcolor[HTML]{A64D79}{{\texttt{VDR}}}-\ref{vdr:ru}, 
\begin{equation} \label{eq:reachstay}
    V^*[\mathsf{FG}\mathsf{\textcolor[rgb]{0.8,0.4,0.4}{q}}](x) = V^*[\mathsf{F}\textcolor[HTML]{59bfb4}{\tilde{\mathsf{r}}}](x), \qquad \textcolor[HTML]{59bfb4}{\tilde{\mathsf{r}}}:=\textcolor[HTML]{FFAB40}{\mathsf{v}}[\mathsf{G}\mathsf{\textcolor[rgb]{0.8,0.4,0.4}{q}}].    
\end{equation}
Hence, this Value may be solved by, first, solving the $\mathcal{A}$-Value for $\mathsf{\textcolor[rgb]{0.8,0.4,0.4}{q}}$, $V^*[\mathsf{G}\mathsf{\textcolor[rgb]{0.8,0.4,0.4}{q}}]$, and then, second, using this to define the target $\textcolor[HTML]{59bfb4}{\tilde{\mathsf{r}}}:=\textcolor[HTML]{FFAB40}{\mathsf{v}}[\mathsf{G}\mathsf{\textcolor[rgb]{0.8,0.4,0.4}{q}}]$ of a $\mathcal{R}$-Value, $V^*[\mathsf{F}\textcolor[HTML]{59bfb4}{\tilde{\mathsf{r}}}]$. More generally, we may know that the Value for a series of \texttt{Until} predicates is equivalent to a chain of $\mathcal{R}\mathcal{A}$ Values (Cor.~\ref{cor:corountilchain} in \Cref{apx:aa-proof}). 

This example illustrates that decomposition via the \textcolor[HTML]{A64D79}{{\texttt{VDR}}} transforms a Value with no apparent Bellman equation into a hierarchy of Values composed of ``atomic'' Bellman equations (the DVG), amounting to a \textit{macro-scale} dynamic programming process. 
These results, however, do not directly apply when the TL is defined by $\min$ as with \texttt{AND}, and thus are insufficient for many common predicates. 
The solution will be to transform the logic into a special, amenable form.
We show several examples of this transformation in the following sections.

\subsection{Multi-\texttt{Reach}-\texttt{Avoid} Value Decomposition}

We next present the first major result of the work concerning the decomposition of the Value for the \texttt{AND} of $N$ \texttt{Until} predicates with a safety spec., or equivalently, the $N$-$\mathcal{R}\mathcal{A}$-$\mathcal{A}$ Value. This is a generalization of the dual-objective base cases that preceded this work \cite{sharpless2025dual}. Akin to later results, the decomposition is achieved by transforming the logic into an amenable form, and then applying the algebraic rules of \Cref{lem:equiv} to decompose the Value into a single $\mathcal{R}\mathcal{A}$-Value.

\begin{tcolorbox}[DefFrame]
    \begin{restatable}{thm}{thmnra}\label{thm:n-ra}
        For the predicate $\mathsf{p} := \bigwedge_{i \in \mathcal{I}} (\mathsf{\textcolor[rgb]{0.8,0.4,0.4}{q}}_i\, \mathsf{U}\, \mathsf{\textcolor[rgb]{0.365,0.494,0.713}{r}}_i) \land \mathsf{G}\mathsf{\textcolor[rgb]{0.8,0.4,0.4}{q}}$ with $N=|\mathcal{I}|$, the corresponding Value satisfies
        {\setlength{\belowdisplayskip}{0pt}%
        \setlength{\belowdisplayshortskip}{0pt}%
        \begin{align*}
            &V^*\left[\bigwedge\nolimits_i (\mathsf{\textcolor[rgb]{0.8,0.4,0.4}{q}}_i\, \mathsf{U}\, \mathsf{\textcolor[rgb]{0.365,0.494,0.713}{r}}_i)\land \mathsf{G}\mathsf{\textcolor[rgb]{0.8,0.4,0.4}{q}}\right]\!(x) \\ 
            & \qquad = V^* \Big[
            \textcolor[HTML]{e070a7}{\underbrace{\textcolor{black}{
                \left(\bigwedge\nolimits_i \mathsf{\textcolor[rgb]{0.8,0.4,0.4}{q}}_i \land \mathsf{\textcolor[rgb]{0.8,0.4,0.4}{q}}\right)
            }}_{\tilde{\mathsf{q}}}}
            \,\mathsf{U}\, 
            \textcolor[HTML]{59bfb4}{\underbrace{\textcolor{black}{
                \left(\bigvee\nolimits_i \left(\mathsf{\textcolor[rgb]{0.365,0.494,0.713}{r}}_i \land \textcolor[HTML]{FFAB40}{\mathsf{v}}[{\mathsf{p}_{-i}}] \right) \right)
            }}_{\tilde{\mathsf{r}}}}
            \Big](x)
        \end{align*}}
        where $\mathsf{p}_{-i} := \bigwedge_{j \in \mathcal{I} \setminus \{i\}} \mathsf{\textcolor[rgb]{0.8,0.4,0.4}{q}}_j\, \mathsf{U}\, \mathsf{\textcolor[rgb]{0.365,0.494,0.713}{r}}_j \land \mathsf{G}\mathsf{\textcolor[rgb]{0.8,0.4,0.4}{q}}$.
    \end{restatable}
\end{tcolorbox}
\begin{proof}
    We first transform $\mathsf{p}$ into an amenable form (Lem.~\ref{lem:symbolic:n_U_and}),
    {\setlength{\belowdisplayskip}{2pt}%
    \setlength{\belowdisplayshortskip}{0pt}%
    \begin{equation*}
        \mathsf{p} := 
        \bigwedge\nolimits_{i} (\mathsf{\textcolor[rgb]{0.8,0.4,0.4}{q}}_i\, \mathsf{U}\, \mathsf{\textcolor[rgb]{0.365,0.494,0.713}{r}}_i) \land \mathsf{G}\mathsf{\textcolor[rgb]{0.8,0.4,0.4}{q}}
        \equiv 
        \textcolor[HTML]{e070a7}{\underbrace{\textcolor{black}{
                    \left(\bigwedge\nolimits_i \mathsf{\textcolor[rgb]{0.8,0.4,0.4}{q}}_i \land \mathsf{\textcolor[rgb]{0.8,0.4,0.4}{q}}\right)
                }}_{\tilde{\mathsf{q}}}} 
        \,\mathsf{U}\, \Bigl( \bigvee\nolimits_i \bigl( \mathsf{\textcolor[rgb]{0.365,0.494,0.713}{r}}_i \land \mathsf{p}_{-i} \bigr) \Bigr).
    \end{equation*}}
    We may then simply apply \textcolor[HTML]{A64D79}{{\texttt{VDR}}}-\ref{vdr:ru}, \textcolor[HTML]{A64D79}{{\texttt{VDR}}}-\ref{vdr:or}, and \textcolor[HTML]{A64D79}{{\texttt{VDR}}}-\ref{vdr:wand},
    \begin{align*}
    V^* [\mathsf{p}](x)
    &= V^* \Big[
            \textcolor[HTML]{e070a7}{\tilde{\mathsf{q}}}
            \,\mathsf{U}\, 
                \textcolor[HTML]{FFAB40}{\mathsf{v}}[\vee_i \left(\mathsf{\textcolor[rgb]{0.365,0.494,0.713}{r}}_i \land {\mathsf{p}_{-i}}\right)]
            \Big](x), \\
    &= V^* \Big[
            \textcolor[HTML]{e070a7}{\tilde{\mathsf{q}}}
            \,\mathsf{U} 
                \left(\vee_i \, \textcolor[HTML]{FFAB40}{\mathsf{v}}[\mathsf{\textcolor[rgb]{0.365,0.494,0.713}{r}}_i \land {\mathsf{p}_{-i}}] \right)
            \Big](x), \\
    &= V^* \Big[
            \textcolor[HTML]{e070a7}{\tilde{\mathsf{q}}}
            \,\mathsf{U} 
                \left(\vee_i \left(\mathsf{\textcolor[rgb]{0.365,0.494,0.713}{r}}_i \land \textcolor[HTML]{FFAB40}{\mathsf{v}}[{\mathsf{p}_{-i}}] \right) \right)
            \Big](x).
            \qedhere
    \end{align*}
\end{proof}
This result gives an equivalence between the $N$-$\mathcal{R}\mathcal{A}$ Value and the Value function of a single $\mathcal{R}\mathcal{A}$ Value, which has abstract reach and avoid predicates $\textcolor[HTML]{59bfb4}{\tilde{\mathsf{r}}}$ \& $\textcolor[HTML]{e070a7}{\tilde{\mathsf{q}}}$ in the sense that they no longer represent physical goals or obstacles. Instead, the new reach predicate $\textcolor[HTML]{59bfb4}{\tilde{\mathsf{r}}}$ is defined by the \texttt{OR} of $N$ \texttt{AND}'s that each correspond to reaching one of the predicates $\textcolor[rgb]{0.365,0.494,0.713}{\mathsf{r}}_i$ and being able to satisfy the remaining logic $\mathsf{p}_{-i}$, i.e. having $V^*\left[{\mathsf{p}_{-i}}\right] \ge 0$. 
The new avoid predicate $\textcolor[HTML]{e070a7}{\tilde{\mathsf{q}}}$ is defined by the \texttt{AND} of all $N$ avoid predicates and hence implies that we need to avoid all $\textcolor[rgb]{0.8,0.4,0.4}{\mathsf{q}}_i$. Intuitively, \Cref{thm:n-ra} breaks down the optimal Value into the goal of safely reaching any of the predicates while being able to satisfy the rest of the predicate of $N-1$ \texttt{Until}'s, denoted $\mathsf{p}_{-i}$, where $\textcolor[rgb]{0.365,0.494,0.713}{\mathsf{r}}_i$ \& $\textcolor[rgb]{0.8,0.4,0.4}{\mathsf{q}}_i$ have been ``popped off'' the original predicate. 

\begin{figure*}[t]
\begin{center}
    \includegraphics[
    width=\textwidth, 
    trim=0pt -8pt 0pt 0pt
    ]{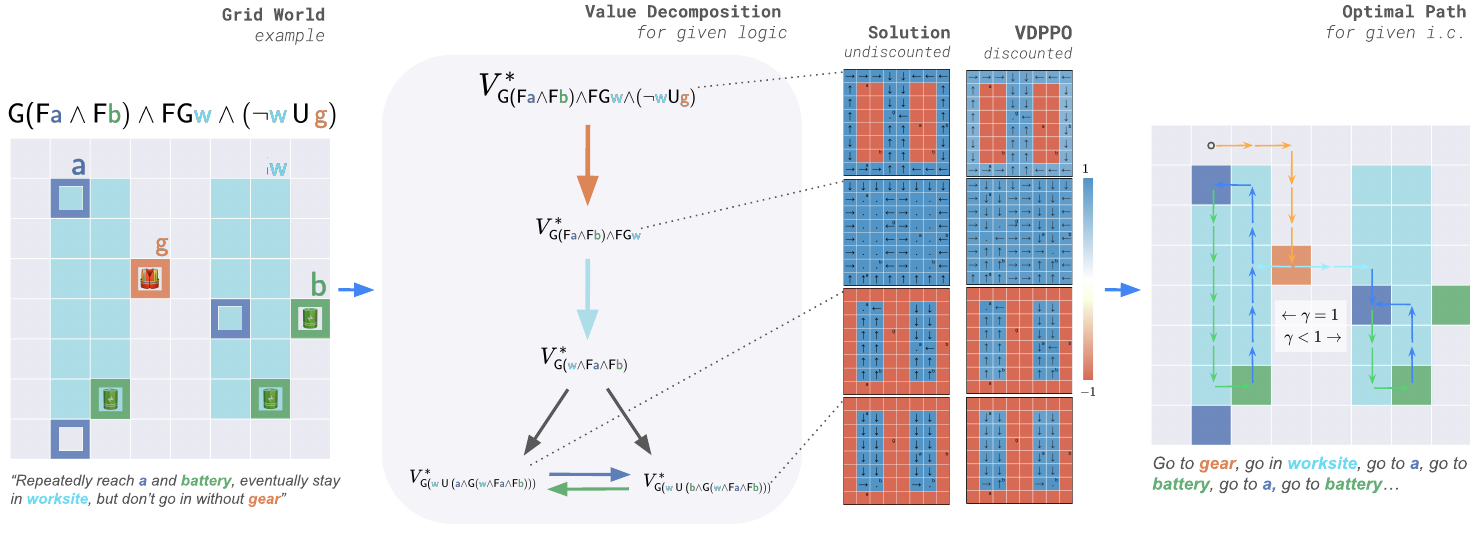}
    \captionof{figure}{\textbf{E.g. Recurrent \texttt{Reach}-\texttt{Avoid} Value Decomposition.} We illustrate the recursive decomposition result (\Cref{thm:n-ra-loop}), with a \texttt{GridWorld} example (left) for a given specification. The plots here are analogous to those of Fig.~\ref{fig:nRAdemofig}, with the DVG (center left), decomposed Values (center right), and optimal path (right). Note, the optimal path for the discounted case differs due to the subtle effect of discounting the Value associated with a $\mathsf{G}$ composition, which selects for shorter loops (Appendix~\ref{apx:nral-disct}).
    }
    \label{fig:alwaysdemofig}
    \vspace{-2em}
\end{center}
\end{figure*}

Notably, \Cref{thm:n-ra} can be applied recursively, and, therefore, we may reapply it iteratively to break the $N$-$\mathcal{R}\mathcal{A}$ Value into $N$ decomposable sub-Values and so forth, giving $2^N -1 $ Values in total. Crucially, as each of these Values is equivalent to a special \texttt{Until} Value, \textit{they may each be solved with the discounted $\mathcal{R}\mathcal{A}$-BE} with their respective rewards. We demonstrate this result in \Cref{fig:nRAdemofig} with a simple $\mathtt{GridWorld}$ problem, where the true solution may be solved via dynamic programming.

\subsection{Recurrent \texttt{Reach}-\texttt{Avoid} Value Decompositions} \label{sec:alwayscomp}

In this section, we consider the family of recurrence relation operations corresponding to the composition of $\mathsf{G}$ with $\mathsf{U}$ (for which $\mathsf{GF}$ is a special case). To always-eventually satisfy a predicate implies that a trajectory must continue to satisfy it indefinitely. These compositions are particularly important as they encompass the liveness (i.e., goal-satisfaction) property, where certain states must be revisited or regenerated in some sense. Moreover, this operation is significantly less strict than the Reach-Stay specification $\mathsf{F}\mathsf{G}$ in \eqref{eq:reachstay}, and thus more desirable, when continuous goal achievement is impossible or unknown.

The temporal coupling of the outer $\mathsf{G}$ with the inner TL makes the Value of these compositions behave differently. Note, by definition, the robustness function alone satisfies
\begin{equation}
    \rho[\mathsf{G}\mathsf{F}\mathsf{\textcolor[rgb]{0.365,0.494,0.713}{r}}](\xi_x) = \min_{t \in \mathbb{N}} \max_{t' \geq t} \rho[\mathsf{\textcolor[rgb]{0.365,0.494,0.713}{r}}](\xi_x, t') = \limsup_{s \to \infty} \rho[\mathsf{\textcolor[rgb]{0.365,0.494,0.713}{r}}](\xi_x, s),
\end{equation}
and, hence, the Value is constant along trajectories, regardless of when they begin.
Yet, this coupling also gives rise to an implicit characterization amenable to the \textcolor[HTML]{A64D79}{{\texttt{VDR}}}: e.g. $\mathsf{p} := \mathsf{G}(\mathsf{\textcolor[rgb]{0.8,0.4,0.4}{q}}\, \mathsf{U}\, \mathsf{\textcolor[rgb]{0.365,0.494,0.713}{r}})$ satisfies $\mathsf{p} \equiv \mathsf{\textcolor[rgb]{0.8,0.4,0.4}{q}}\, \mathsf{U}\, (\mathsf{\textcolor[rgb]{0.365,0.494,0.713}{r}} \land \mathsf{X}\mathsf{p})$, 
thus, we may apply \textcolor[HTML]{A64D79}{{\texttt{VDR}}}-\ref{vdr:ru},\ref{vdr:wand},\ref{vdr:next} to have the following.





\begin{tcolorbox}[DefFrame]
    \begin{restatable}{lemma}{thmraloop} \label{thm:raloop}
        For the predicate $\mathsf{p} := \mathsf{G} \left( \mathsf{\textcolor[rgb]{0.8,0.4,0.4}{q}}\, \mathsf{U}\, \mathsf{\textcolor[rgb]{0.365,0.494,0.713}{r}} \right)$ the corresponding Value satisfies
        \begin{equation*}
            V^*\left[\mathsf{p}\right](x) = V^* \left[\mathsf{\textcolor[rgb]{0.8,0.4,0.4}{q}} \,\mathsf{U}\, (\mathsf{\textcolor[rgb]{0.365,0.494,0.713}{r}} \land \mathsf{X}\textcolor[HTML]{FFAB40}{\mathsf{v}}[\mathsf{p}\right])](x).
        \end{equation*}
    \end{restatable}
\end{tcolorbox}
This result demonstrates that the Value function associated with the predicate $\mathsf{G} \left( \mathsf{\textcolor[rgb]{0.8,0.4,0.4}{q}}\, \mathsf{U}\, \mathsf{\textcolor[rgb]{0.365,0.494,0.713}{r}} \right)$ can be characterized implicitly. For intuition, one may consider this Value as a special $\mathcal{R}\mathcal{A}$ Value that aims to reach an intersection of the target predicate $\mathsf{\textcolor[rgb]{0.365,0.494,0.713}{r}}$ and its own satisfiable set (denoted by $\textcolor[HTML]{FFAB40}{\mathsf{v}}[{\mathsf{p}}]$) at the next step, and hence, maintain the ability to satisfy it again in the future. More generally, we may expand this result to the case involving a composition of $\mathsf{G}$ with $N$-\texttt{AND}-\texttt{Until}'s by considering a loop of predicates $\mathsf{p}_i^\ell$.
\begin{defi} \label{def:predloop}
    Given a predicate, $\mathsf{p} := \mathsf{G} ( \bigwedge_{i \in \mathcal{I}} (\mathsf{\textcolor[rgb]{0.8,0.4,0.4}{q}}_i\, \mathsf{U}\, \mathsf{\textcolor[rgb]{0.365,0.494,0.713}{r}}_i))$, let a loop of $N:=|\mathcal{I}|$ predicates $\mathsf{p}^\ell_{i}$ be 
    defined\footnote{This depends on the inclusion of the end-point in $\mathsf{U}$ (Def.~\ref{def:tlpayoff}) following HJR convention, which implies $\mathsf{p} \equiv \mathsf{G} (\bigwedge_{i} \mathsf{F}\mathsf{\textcolor[rgb]{0.365,0.494,0.713}{r}}_i \land \bigwedge\nolimits_i \mathsf{\textcolor[rgb]{0.8,0.4,0.4}{q}}_i)$. Otherwise, let $\mathsf{p}^\ell_{i} := (\bigwedge\nolimits_j \mathsf{\textcolor[rgb]{0.8,0.4,0.4}{q}}_j \land \tilde{\mathsf{w}}_i) \, \mathsf{U} \, (\mathsf{\textcolor[rgb]{0.365,0.494,0.713}{r}}_i \land \tilde{\mathsf{w}}_i \land \mathsf{X}\mathsf{p}_{i+1}^\ell)$ where $\tilde{\mathsf{w}}_i:= \wedge_{j \ne i} 
    (\mathsf{\textcolor[rgb]{0.8,0.4,0.4}{q}}_j \land \mathsf{\textcolor[rgb]{0.365,0.494,0.713}{r}}_j)$.} s.t. $\mathsf{p}_{N+1}^\ell:=\mathsf{p}_1^\ell$ and 
    \begin{equation}
        \mathsf{p}^\ell_{i} := \left(\bigwedge\nolimits_{j\in \mathcal{I}} \mathsf{\textcolor[rgb]{0.8,0.4,0.4}{q}}_j \right) \, \mathsf{U} \, \Big(\mathsf{\textcolor[rgb]{0.365,0.494,0.713}{r}}_i \land \mathsf{X}\mathsf{p}_{i+1}^\ell \Big).
    \end{equation}
\end{defi}
\begin{tcolorbox}[DefFrame]
    \begin{restatable}{thm}{thmnraloop} \label{thm:n-ra-loop}
        For predicate $\mathsf{p} := \mathsf{G}\left(\bigwedge_{i \in \mathcal{I}} (\mathsf{\textcolor[rgb]{0.8,0.4,0.4}{q}}_i\, \mathsf{U}\, \mathsf{\textcolor[rgb]{0.365,0.494,0.713}{r}}_i)\right)$, the corresponding Value satisfies $\forall i,$
        {\setlength{\belowdisplayskip}{0pt}%
        \setlength{\belowdisplayshortskip}{0pt}%
        \begin{align*}
            &V^*\left[\mathsf{G}\left(\bigwedge\nolimits_{i} (\mathsf{\textcolor[rgb]{0.8,0.4,0.4}{q}}_i\, \mathsf{U}\, \mathsf{\textcolor[rgb]{0.365,0.494,0.713}{r}}_i)\right)\right]\!(x) \\ 
            & = V^*\left[\mathsf{p}_i^\ell\right]\!(x) = V^* \Big[
            \textcolor[HTML]{e070a7}{\underbrace{\textcolor{black}{
                \left(\bigwedge\nolimits_{j\in\mathcal{I}} \mathsf{\textcolor[rgb]{0.8,0.4,0.4}{q}}_{j} \right)
            }}_{\tilde{\mathsf{q}}}}
            \,\mathsf{U}\, 
            \textcolor[HTML]{59bfb4}{\underbrace{\textcolor{black}{
                \Big(\mathsf{\textcolor[rgb]{0.365,0.494,0.713}{r}}_i \land \mathsf{X}\textcolor[HTML]{FFAB40}{\mathsf{v}}[{\mathsf{p}^{\ell}_{i+1}}] \Big)
            }}_{\tilde{\mathsf{r}}}}
            \Big](x).
        \end{align*}}
        where $\mathsf{p}^{\ell}_{i}$ and $\mathsf{p}^{\ell}_{i+1}$ are defined in Def.~\ref{def:predloop}.
    \end{restatable}
\end{tcolorbox}
\begin{proof}
    Here, the amenable logical form is the loop of predicates $\mathsf{p}_i^\ell$: $\forall i$, we have $\mathsf{p}=\mathsf{p}_i^\ell$ (Lem.~\ref{lem:symbolic:G_of_double_U}).
    We may then apply \textcolor[HTML]{A64D79}{{\texttt{VDR}}}-\ref{vdr:ru}, \textcolor[HTML]{A64D79}{{\texttt{VDR}}}-\ref{vdr:wand}, and \textcolor[HTML]{A64D79}{{\texttt{VDR}}}-\ref{vdr:next},
    \begin{align*}
    V^* [\mathsf{p}_i^\ell](x)
    &= V^* \Big[
            \textcolor[HTML]{e070a7}{\tilde{\mathsf{q}}}
            \,\mathsf{U}\, 
                \textcolor[HTML]{FFAB40}{\mathsf{v}}[\mathsf{\textcolor[rgb]{0.365,0.494,0.713}{r}}_i \land {\mathsf{X}\mathsf{p}_{i+1}^\ell}]
            \Big](x), \\
    &= V^* \Big[
            \textcolor[HTML]{e070a7}{\tilde{\mathsf{q}}}
            \,\mathsf{U} 
                \left(\mathsf{\textcolor[rgb]{0.365,0.494,0.713}{r}}_i \land \textcolor[HTML]{FFAB40}{\mathsf{v}}[{\mathsf{X}\mathsf{p}_{i+1}^\ell}] \right)
            \Big](x), \\
    &= V^* \Big[
            \textcolor[HTML]{e070a7}{\tilde{\mathsf{q}}}
            \,\mathsf{U} 
                \left(\mathsf{\textcolor[rgb]{0.365,0.494,0.713}{r}}_i \land \mathsf{X}\textcolor[HTML]{FFAB40}{\mathsf{v}}[{\mathsf{p}_{i+1}^\ell}] \right) 
            \Big](x).
            \qedhere
    \end{align*}
\end{proof}
\begin{figure*}[t]
\begin{center}
    \vspace{-1em}
    \centering
    \includegraphics[
    width=\linewidth,
    trim=0pt -6pt 0pt 0pt,
    ]{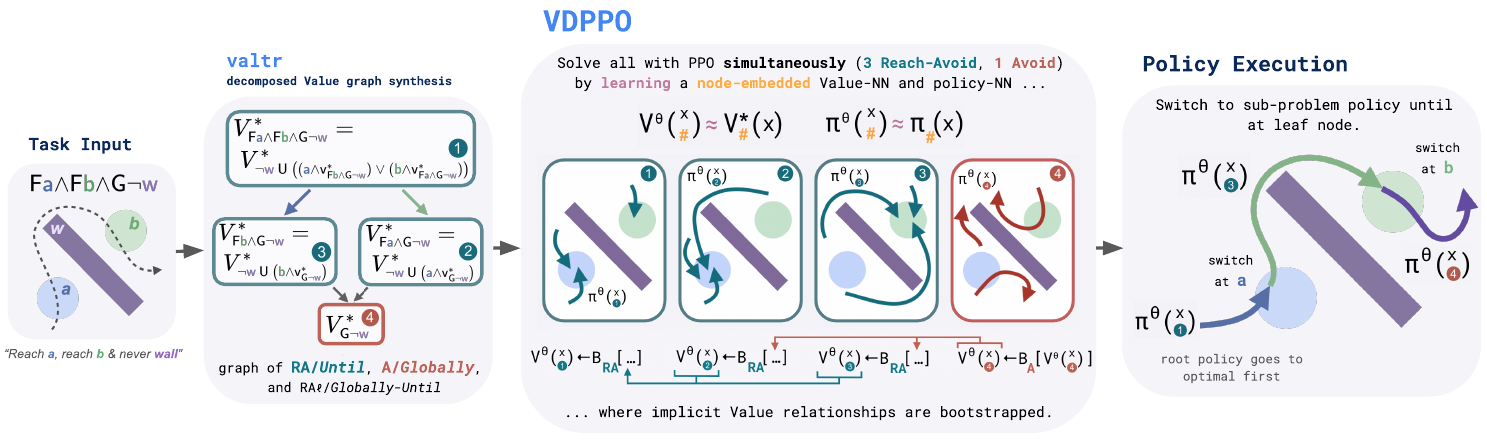}
    \captionof{figure}{
    \textbf{Graphical Depiction of \texttt{valtr} and \texttt{VDPPO}.} TL is inputted, e.g., reaching two goals and avoiding a wall: $\mathsf{F\textcolor[rgb]{0.302,0.447,0.690}{a}} \land \mathsf{F\textcolor[rgb]{0.333,0.6588,0.4078}{b}}  \land  \mathsf{G}\lnot \mathsf{\textcolor[rgb]{0.549,0.447,0.702}{w}}$ (far left). Next, \texttt{valtr} decomposes the logic based on the \textcolor[HTML]{A64D79}{{\texttt{VDR}}}, yielding a graph of $\mathcal{RA}$, $\mathcal{A}$, and $\mathcal{RA}_\ell$ nodes, e.g., $4$ Values with $1$ $\mathcal{A}$ leaf and $3$ $\mathcal{RA}$ parents (mid-left). Next, \texttt{VDPPO} learns an actor and critic for all Values simultaneously by embedding the graph. The Bellman update for each node bootstraps the corresponding children (mid-right). The policy is executed beginning at the root node and switches upon predicate satisfaction (far right).
    }
    \label{fig:algplot0}
    \vspace{-2em}
\end{center}
\end{figure*}
This result allows us to consider the problem of recurrently reach-avoiding $N$ tasks as a loop of $N$ coupled Values. 
Perhaps bizarrely, the implicit nature implies that the loop order is arbitrary (discounting in Def.~\ref{def:raloopbe}, however, gives higher Value to shorter loops, see Appendix \Cref{apx:nral-disct}).

Although these results appear similar to the previous decompositions, it is important to note that they are fundamentally different due to the implicit definition. These characterizations do \textit{not} offer a method to solve the Value, and in some state spaces, may be meaningless. 
To certify the well-posedness in certain scenarios (e.g. finite state spaces), we show in \Cref{apx:fxd-proof} that these Values are equivalent to the limit of finite recurrence, however, this is not generally a practical procedure.
Moreover, a straightforward application of the discounted $\mathcal{R}\mathcal{A}$-BE yields a BE that is \textit{not necessarily contractive}, due to the appearance of the Value in both $(1-\gamma)$ and $\gamma$ terms. To address these challenges, we propose a novel contractive Bellman Equation, which we call the  \texttt{Reach-Avoid-Loop} ($\mathcal{R}\mathcal{A}_\ell$) BE, which is guaranteed to yield the desired Value in the limit of discounting.
\begin{defi} \label{def:raloopbe}
    Given a set of $N$ values $\{V_i\}^N$, let the $\mathcal{R}\mathcal{A}_\ell$ Bellman operator be defined for each element as
    \begin{align*}
        \mathcal{B}^\gamma_{\mathcal{R}\mathcal{A}_\ell}& [V_i] := (1- \gamma) \min \{\textcolor[rgb]{0.365,0.494,0.713}{r}_i, \min\nolimits_j \textcolor[rgb]{0.8,0.4,0.4}{q}_j\} \: + \\
        & \qquad \gamma \min \Big\{ \max \Big\{  \min \left\{ \textcolor[rgb]{0.365,0.494,0.713}{r}_i, V^{\scriptscriptstyle +}_{i+1} \right\}, V^{\scriptscriptstyle +}_i \Big\}, \min\nolimits_j \textcolor[rgb]{0.8,0.4,0.4}{q}_j \Big\}.
    \end{align*} 
\end{defi}
\begin{tcolorbox}[DefFrame]
    \begin{restatable}{thm}{lemraloopbe} \label{lem:raloopbe}
        The $\mathcal{R}\mathcal{A}_\ell$-BE given in Def.~$\ref{def:raloopbe}$ is contractive such that for a set of $N$ values $V^\gamma[\mathsf{p}^\ell_i]$, $$V^\gamma[\mathsf{p}^\ell_i] \leftarrow \mathcal{B}^\gamma_{\mathcal{R}\mathcal{A}_\ell}[V^\gamma[\mathsf{p}^\ell_i]]$$ is a unique fixed point, satisfying $\forall i$, $$\lim_{\gamma \to 1} V^\gamma[\mathsf{p}^\ell_i] = V^*\left[\mathsf{G}\left(\bigwedge\nolimits_{i} (\mathsf{\textcolor[rgb]{0.8,0.4,0.4}{q}}_i\, \mathsf{U}\, \mathsf{\textcolor[rgb]{0.365,0.494,0.713}{r}}_i)\right)\right].$$
    \end{restatable}
\end{tcolorbox}
Equipped with the $\mathcal{R}\mathcal{A}_\ell$-BE, we can now approximate the Value for the family of recurrence predicates effectively.

\subsection{General Decomposition and Policy}

In general, the \textcolor[HTML]{A64D79}{{\texttt{VDR}}} may be used to decompose the Value beyond the cases derived in Thms.~\ref{thm:n-ra} \& \ref{thm:n-ra-loop}. These were presented for their practical relevance, and because they serve as examples of a general proof strategy, however, they are not the limit of the \textcolor[HTML]{A64D79}{{\texttt{VDR}}}. For example, using the same techniques, it is simple to consider a combination of the two classes as presented in Fig.~\ref{fig:front}, which we give in \Cref{apx:gen-proof}. Moreover, in \Cref{apx:policy}, we demonstrate how this strategy can be used to synthesize a corresponding policy.

\section{Algorithm(s)}

In this section, we introduce Value-Decomposition PPO, a variant of PPO that solves the Value associated with the class of TL predicates in Sec.~\ref{sec:results} using the decomposed Value graph (DVG). We also describe the tools required to generate the DVG and to solve it via dynamic programming for low-dimensional problems. A graphical overview is shown in Fig.~\ref{fig:algplot0}.

\textbf{\texttt{valtr}: Generating the DVG}. We introduce \texttt{valtr}, a tool that converts a parsed temporal logic specification into the general predicate form of Thm.~\ref{thm:master} by recursively applying standard TL rules into a form suitable form for the \textcolor[HTML]{A64D79}{{\texttt{VDR}}}. This representation is then transformed into the DVG, where nodes correspond to predicates, negations, $\max$, $\min$, and Value functions, and edges encode their dependencies. Cyclic recurrence compositions are handled via a special node, enabling efficient parsing and transformation of arbitrary predicates into DVGs. See \Cref{apx:valtr} for details.

\textbf{Dynamic Programming with the DVG.} With the DVG, one may compute the Value of a given predicate by performing a topological sort of the DVG and applying dynamic programming to compute the Value of each subformula in the correct order. Cyclic nodes, arising from recurrence specifications, can be approximate by a finite number of recurrences as described in \Cref{apx:fxd-proof}. This allows us to compute the dynamic programming solution for the low-dimensional test cases given in Figs.~\ref{fig:nRAdemofig} and \ref{fig:alwaysdemofig}.

\textbf{Policy Inference with the DVG.} To execute the policy, we keep track of the current node in the DVG initialized at the root node.
At each step, we execute the policy associated with the current node and check its satisfaction.
If the predicate is satisfied, we switch to the child node corresponding to the satisfied predicate and repeat this process.
We provide more details in \Cref{apx:policy}. For an explicit construction and rigorous proof of the optimalality of the constructed policy, see \cite{so2026value}.

\begin{figure*}[t]
\begin{center}
    \vspace{-.8em}
    \includegraphics[width=1.\textwidth]{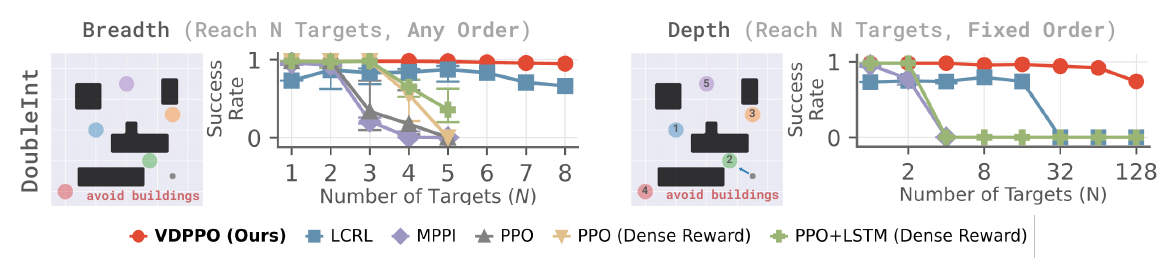}
    \captionof{figure}{
    \textbf{Sim. Performance Scaling with TL Complexity.} On the left, the robustness of \texttt{VDPPO} with respect breadth of a multi-reach spec. is assessed with a $N$-$\mathcal{R}\mathcal{A}$-$\mathcal{A}$ spec. (Thm.~\ref{thm:n-ra}) in which the agent must reach $N$ targets and avoid crashing into the buildings. Both the DVG and automatons scale exponentially with breadth, however, standard methods with additive rewards prove the worst. On the right, depth of a multi-reach spec. is assessed with a chain of $N$ specs. and a safety spec. Here, the linear topology of the DVG allows VDPPO to scale to higher $N$. 
    }
    \label{fig:ablations}
\end{center}
\end{figure*}

\begin{figure*}[t]
\begin{center}
    \vspace{-1.5em}

    \includegraphics[width=0.95\linewidth,trim={0 -5pt 0 -5pt}, clip]{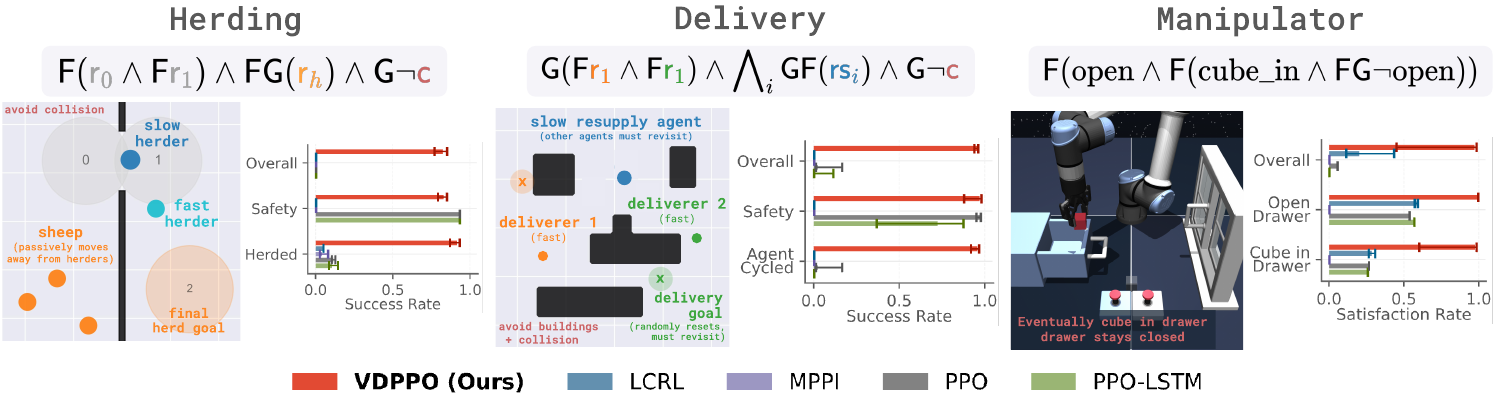}

    \captionof{figure}{\textbf{Sim. Performance on Complex, High-Dimensional Tasks}.
    Here, we assess \texttt{VDPPO} on more realistic tasks involving complex specifications, namely the \texttt{Herding} (30D, left), \texttt{Delivery} (18D, middle), and \texttt{Manipulator} (39D, right) environments. The first two tasks involve heterogeneous teams and a wide breadth of specifications, which leads to extraneous caution in baselines. The latter task involves a single agent with complex grasping and collision dynamics and a simpler series of tasks, yielding better baseline performance initially but poor sequential performance. In all cases, \texttt{VDPPO} excels, automatically breaking down complex tasks into learnable sub-tasks via the DVG and graph embedding.
    }
    \label{fig:simenvs}
    \vspace{-2em}
\end{center}
\end{figure*}

\texttt{\textbf{VDPPO.}} Finally, we propose Value-Decomposition PPO (\texttt{VDPPO}), a special variant of PPO which solves the Value associated with the class of TL predicates in Sec.~\ref{sec:results} by using the DVG. \texttt{VDPPO} uses one two-layer neural net (NN) for all Values in the DVG and one for all corresponding policies, by embedding the node representations with a one-hot vector. \texttt{VDPPO} solves all nodes in the DVG simultaneously by conducting rollouts for each node with the corresponding embedded policy. Depending on the embedding value, the target and advantage for each node are computed with the corresponding discounted $\mathcal{A}$-BE, $\mathcal{R}\mathcal{A}$-BE or $\mathcal{R}\mathcal{A}_\ell$-BE. To train all nodes concurrently, at this step \texttt{VDPPO} uses the current child node Value estimate for the corresponding BE (``bootstraps''), which is represented by the feedback loop in Fig.~\ref{fig:algplot0}. The policy is trained with the standard PPO objective, using the advantage estimate corresponding to the embedding. This allows us to leverage the decomposed structure of the Value functions to efficiently learn policies that satisfy complex TL specifications without sequentially approximating the Value.
See \Cref{apx:vdppo} for further details.





\section{Simulation Results}

To better understand the performance of \texttt{VDPPO}, we design simulation experiments to answer
the following questions:

\begin{enumerate}[leftmargin=1.2cm,label={(\bfseries Q\arabic*):}]
    \item Does Value decomposition help satisfy \& optimize a wider spec. with more options?
    \item Does Value decomposition help satisfy \& optimize a longer horizon spec.?
    \item Can \texttt{VDPPO} scale to more complex dynamics? 
\end{enumerate}

Details on experiments are in \Cref{apx:sim} (sim details) and \Cref{apx:baselines} (baselines).
Additional ablation studies are provided in \Cref{apx:ablations}.

\subsection{Setup}

\textbf{Environments.}
We evaluate on four simulated domains:
\texttt{DoubleInt} (toy double integrator environment to focus on TL challenges),
\texttt{Herding} (a team of herders collaborates to herd multiple targets to a designated location while avoiding obstacles),
\texttt{Delivery} (agents must continuously pick up and deliver packages to a special agent while avoiding collisions with each other and static obstacles), and
\texttt{Manipulator} (a robotic arm interacts with a cube and a drawer as specified by TL formulas) from OGBench's \texttt{scene} task \cite{ogbench_park2025}.

\textbf{Baselines.} We compare \texttt{VDPPO} with other model-free methods that can solve TL specifications with black-box dynamics. These include 1) \texttt{LCRL} \cite{hasanbeig2022lcrl}, a deep RL method that solves TL tasks by augmenting the state space with an automata representation of the TL formula, 2) an extension of Model Predictive Path Integral (MPPI) \cite{williams2016aggressive} to tackle TL problems \cite{halder2025trajectory}, \texttt{TL-MPPI}, 3) \texttt{PPO} using the robustness metric as a reward, and 4) an extension of \texttt{PPO} using a Long Short-Term Memory backbone \cite{hochreiter1997long}, \texttt{PPO-LSTM}.
For each environment, \texttt{LCRL}, \texttt{PPO}, \texttt{PPO-LSTM} and \texttt{VDPPO} are run for the same number of update steps, while for \texttt{TL-MPPI} we follow the hyperparameters chosen in \cite{halder2025trajectory}.

\textbf{Evaluation criteria.} Performance is measured by success rate on finite-horizon TL satisfaction; we additionally report satisfaction rates of individual subformulas.
All methods are trained with three seeds and evaluated on 256 initial conditions.


\subsection{Results}


\textbf{(Q1): Value decomposition improves breadth scalability with TL complexity.}
To answer this question, we study the effect of an increasing $N$-$\mathcal{R}\mathcal{A}$-$\mathcal{A}$ spec., i.e. a $N$-\texttt{AND} of \texttt{Until} in a single-agent double integrator environment. The spec. requires an agent to visit $N$ targets in any order while maintaining safety. The problem is difficult because the number of paths grows exponentially. All methods solve the singular spec. but degrade as the number of specifications increases.
\texttt{VDPPO} consistently outperforms all baselines as the breadth complexity of the TL spec. increases (Fig.\ref{fig:ablations}, left), demonstrating the effectiveness of Value decomposition in handling complex TL tasks.

\textbf{(Q2): Value decomposition improves depth scalability with TL complexity.}
To answer this question, we study the effect of an increasing $N$-chain-of-$\mathcal{R}\mathcal{A}$ spec., i.e. a $N$-chain of \texttt{Until} predicates in a single-agent double integrator environment. The spec. requires an agent to visit $N$ targets in a fixed order while maintaining safety. Again, \texttt{VDPPO} consistently outperforms all baselines as the depth complexity of the TL spec. increases (Fig.\ref{fig:ablations}, right), achieving up to $128$ targets reliably.
This is likely because the probability of randomly satisfying nested TL specifications decreases exponentially with depth, making it difficult for non-decompositional methods to learn effective policies.


\textbf{(Q3): \texttt{VDPPO} outperforms on problems with difficult dynamics. }
We now consider more challenging problems, either due to complex interactions with uncontrolled agents (\texttt{Herding}), heterogeneous collaboration (\texttt{Delivery} \& \texttt{Herding}), or complex dynamics (\texttt{Manipulator}), and show the results in Fig.~\ref{fig:simenvs}. In all three tasks, \texttt{VDPPO} achieves the highest success rate by a significant margin. See \Cref{apx:ablations} for ablations.




\section{Hardware Results}



Lastly, we perform hardware experiments corresponding to the \texttt{Herding} and \texttt{Delivery} environments using a swarm of Crazyflie (CF) drones collaborating with the Unitree Go2 to demonstrate the ability of \texttt{VDPPO} to solve complex task specifications in high-dimensional real-world settings with heterogeneous collaboration. See  Fig.~\ref{fig:hardwareoverview} for an overview.

\subsection{\texttt{Herding}}

In this experiment, we consider a team of one CF and the Go2 tasked with herding three ``sheep'' CFs through a narrow gap to a target location while avoiding obstacles and collisions. The sheep CFs have a fixed nominal policy, using the softmin to drive them away from the nearest obstacle or agent, and thus will move only when approached. 

The TL specification for the task is given by,
\begin{equation}
\mathsf{p}_{\text{herding}} := \mathsf{F}( \textcolor[rgb]{0.64,0.64,0.64}{\mathsf{r}_0} \land \mathsf{F}\textcolor[rgb]{0.64,0.64,0.64}{\mathsf{r}_1} ) \land \mathsf{F G} (\textcolor[rgb]{1.,0.64,0.25}{\mathsf{r}_h}) \land \mathsf{G}\lnot \textcolor[rgb]{0.8,0.4,0.4}{\mathsf{c}},
\end{equation}
where $\textcolor[rgb]{0.8,0.4,0.4}{\mathsf{c}}$ denotes collisions, $\textcolor[rgb]{0.64,0.64,0.64}{\mathsf{r}_0}$ the herd reaching the pre-gate region, $\textcolor[rgb]{0.64,0.64,0.64}{\mathsf{r}_1}$ passage through the gate, and $\textcolor[rgb]{1.,0.64,0.25}{\mathsf{r}_h}$ arrival at the target. This encodes a sequence of reach–avoid objectives followed by a reach–stabilize objective requiring indefinite herding.

The herders (CF and Go2) are initialized opposite the gap from the herd and have asymmetric dynamics, with the Go2 moving more slowly. To satisfy the specification, the herders must coordinate to pass through the gap, collect the herd, and guide it to the target while avoiding obstacles. We train a \texttt{VDPPO} policy using the DVG and deploy it on hardware, where the agents adapt online to real-time state feedback.

Ultimately, we observe that the CF and Go2 \textbf{learn to divide the labor} of the task such that the CF passes through the gap to gather the agents (Fig.\ref{fig:hardware_rollouts}.B), while the Go2 waits to receive on the herding side (Fig.\ref{fig:hardware_rollouts}.C). When the herd passes through the narrow gap, the Go2 initially moves out of the way (Fig.\ref{fig:hardware_rollouts}.C) and then transitions to providing support, rapidly shifting position to block the Herd from distributing across the new space (Fig.\ref{fig:hardware_rollouts}.E). This behavior is entirely emergent and demonstrates the wide-ranging ability of \texttt{VDPPO} to solve complex tasks automatically.

\begin{figure}[t]
\begin{center}
    \vspace{-.5em}
    \centering
    \includegraphics[width=0.9\linewidth,trim={2mm 0 0 0},clip]{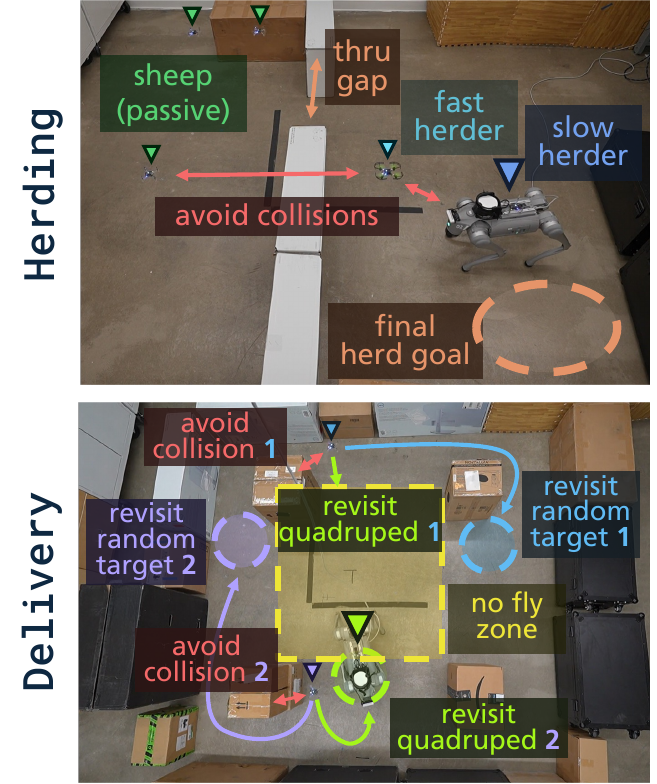}

    \vspace{0.2em}
    \captionof{figure}{\textbf{Hardware Overview for \texttt{Herding} and \texttt{Delivery} Tasks}}
    \label{fig:hardwareoverview}
    \vspace{-2em}
\end{center}
\end{figure}
\begin{figure*}[t]
\begin{center}
    \includegraphics[width=1\textwidth, trim=0 -1mm 0 0]{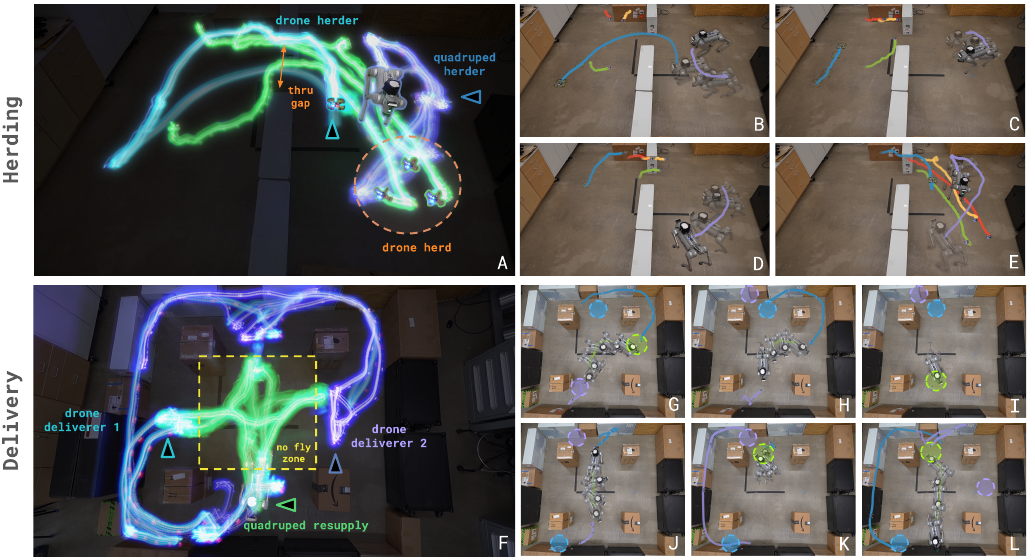}
    \captionof{figure}{\textbf{Trajectory snapshots from \texttt{Herding} and \texttt{Delivery} hardware tasks}. We show a long-exposure photo (left), and
    stills from independent times (right), with depictions corresponding to Fig.~\ref{fig:hardwareoverview}. Videos can be seen at \url{https://willsharpless.github.io/valdec-site/}.}
    \label{fig:hardware_rollouts}
    \vspace{-2em}
\end{center}
\end{figure*}

\subsection{\texttt{Delivery}}

In this experiment, we consider a team of two CFs and the Go2 tasked with recurrently visiting agent-specific target locations and recurrently revisiting the Go2 agent (to model package delivery and resupply), while avoiding building obstacles, collisions, and a ``no fly zone'' (for the CFs).

The TL specification for the task is given by,
$$\mathsf{p}_{\text{delivery}} := \mathsf{G}(\mathsf{F}
\textcolor[rgb]{0.26, 0.52, 0.96}{\mathsf{r}_1} \land \mathsf{F}
\textcolor[rgb]{0.68, 0.52, 0.96}{\mathsf{r}_2}) \land \bigwedge\nolimits_i \mathsf{GF}(
\textcolor[rgb]{0.34, 0.82, 0.45}{\mathsf{rs}_i}) \land \mathsf{G}\lnot(\textcolor[rgb]{0.8, 0.4, 0.4}{\mathsf{ac}} \lor \textcolor[rgb]{0.8, 0.4, 0.4}{\mathsf{ob}} \lor \textcolor[HTML]{c6c103}{\mathsf{nf}})
\vspace{-.5em}$$
where the predicate $\textcolor[rgb]{0.26, 0.52, 0.96}{\mathsf{r}_1}$ \& $
\textcolor[rgb]{0.68, 0.52, 0.96}{\mathsf{r}_2}$ capture CF $i$ visiting target $i$, $\textcolor[rgb]{0.34, 0.82, 0.45}{\mathsf{rs}_i}$ captures CF $i$ visiting the Go2 (resupplying), $\textcolor[rgb]{0.8, 0.4, 0.4}{\mathsf{ac}}$ captures aerial collision, $\textcolor[rgb]{0.8, 0.4, 0.4}{\mathsf{ob}}$ captures obstacle collision, and  $\textcolor[HTML]{c6c103}{\mathsf{nf}}$ captures the no-fly-zone (for the CFs only). Here, the task logic is dominated by $\mathsf{GF}$, and hence is largely solved with the $\mathcal{R}\mathcal{A}_\ell$-BE.

In this environment, the CF targets jump to a new random location after an agent has visited it, requiring a policy that is conditioned to various target locations. The real difficulty of this problem arises in the tightness of the layout; the obstacles confine the Go2 to the central area where the CFs are not allowed to fly (modeling a busy intersection), yet the CFs must visit the Go2 to ``resupply.'' We again implement \texttt{VDPPO} to learn a policy to solve the complex task and deploy it live.

Ultimately, we observe sophisticated coordination between the three agents to \textbf{distribute the difficulty of the task fairly}. Namely, as the CFs move around the outskirts of the arena, avoiding one another carefully but not slowly (Fig.\ref{fig:hardware_rollouts}.L), the Go2 anticipates their movements, moving between each of the agents (Fig.\ref{fig:hardware_rollouts}.G-I) to be in position to resupply them as close to their target as possible. This complex collaboration generated by \texttt{VDPPO} allows the agents to rapidly meet deliveries and resupply without crashing at all.

\section{Discussion and Limitations}

We briefly discuss some limitations of the proposed work and future directions. 

\textit{1. Any-Order Task Scaling:}
Thm.~\ref{thm:n-ra} yields a graph of $2^N$ Values, i.e. the compute for a simultaneous any-order task grows exponentially. This Traveling Salesman Problem like growth is fundamental and also arises in automaton construction \cite{cichon2009minimal}. Thus, LCRL and other RL-automata methods share this limitation. Approximations that only consider a subgraph could mitigate this, for example, but we leave this to future work.

\textit{2. Running Cost Minimization:}
\texttt{VDPPO} does not minimize a running cost, but can be extended to do so (i.e., a constrained
optimization problem formulation) via Lagrangian \cite{ganai2023iterative} or epigraph formulations \cite{So-RSS-23,so2024solving}. Similarly, \texttt{VDPPO} might be extended to generate Lyapunov/Barrier functions with an anti-discounting term \cite{Choi2021CBVF}. We leave this to future work.

\textit{3. Optimizing Robustness, Tradeoffs:}
Dense robustness scores can induce unintended behaviors while maintaining satisfaction (e.g., reaching the target then moving far away to increase safety margin). However, the presented theory is amenable to any predicate function with the same level set, e.g. allowing for tuning like $\textcolor[rgb]{0.365,0.494,0.713}{r}':=10\textcolor[rgb]{0.365,0.494,0.713}{r}$ without sacrificing ``priority'' of $\mathsf{\textcolor[rgb]{0.365,0.494,0.713}{r}}$.
Conveniently, this allows for both dense \& sparse reward (and penalty) functions; in our results, we mostly use sparse/binary predicate functions and employ pseudo-binary dense functions for rare tasks. 

\textit{4. Stochasticity:}
The presented decompositions depend on $\max$/$\min$ algebra which does not hold when wrapped with the Expectation operator (all equivalences become inequalities). As discussed in \cite{so2024solving}, a policy gradient result is still derivable, however, and has been demonstrated to work in practice \cite{sharpless2025dual}.


\section{Conclusion}

In this work, we propose a novel approach to solving the Bellman Value function associated with complex temporal specifications via decomposition. Namely, we demonstrate that for a large class of TL predicates, the corresponding Value may be decomposed into a graph of Values connected by a set of ``atomic'' Bellman equations. With this perspective, we propose \texttt{VDPPO} that is shown to solve optimal policies in complex tasks well beyond existing methods. This work highlights a novel and powerful approach to tackling complex task logic for real-world autonomy.

\section*{Acknowledgements}
The authors would like to thank Sander Tonkens, Gioele Zardini, and Jaime Fisac for the valuable discussions that lead to the development of the work. This work was supported by the National Science Foundation (NSF) CAREER Award \#CCF-2238030 and Office of Naval Research Award (ONR) ONR N000142412661.


\bibliographystyle{IEEEtran}
\bibliography{references,refs_sharpless_iclr26}

\newpage

\setlength\abovecaptionskip{0.5\baselineskip}
\onecolumn 

\setcounter{section}{0}
\renewcommand{\thesection}{\Alph{section}}
\renewcommand{\theHsection}{appendix.\Alph{section}}

\setcounter{subsection}{0}
\renewcommand{\thesubsection}{\thesection.\arabic{subsection}}
\renewcommand{\theHsubsection}{appendix.\Alph{section}.\arabic{subsection}}

\crefalias{section}{appendix}
\crefalias{subsection}{appendix}
\crefalias{subsubsection}{appendix}

\titleformat{\section}
  {\Large\bfseries}
  {\thesection}
  {1em}
  {}

\titleformat{\subsection}
  {\large\bfseries}
  {\thesubsection}
  {1em}
  {}

\makeatletter
\renewcommand\@seccntformat[1]{%
  \csname the#1\endcsname\quad
}
\makeatother

{\Large \center \textsc{APPENDIX}\par}

{\large \textsc{Contents}\par}
\vspace{1em}

\makeatletter
\newcommand{\apxtoc@numberline}[1]{#1\quad}
\newcommand{\apxtoc@section}[2]{%
  \noindent\begin{tabular*}{\columnwidth}{@{}l@{\extracolsep{\fill}}r@{}}
  \begingroup\let\numberline\apxtoc@numberline\textsc{#1}\endgroup & #2
  \end{tabular*}\\[0.5em]
}
\newcommand{\apxtoc@subsection}[2]{%
  \noindent\begin{tabular*}{\columnwidth}{@{}l@{\extracolsep{\fill}}r@{}}
  \hspace*{1.5em}\begingroup\let\numberline\apxtoc@numberline\textsc{#1}\endgroup & #2
  \end{tabular*}\\[0.5em]
}
\begingroup
  \let\l@section\apxtoc@section
  \let\l@subsection\apxtoc@subsection
  \@starttoc{apxtoc}
\endgroup

\let\apxtoc@addcontentsline\addcontentsline
\newcommand{\apxtoc@maybeaddcontentsline}[3]{%
  \def\apxtoc@file{#1}%
  \def\apxtoc@type{#2}%
  \def\apxtoc@toc{toc}%
  \def\apxtoc@sectiontype{section}%
  \def\apxtoc@subsectiontype{subsection}%
  \ifx\apxtoc@file\apxtoc@toc
    \ifx\apxtoc@type\apxtoc@sectiontype
      \apxtoc@addcontentsline{apxtoc}{#2}{#3}%
    \else\ifx\apxtoc@type\apxtoc@subsectiontype
      \apxtoc@addcontentsline{apxtoc}{#2}{#3}%
    \fi\fi
  \fi
}
\renewcommand{\addcontentsline}[3]{%
  \apxtoc@addcontentsline{#1}{#2}{#3}%
  \apxtoc@maybeaddcontentsline{#1}{#2}{#3}%
}
\makeatother

\newpage
\section*{Useful Properties and Notation}

We give here properties and notations for simplifying the following proofs. For a given action sequence $\alpha$, 
$$\alpha := (a_1, a_2, \dots ) \in \mathbb{A} := \mathcal{A}^\mathbb{N} $$
let a portion beginning at $i$ and ending at $j$ be written $$\alpha_{i:j} := (a_i, \dots a_j).$$
Moreover, for a trajectory $\xi_x^\alpha$,
$$\xi_x^\alpha := (x, x_1, \dots) \in \mathbb{X} := \mathcal{X}^\mathbb{N},$$
where $x_{i+1}=f(x_i, \alpha_i)$, it follows then that for $\alpha$ divided into $\alpha_{t^-} :=\alpha_{1:t}$ \& $\alpha_{t^+} :=\alpha_{t+1:\infty}$,
$$\xi_x^\alpha = \xi_{y}^{\alpha_{t^+}},$$
where $y=\xi_x^{\alpha_{t^-}}(t)$. We then have the following result corresponding to the decomposition of a controlled trajectory, which will be used ubiquitously.

\begin{lemma} \label{lem:ctrlsplit}
    Let $\mathcal{X}$ s.t. $|\mathcal{X}| < \infty$. Then for $t \in \mathbb{N}$, $\alpha \in \mathbb{A}$, $\xi_x^\alpha \in \mathbb{X}$, and, $x \in \mathcal{X}$,
    $$
    \max_{\alpha} \max_{t} f(\xi_x^\alpha, t) = \max_t \max_{\alpha_{t^-}} \max_{\alpha_{t^+}} f(\xi_{\xi_x^{\alpha_{t^-}}(t)}^{\alpha_{t^+}},t).
    $$
\end{lemma}

\newpage

\section{More Related Works} \label{apx:related}
We here give a slightly more expanded description of the related works compared to the main text.
We refer the reader to \cite{sharpless2025dual} for additional discussion of many of these works. 

\textbf{Reinforcement Learning with TL Objectives.} 
Many works have explored ways to optimize objectives that encode TL specifications \cite{STL1,STL2,STL3,STL13,Rabin-Automaton-Sastry,Bozkurt_2020, concurrent-learning} (or conversely learn TL specifications from agent behavior \cite{interpretable-apprenticeship}). One line of such works uses Non-Markovian-Reward Decision Processes (NMRDPs), which allow for history-dependent rewards \cite{LTL-and-beyond,Decision-theoretic-planning,Guiding-search-LTL,reward-machines,Gaon_Brafman_2020}. 
Other works optimize the quantitative semantics associated with an STL objective, approximating the maximums and minimums in a sum-of-discounted rewards fashion, which are then solved with traditional methods \cite{STL5,STL14}, or otherwise encoding TL objectives through expectations \cite{Cai_2021}.
Several other methods also exist that attempt to optimize general objective functions using non-traditional Bellman equations \cite{cdby1,cdby2,cdby3} or handle discounted sums of multiple rewards or penalties \cite{cdby4,cdby5,cdby6}.
We also refer the reader to \cite{compositional-reinforcement-learning} for an approach that proceeds by composing learned sub-tasks into higher level ones using an additional planning algorithms rather than breaking a high-level task down into subtasks.
By contrast to most of these previous approaches, our approach proceeds by decomposition of a TL-specified problem in an exact manner.
Specifically, we decompose the value function associated with a quantitative semantic for a TL predicate into value functions associated with simpler objectives.
These simpler objectives are then solved by leveraging powerful recent Hamilton-Jacobi Reachability (HJR) methods.
(Note that these decompositions of the value functions are fundamentally different from decompositions of the quantitative semantics themselves.)
This approach allows one to avoid approximations of the objective function or issues associated with sparsity of long-horizon rewards, which commonly afflict the previous methods.

\textbf{Constrained, Multi-Objective, and Goal-Conditioned RL}
A number of techniques in RL have arisen to handle constraints or multiple goals. Constrained MDPs (CMDPs) attempt to maximize sums of discounted rewards subject to a safety or liveness condition, which is often handled via a Lagrangian term in the objective function \cite{Altman-CMDPs,Abbeel-Constrained-Policy-Optimization,Safe-RL-CMDPs,Abbeel-Lagrangian-CMDP,CMDP3,CMDP4,CMDP5,CMDP6,CMDP7,CMDP11,CMDP12,CMDP13,pmlr-v168-castellano22a,pmlr-v238-mcmahan24a,So-RSS-23}.
For CMDPs, the Lagrangian term involved typically requires substantial tuning for desired behavior, severly limiting its use for satisfying general TL tasks.
Multi-objective RL techniques, by contrast Pareto-optimize multiple sums of discounted rewards \cite{Model-based-multi-objective-2014, Pareto-Dominating-Policies-2014,Distributional-Multi-Objective,Multi-Objective-2016,MORL2,Generalized-Algorithm-Multi-Objective-2019,Liu_2025}.
This allows users to balance multiple objectives, but generally are not built for handling TL-like specifications. 
Goal-conditioned RL, by contrast, simultaneously learns policies for a range of possible task specifications \cite{Goal-Conditioned-Problems-and-Solutions,Multi-Goal-Reinforcement-Learning,Exploration-via-Hindsight-GCRL,GCRL1,GCRL2,GCRL6,GCRL4,GCRL1,GCRL2,GCRL3}.
At the time of deployment, a user can then decide which specification is most appropriate. This is fundamentally different from TL tasks, where all specifications must be satisfied.

\textbf{Hamilton-Jacobi Reachability} 
Hamilton-Jacobi Reachability (HJR) methods were initially designed to solve value functions associated with "reach", "avoid", or "reach-avoid" problems using traditional dynamic programming for continuous space and times \cite{mitchell2005time,fisac2015reach}.
The objectives for these tasks are precisely the quantitative semantics for eventually, never, and until predicates.
HJR approaches have recently been adapted to solve these same problems in RL settings, with exciting performance \cite{so2024solving,hsu2021safety,fisac2019bridging,Ganai2023,yu2022reachability,Zhu2024-ck}.
Our work builds on such advancements, using the RL algorithms developed by these building blocks to accomplish higher-level tasks.

\newpage
\section{Temporal Logic} \label{apx:tlbg}
In this section, we give further background on the temporal logic used in the main text. We begin with the logical definitions of the operators $\lor, \land, \lnot, \mathsf{X}, \mathsf{F}, \mathsf{G}, \mathsf{U}$, alternatively defined by their robustness metric in the main text. 

\begin{tcolorbox}[DefFrame]
    \begin{defi}
        Let $\mathsf{p},\mathsf{p'}$ be predicates, $\xi_x \in \mathbb{X}$ a trajectory beginning at $x \in \mathcal{X}$, and $t \in \mathbb{N}$ a starting time. The relation $(\xi_x,t)\models\mathsf{p}$ is defined as follows,
        \begin{align*}
        (\xi_x,t)\models\mathsf{r}_i &\iff r_i(\xi_x(t))\ge 0, \\
        (\xi_x,t)\models\lnot\mathsf{p} &\iff (\xi_x,t)\not\models\mathsf{p}, \\
        (\xi_x,t)\models\mathsf{p}\land\mathsf{p'} &\iff (\xi_x,t)\models\mathsf{p}\ \text{and}\ (\xi_x,t)\models\mathsf{p'}, \\
        (\xi_x,t)\models\mathsf{p}\lor\mathsf{p'} &\iff (\xi_x,t)\models\mathsf{p}\ \text{or}\ (\xi_x,t)\models\mathsf{p'}, \\
        (\xi_x,t)\models\mathsf{X}\mathsf{p} &\iff (\xi_x,t+1)\models\mathsf{p}, \\
        (\xi_x,t)\models\mathsf{F}\mathsf{p} &\iff \exists \tau\ge t\ \text{s.t.}\ (\xi_x,\tau)\models\mathsf{p}, \\
        (\xi_x,t)\models\mathsf{G}\mathsf{p} &\iff \forall \tau\ge t,\ (\xi_x,\tau)\models\mathsf{p}, \\
        (\xi_x,t)\models\mathsf{p}\,\mathsf{U}\,\mathsf{p'} &\iff \exists \tau\ge t\ \text{s.t.}\ (\xi_x,\tau)\models\mathsf{p'}\ \text{and}\ \\ & \qquad \qquad \forall \kappa\in[t,\tau],\ (\xi_x,\kappa)\models\mathsf{p}.
        \end{align*}
    \end{defi}
\end{tcolorbox}

From these definitions, we may certify a few equivalence relations for rearranging certain combinations of operators, which will later prove to be useful. Note, for the interested reader all of the following equivalences may be automatically verified with the tool Spot \cite{duret2022spot}.

\begin{tcolorbox}[DefFrame]
    \begin{lemma} \label{lem:symbolic:U_and}
        \begin{align*}
            (\mathsf{q}_1 \,\mathsf{U}\, \mathsf{r}_1) \land& (\mathsf{q}_2 \,\mathsf{U}\, \mathsf{r}_2) \equiv \\ 
            &(\mathsf{q}_1 \land \mathsf{q}_2) \,\mathsf{U}\, \bigl( (\mathsf{r}_1 \land \mathsf{q}_2 \,\mathsf{U}\, \mathsf{r}_2) \lor (\mathsf{r}_2 \land \mathsf{q}_1 \,\mathsf{U}\, \mathsf{r}_1) \bigr)
        \end{align*}
    \end{lemma}
\end{tcolorbox}
\proof
    We show this via double entailment.

    \,

    \noindent{{{1. LHS $\models$ RHS}}}:

    Suppose $\sigma, 0 \models \mathsf{q}_1 \,\mathsf{U}\, \mathsf{r}_1 \land \mathsf{q}_2 \,\mathsf{U}\, \mathsf{r}_2$. Then,
    \begin{enumerate}
        \item Since $\sigma, 0 \models \mathsf{q}_1 \,\mathsf{U}\, \mathsf{r}_1$, there exists $t_1 \geq 0$ such that $\sigma, t_1 \models \mathsf{r}_1$, and for all $0 \leq k < t_1$, $\sigma, k \models \mathsf{q}_1$.
        \item Since $\sigma, 0 \models \mathsf{q}_2 \,\mathsf{U}\, \mathsf{r}_2$, there exists $t_2 \geq 0$ such that $\sigma, t_2 \models \mathsf{r}_2$, and for all $0 \leq k < t_2$, $\sigma, k \models \mathsf{q}_2$.
    \end{enumerate}
    Let $t = \min(t_1, t_2)$. Since $\sigma, k \models \mathsf{q}_1$ and $\sigma, k \models \mathsf{q}_2$ for all $0 \leq k < t$, we have $\sigma, k \models \mathsf{q}_1 \land \mathsf{q}_2$ for all $0 \leq k < t$.
    
    We now show that the goal is reached at time $t$.
    \begin{itemize}
        \item ($t_1 \leq t_2$): Then, $t = t_1$, and $\sigma, t \models \mathsf{r}_1$. Since $t_2 \geq t_1$ and $\sigma, k \models \mathsf{q}_2$ for all $t \leq k < t_2$, we have $\sigma, t \models \mathsf{q}_2 \,\mathsf{U}\, \mathsf{r}_2$. Hence, $\sigma, t \models \mathsf{r}_1 \land \mathsf{q}_2 \,\mathsf{U}\, \mathsf{r}_2$.
        \item ($t_2 < t_1$): Then, $t = t_2$, and $\sigma, t \models \mathsf{r}_2$. Since $t_1 > t_2$ and $\sigma, k \models \mathsf{q}_1$ for all $t \leq k < t_1$, we have $\sigma, t \models \mathsf{q}_1 \,\mathsf{U}\, \mathsf{r}_1$. Hence, $\sigma, t \models \mathsf{r}_2 \land \mathsf{q}_1 \,\mathsf{U}\, \mathsf{r}_1$.
    \end{itemize}
    Thus, $\sigma, 0 \models (\mathsf{q}_1 \land \mathsf{q}_2) \,\mathsf{U}\, \bigl( (\mathsf{r}_1 \land \mathsf{q}_2 \,\mathsf{U}\, \mathsf{r}_2) \lor (\mathsf{r}_2 \land \mathsf{q}_1 \,\mathsf{U}\, \mathsf{r}_1) \bigr)$.

    \,

    \noindent{{{2. RHS $\models$ LHS}}}:

    Suppose $\sigma, 0 \models (\mathsf{q}_1 \land \mathsf{q}_2) \,\mathsf{U}\, \bigl( (\mathsf{r}_1 \land \mathsf{q}_2 \,\mathsf{U}\, \mathsf{r}_2) \lor (\mathsf{r}_2 \land \mathsf{q}_1 \,\mathsf{U}\, \mathsf{r}_1) \bigr)$.
    Then, there exists $t \geq 0$ such that 
    \begin{itemize}
        \item $\sigma, t \models (\mathsf{r}_1 \land \mathsf{q}_2 \,\mathsf{U}\, \mathsf{r}_2) \lor (\mathsf{r}_2 \land \mathsf{q}_1 \,\mathsf{U}\, \mathsf{r}_1)$
        \item For all $0 \leq k < t$, $\sigma, k \models \mathsf{q}_1 \land \mathsf{q}_2$.
    \end{itemize}
    
    We now split into two cases.
    \begin{enumerate}
        \item ($\sigma, t \models \mathsf{r}_1 \land \mathsf{q}_2 \,\mathsf{U}\, \mathsf{r}_2$):
        \begin{itemize}
            \item $\sigma, t \models \mathsf{r}_1$
            \item Since $\sigma, k \models \mathsf{q}_1$ for all $0 \leq k < t$, we have $\sigma, 0 \models \mathsf{q}_1 \,\mathsf{U}\, \mathsf{r}_1$.
            \item There exists $t_2 \geq t$ such that $\sigma, t_2 \models \mathsf{r}_2$, and $\sigma, k \models \mathsf{q}_2$ for all $t \leq k < t_2$.
            \item Since $\sigma, k \models \mathsf{q}_2$ for all $0 \leq k < t_2$, we have $\sigma, 0 \models \mathsf{q}_2 \,\mathsf{U}\, \mathsf{r}_2$.
            \item Thus, $\sigma, 0 \models \mathsf{q}_1 \,\mathsf{U}\, \mathsf{r}_1 \land \mathsf{q}_2 \,\mathsf{U}\, \mathsf{r}_2$.
        \end{itemize}

        \item ($\sigma, t \models \mathsf{r}_1 \land \mathsf{q}_2 \,\mathsf{U}\, \mathsf{r}_2$):
        The reasoning is symmetric to the previous case, yielding $\sigma, 0 \models \mathsf{q}_1 \,\mathsf{U}\, \mathsf{r}_1 \land \mathsf{q}_2 \,\mathsf{U}\, \mathsf{r}_2$.
    \end{enumerate}
    Thus, $\sigma, 0 \models \mathsf{q}_1 \,\mathsf{U}\, \mathsf{r}_1 \land \mathsf{q}_2 \,\mathsf{U}\, \mathsf{r}_2$.

    Since we have shown both directions, the equivalence holds.
\endproof

\begin{tcolorbox}[DefFrame]
    \begin{lemma} \label{lem:symbolic:n_U_and}
        \begin{equation*}
            \mathsf{p} := \bigwedge_{i=1}^n (\mathsf{q}_i \,\mathsf{U}\, \mathsf{r}_i) \equiv \Bigl( \bigwedge_{i=1}^n \mathsf{q}_i \Bigr) \,\mathsf{U}\, \Bigl( \bigvee_{i=1}^n \bigl( \mathsf{r}_i \land \mathsf{p}^{-i} \bigr) \Bigr)
        \end{equation*}
        where $\mathsf{p}^{-i} := \bigwedge_{j=1, j \neq i}^n (\mathsf{q}_j \,\mathsf{U}\, \mathsf{r}_j)$.
    \end{lemma}
\end{tcolorbox}
\proof
    We prove this using induction on $n$.
    
    \,

    \noindent{Base Case ($n=2$):} This is exactly the previous lemma \ref{lem:symbolic:U_and}.

    \,

    \noindent{Inductive Step:} Assume the statement holds for $n = k$, i.e.,
    \begin{equation*}
        \bigwedge_{i=1}^k (\mathsf{q}_i \,\mathsf{U}\, \mathsf{r}_i) \equiv \underbrace{\Bigl( \bigwedge_{i=1}^k \mathsf{q}_i \Bigr)}_{\coloneqq \tilde{\mathsf{q}}} \,\mathsf{U}\, \underbrace{\Bigl( \bigvee_{i=1}^k \bigl( \mathsf{r}_i \land \bigwedge_{j=1, j \neq i}^n (\mathsf{q}_j \,\mathsf{U}\, \mathsf{r}_j) \bigr) \Bigr)}_{\coloneqq \tilde{\mathsf{r}}}.
    \end{equation*}
    We need to show it holds for $n = k + 1$.
    \begin{align*}
        \bigwedge_{i=1}^{k+1} (\mathsf{q}_i \,\mathsf{U}\, \mathsf{r}_i)
        &= \Bigl( \bigwedge_{i=1}^k (\mathsf{q}_i \,\mathsf{U}\, \mathsf{r}_i) \Bigr) \land (\mathsf{q}_{k+1} \,\mathsf{U}\, \mathsf{r}_{k+1}) \\
        &\equiv \tilde{\mathsf{q}} \,\mathsf{U}\, \tilde{\mathsf{r}} \land (\mathsf{q}_{k+1} \,\mathsf{U}\, \mathsf{r}_{k+1}) \\
        &\equiv ( \tilde{\mathsf{q}} \land \mathsf{q}_{k+1} ) \,\mathsf{U}\, \\ & \qquad ( ( \tilde{\mathsf{r}} \land \mathsf{q}_{k+1} \,\mathsf{U}\, \mathsf{r}_{k+1} ) \lor ( \mathsf{r}_{k+1} \land \tilde{\mathsf{q}} \,\mathsf{U}\, \tilde{\mathsf{r}} ) )
    \end{align*}
    Note that $\tilde{\mathsf{q}} \land \mathsf{q}_{k+1} = \bigwedge_{i=1}^{k+1} \mathsf{q}_i$. For the first part,
    \begin{align*}
        \tilde{\mathsf{r}} \land \mathsf{q}_{k+1} \,\mathsf{U}\, \mathsf{r}_{k+1}
        &= \bigvee_{i=1}^k \Bigl( \mathsf{r}_i \land \bigwedge_{j=1, j \neq i}^k (\mathsf{q}_j \,\mathsf{U}\, \mathsf{r}_j) \Bigr) \land \mathsf{q}_{k+1} \,\mathsf{U}\, \mathsf{r}_{k+1} \\
        &= \bigvee_{i=1}^k \Bigl( \mathsf{r}_i \land \bigwedge_{j=1, j \neq i}^k (\mathsf{q}_j \,\mathsf{U}\, \mathsf{r}_j)  \land \mathsf{q}_{k+1} \,\mathsf{U}\, \mathsf{r}_{k+1} \Bigr) \\
        &= \bigvee_{i=1}^k \Bigl( \mathsf{r}_i \land \bigwedge_{j=1, j \neq i}^{k+1} (\mathsf{q}_j \,\mathsf{U}\, \mathsf{r}_j)  \land \mathsf{q}_{k+1} \,\mathsf{U}\, \mathsf{r}_{k+1} \Bigr).
    \end{align*}
    For the second part,
    \begin{align*}
        \mathsf{r}_{k+1} \land \tilde{\mathsf{q}} \,\mathsf{U}\, \tilde{\mathsf{r}}
        &= \mathsf{r}_{k+1} \land \bigwedge_{i=1}^k ( \mathsf{q}_i \,\mathsf{U}\, \mathsf{r}_i ), \\
        &= \mathsf{r}_{k+1} \land \bigwedge_{j=1, j \neq k+1}^{k+1} ( \mathsf{q}_j \,\mathsf{U}\, \mathsf{r}_j ).
    \end{align*}
    Combining these two parts completes the inductive step:
    \begin{equation*}
        \bigvee_{i=1}^{k+1} \Bigl( \mathsf{r}_i \land \bigwedge_{j=1, j \neq i}^{k+1} (\mathsf{q}_j \,\mathsf{U}\, \mathsf{r}_j) \Bigr).
    \end{equation*}

    Since the base case and inductive step hold, the statement holds for all $n \geq 2$.
\endproof

\begin{tcolorbox}[DefFrame]
    \begin{coroll} \label{lem:symbolic:U_and_G}
        \begin{equation*}
            \mathsf{p} := \bigwedge_{i=1}^n (\mathsf{q}_i \,\mathsf{U}\, \mathsf{r}_i) \land \mathsf{Gq} \equiv \Bigl( \bigwedge_{i=1}^n \mathsf{q}_i \land \mathsf{q}\Bigr) \,\mathsf{U}\, \Bigl( \bigvee_{i=1}^n \bigl( \mathsf{r}_i \land \mathsf{p}^{-i} \bigr) \Bigr)
        \end{equation*}
        where $\mathsf{p}^{-i} := \bigwedge_{j=1, j \neq i}^n (\mathsf{q}_j \,\mathsf{U}\, \mathsf{r}_j) \land \mathsf{Gq}$.
    \end{coroll}
\end{tcolorbox}
\proof
    It suffices to show that $\mathsf{Gq} = \mathsf{q} \, \mathsf{U} \, \tilde{\mathsf{r}}$ where $\tilde{\mathsf{r}} = \mathsf{Gq}$. This follows directly from the definition of $\mathsf{G}$ and $\mathsf{U}$,
    \begin{align*}
        \sigma, 0 \models \mathsf{Gq} &\iff \forall t \geq 0, \sigma, t \models \mathsf{q} \\
        &\iff \exists t' \geq 0 \text{ s.t. } \sigma, t' \models {\mathsf{Gq}} \\ 
        & \qquad \qquad \text{ and } \forall 0 \leq t < t', \sigma, t \models \mathsf{q} \\
        &\iff \sigma, 0 \models \mathsf{q} \, \mathsf{U} \, \tilde{\mathsf{r}}.
\end{align*}
\endproof

Additionally, we can show this kind of rearrangement for the $\mathsf{GU}$ composition as well, given by the following result.

\begin{tcolorbox}[DefFrame]
    \begin{lemma} \label{lem:symbolic:G_of_U}
        \begin{equation*}
            \mathsf{G} (\mathsf{q} \, \mathsf{U} \, \mathsf{r}) \equiv \mathsf{q} \, \mathsf{U} \, (\mathsf{r} \land \mathsf{X}\mathsf{G} (\mathsf{q} \, \mathsf{U} \, \mathsf{r}))
        \end{equation*}
    \end{lemma}
\end{tcolorbox}
\proof
    We show this via double entailment.
    
    \,

    \noindent{1. (LHS $\models$ RHS)}
    Suppose $\sigma, 0 \models \mathsf{G} (\mathsf{q} \, \mathsf{U} \, \mathsf{r})$.
    \begin{itemize}
        \item For all $t \geq 0$, there exists $s_t \geq t$ such that $\sigma, s_t \models \mathsf{r}$ and $\forall 0 \leq t' < s_t, \sigma, t' \models \mathsf{q}$. In particular, for $t = 0$, there exists $s_0 \geq 0$ such that $\sigma, s_0 \models \mathsf{r}$.
        \item Since $\mathsf{G} (\mathsf{q} \, \mathsf{U} \, \mathsf{r})$ is a tail property, we have $\sigma, s_0 + 1 \models \mathsf{G} (\mathsf{q} \, \mathsf{U} \, \mathsf{r})$.
        \item Thus, $\sigma, s_0 \models \mathsf{r} \land \mathsf{X} \mathsf{G} (\mathsf{q} \, \mathsf{U} \, \mathsf{r})$. 
        \item Hence, $\sigma, 0 \models \mathsf{q} \, \mathsf{U} \, (\mathsf{r} \land \mathsf{X} \mathsf{G} (\mathsf{q} \, \mathsf{U} \, \mathsf{r}))$.
    \end{itemize}

    \,

    \noindent{2. (RHS $\models$ LHS)}
    Suppose $\sigma, 0 \models \mathsf{q} \, \mathsf{U} \, (\mathsf{r} \land \mathsf{X} \mathsf{G} (\mathsf{q} \, \mathsf{U} \, \mathsf{r}))$.
    \begin{itemize}
        \item By definition of $\mathsf{U}$, there exists $t_0 \geq 0$ such that $\sigma, t_0 \models \mathsf{r} \land \mathsf{X} \mathsf{G} (\mathsf{q} \, \mathsf{U} \, \mathsf{r})$ s.t. $\forall 0 \leq t < t_0, \sigma, t \models \mathsf{q}$.
        \item The conjunction implies that $\sigma, t_0 + 1 \models \mathsf{G} (\mathsf{q} \, \mathsf{U} \, \mathsf{r})$.
        \item Since $\mathsf{G} (\mathsf{q} \, \mathsf{U} \, \mathsf{r})$ is a tail property, this implies that $\sigma, 0 \models \mathsf{G} (\mathsf{q} \, \mathsf{U} \, \mathsf{r})$.
    \end{itemize}

    Since we have shown both directions, the equivalence holds.
\endproof

Next, we may extend this to the multi-Until case, in order to capture the behavior of multiple recurrent Until operators. Notably, in this case, the order does not matter, as all must be satisfied infinitely often. This is formalized in the following result.

\begin{tcolorbox}[DefFrame]
\begin{lemma}\label{lem:symbolic:G_of_double_U}
Given
$$
\mathsf{p}
:=
\mathsf{G}\big(
(\mathsf{q}_1 \, \mathsf{U} \, \mathsf{r}_1)
\land
(\mathsf{q}_2 \, \mathsf{U} \, \mathsf{r}_2)
\big),
$$
it holds that
\begin{align*}
\mathsf{p}
&\equiv
\tilde{\mathsf{q}}_1
\, \mathsf{U} \,
\Big(
\tilde{\mathsf{r}}_1
\land
\big(
\tilde{\mathsf{q}}_2
\, \mathsf{U} \,
(\tilde{\mathsf{r}}_2 \land \mathsf{p})
\big)
\Big) \\
&\equiv
\tilde{\mathsf{q}}_2
\, \mathsf{U} \,
\Big(
\tilde{\mathsf{r}}_2
\land
\big(
\tilde{\mathsf{q}}_1
\, \mathsf{U} \,
(\tilde{\mathsf{r}}_1 \land \mathsf{p})
\big)
\Big),
\end{align*}
where $\tilde{\mathsf{q}}_i := \mathsf{q}_i \land (\mathsf{q}_j \lor \mathsf{r}_j)$, $\tilde{\mathsf{r}}_i := \mathsf{r}_i \land (\mathsf{q}_j \lor \mathsf{r}_j)$.
\end{lemma}
\end{tcolorbox}

\proof
We show this via double entailment. For brevity, let $\mathsf{w}_1 := (\mathsf{q}_1 \lor \mathsf{r}_1)$,
$\mathsf{w}_2 := (\mathsf{q}_2 \lor \mathsf{r}_2)$.

\,

\noindent{1. (LHS $\models$ RHS)}
Assume $\sigma,0 \models \mathsf{p}$.
\begin{itemize}
\item For all $t \ge 0$,
$\sigma,t \models
(\mathsf{q}_1 \mathsf{U} \mathsf{r}_1)
\land
(\mathsf{q}_2 \mathsf{U} \mathsf{r}_2)$.
Choose $k_1 \ge 0$ with
$\sigma,k_1 \models \mathsf{r}_1$
and
$\sigma,t \models \mathsf{q}_1$ for $t<k_1$.
Then $\sigma,t \models \mathsf{w}_2$ for $t \le k_1$,
so $\sigma,t \models \mathsf{q}_1 \land \mathsf{w}_2$ for $t<k_1$.
\item From $\sigma,k_1 \models
\mathsf{q}_2 \mathsf{U} \mathsf{r}_2$,
choose $k_2 \ge k_1$ with
$\sigma,k_2 \models \mathsf{r}_2$
and
$\sigma,t \models \mathsf{q}_2$ for $k_1 \le t < k_2$.
Since $\mathsf{p}$ holds globally,
$\sigma,t \models \mathsf{w}_1$ on $[k_1,k_2]$
and $\sigma,k_2 \models \mathsf{p}$.
\end{itemize}

Thus
$
\sigma,k_1 \models
\tilde{\mathsf{q}}_2
\, \mathsf{U} \,
(\tilde{\mathsf{r}}_2 \land \mathsf{p}),
$
so $\sigma,0$ satisfies the RHS.

\,

\noindent{2. (RHS $\models$ LHS)}
Assume $\sigma,0$ satisfies the RHS.
\begin{itemize}
\item There exists $k_1 \ge 0$ with
$\sigma,k_1 \models
\tilde{\mathsf{r}}_1 \land \Psi$
and
$\sigma,t \models
\mathsf{q}_1 \land \mathsf{w}_2$
for $t<k_1$, where
$$
\Psi :=
\tilde{\mathsf{q}}_2
\, \mathsf{U} \,
(\tilde{\mathsf{r}}_2 \land \mathsf{p}).
$$
\item From $\Psi$ there exists $k_2 \ge k_1$
with
$\sigma,k_2 \models
\tilde{\mathsf{r}}_2 \land \mathsf{p}$
and
$\sigma,t \models
\mathsf{q}_2 \land \mathsf{w}_1$
for $k_1 \le t < k_2$.
\end{itemize}

Since $\sigma,k_2 \models \mathsf{p}$,
the property
$(\mathsf{q}_1 \mathsf{U} \mathsf{r}_1)
\land
(\mathsf{q}_2 \mathsf{U} \mathsf{r}_2)$
holds for all $t \ge k_2$.
Using the witnesses $k_1$ and $k_2$
and the safety conditions above,
it also holds for all $t<k_2$.
Hence $\sigma,0 \models \mathsf{p}$.

Both directions hold, so the equivalence follows.
\endproof

We now give a logical equivalence for the general class of predicates considered in Thm.~\ref{thm:master}. 

\newcommand{\IGUSet}{\mathcal{I}}
\newcommand{\qgu}{\mathsf{q}}
\newcommand{\rgu}{\mathsf{r}}

\newcommand{\IUSet}{\mathcal{J}}
\newcommand{\qu}{\mathsf{q}}
\newcommand{\ru}{\mathsf{r}}

\newcommand{\qg}{\mathsf{q}}

\begin{tcolorbox}[DefFrame]
\begin{lemma} \label{lem:symbolic:master_GU_U_G}
    Consider the formula $\mathsf{p}_{\IGUSet, \IUSet}$ defined as
    \begin{equation*}
        \mathsf{p}_{\IGUSet, \IUSet} \coloneqq \bigwedge_{i \in \IGUSet} \mathsf{G}(\qgu_i\, \mathsf{U} \, \rgu_i) \land \bigwedge_{j \in \IUSet} (\mathsf{q}_j\, \mathsf{U} \, \mathsf{r}_j) \land \mathsf{G} \qg
    \end{equation*}
    Then, $\mathsf{p}_{\IGUSet, \IUSet}$ can be equivalently written as a single nested until formula:
    \begin{equation*}
        \mathsf{p}_{\IGUSet, \IUSet}
        \equiv \tilde{\mathsf{q}}_{\IGUSet, \IUSet} \, \mathsf{U} \, \tilde{\mathsf{r}}_{\IGUSet, \IUSet}, \\
    \end{equation*}
    where
    \begin{align*}
        \tilde{\mathsf{q}}_{\IGUSet, \IUSet}
        &\coloneqq \bigwedge_{j \in \IUSet} \qu_j \land \qg \land \bigwedge_{i \in \IGUSet} (\qgu_i \lor \rgu_i), \\
        \tilde{\mathsf{r}}_{\IGUSet, \IUSet}
        &\coloneqq \bigvee_{j \in \IUSet} \ru_j \land \Phi_{\IGUSet, \IUSet \setminus \{j \} }
    \end{align*}
    and
    \begin{equation*}
        \mathsf{p}_{\IGUSet, \emptyset}
        \equiv \mathsf{G}\Bigl( \bigwedge_{i \in \IGUSet} \bigl( \qgu_i \land \qg \bigr)\, \mathsf{U} \, \bigl( \rgu_i \land \qg \bigr) \Bigr).
    \end{equation*}
\end{lemma}
\end{tcolorbox}
\proof
    We start by proving $\mathsf{p}_{\IGUSet, \emptyset}$. Then, we prove $\mathsf{p}_{\IGUSet, \IUSet}$.
    \begin{align*}
        \mathsf{p}_{\IGUSet, \emptyset}
        &\coloneqq \bigwedge_{i \in \IGUSet} \mathsf{G}(\qgu_i\, \mathsf{U} \, \rgu_i) \land \mathsf{G} \qg, \\
        &\equiv \mathsf{G}\Bigl( \bigwedge_{i \in \IGUSet} (\qgu_i\, \mathsf{U} \, \rgu_i) \land \qg \Bigr), \\
        &\equiv \mathsf{G}\Bigl( \bigwedge_{i \in \IGUSet} \bigl( \qgu_i \land \qg \bigr)\, \mathsf{U} \, \bigl( \rgu_i \land \qg \bigr)\Bigr).
    \end{align*}
    
    Now we prove $\mathsf{p}_{\IGUSet, \IUSet}$.
    Define $\tilde{\mathsf{r}}^{\text{U}, \IUSet}$ as the reward function obtained when applying the transformation to a conjunction of until formulas, i.e.,
    \begin{equation*}
        \tilde{\mathsf{r}}^{\text{U}, \IUSet} \coloneqq \bigvee_{j \in \IUSet} \Bigl\{ \ru_j \land \bigwedge_{j \in \IUSet \setminus \{j\}} (\qu_j\, \mathsf{U} \, \ru_j) \Bigr\}.
    \end{equation*}
    Then,
    \begin{align*}
        \mathsf{p}_{\IGUSet, \IUSet}
        &\coloneqq \bigwedge_{i \in \IGUSet} \mathsf{G}(\qgu_i\, \mathsf{U} \, \rgu_i) \land \bigwedge_{j \in \IUSet} (\mathsf{q}_j\, \mathsf{U} \, \mathsf{r}_j) \land \mathsf{G} \qg, \\
        &\equiv \bigwedge_{i \in \IGUSet} \mathsf{G}(\qgu_i\, \mathsf{U} \, \rgu_i) \land \tilde{\mathsf{q}}^{\text{U}, \IUSet}\, \mathsf{U} \, \tilde{\mathsf{r}}^{\text{U}, \IUSet} \land \mathsf{G} \qg, \\
        &\equiv \Bigl( \tilde{\mathsf{q}}^{\text{U}, \IUSet} \land \qg \land \bigwedge_{i \in \IGUSet} (\qgu_i \lor \rgu_i) \Bigr)\, \mathsf{U} \, \Bigl( \tilde{\mathsf{r}}^{\text{U}, \IUSet} \land \bigwedge_{i \in \IGUSet} \qgu_i\, \mathsf{U} \, \rgu_i \land \mathsf{G} \qg \Bigr).
    \end{align*}
    Examining the right argument of the $U$ operator, we see that
    \begin{align*}
        &\tilde{\mathsf{r}}^{\text{U}, \IUSet} \land \bigwedge_{i \in \IGUSet} \qgu_i\, \mathsf{U} \, \rgu_i \land \mathsf{G} \qg \\
        &= \Bigl( \bigvee_{j \in \IUSet} \Bigl\{ \ru_j \land \bigwedge_{j \in \IUSet \setminus \{j\}} (\qu_j\, \mathsf{U} \, \ru_j) \Bigr\} \Bigr) \land \bigwedge_{i \in \IGUSet} \qgu_i\, \mathsf{U} \, \rgu_i \land \mathsf{G} \qg, \\
        &= \bigvee_{j \in \IUSet} \Bigl( \ru_j \land \underbrace{\bigwedge_{j \in \IUSet \setminus \{j\}} (\qu_j\, \mathsf{U} \, \ru_j) \land \bigwedge_{i \in \IGUSet} \qgu_i\, \mathsf{U} \, \rgu_i \land \mathsf{G} \qg}_{\coloneqq \Phi_{\IGUSet, \IUSet \setminus \{j\}}} \Bigr).
    \end{align*}
    Plugging this back in completes the proof.
\endproof

\newpage
\section{Logic vs. Value Examples} \label{apx:counter}

In this section we reproduce an argument from \cite{sharpless2025dual} to demonstrate the following point: the algebraic relations that apply to the quantitative semantics in TL do not generally apply to the optimal value functions associated with the quantitative semantics.
Many previous works have explored and leveraged the algebraic relations dictating quantitative semantics, while we focus on building an algebra for the value functions.
An example highlighting the difference between the two is as follows.

Consider a reach-always-avoid (RAA) problem (i.e. reach a target set while avoiding an obstacle both before and after the target is reached) in which an agent would like to canoe across a river without hitting any rocks. Let $\mathsf{r}$ represent reaching the other side of the river and $\mathsf{q}$ represent not hitting a rock.
The TL formula for the RAA problem is then $\mathsf{Fr} \land \mathsf{Gg}$.
By definition, the following algebraic decomposition of the quantitative semantic for this predicate always holds:
\begin{equation}\label{apx-eqn:canoe-qs}
    \rho[\mathsf{Fr} \land \mathsf{Gg}](\xi_x^\alpha) = \min \left\{\rho[\mathsf{Fr}](\xi_x^\alpha), \rho[\mathsf{Gg}](\xi_x^\alpha) \right\}.
\end{equation}

However, the analogous relation does not generally hold for the optimal value functions.
To see this point, recall that
\begin{align*}
    V^{*}[\mathsf{Fr}](x) &:= \max_{\alpha} \rho[\mathsf{Fr}](\xi_x^{\alpha}),\\
    V^{*}[\mathsf{Gq}](x) &:= \max_{\alpha} \rho[\mathsf{Gq}](\xi_x^{\alpha}),\\
    V^{*}[\mathsf{Fr \land Gq}](x) &:= \max_{\alpha} \rho[\mathsf{Fr} \land \mathsf{Gq}](\xi_x^{\alpha})\\
    &= \max_\alpha \min \left\{\rho[\mathsf{Fr}](\xi_x^\alpha), \rho[\mathsf{Gg}](\xi_x^\alpha) \right\}.
\end{align*}
It is always the case that 
\begin{align*}
    \max_\alpha & \min \left\{\rho[\mathsf{Fr}](\xi_x^\alpha), \rho[\mathsf{Gg}](\xi_x^\alpha) \right\}\\
    &\le \min \left\{ \max_\alpha \rho[\mathsf{Fr}](\xi_x^\alpha), \max_\alpha \rho[\mathsf{Gg}](\xi_x^\alpha)  \right\},
\end{align*}
so that
\begin{equation}\label{apx-eqn:canoe-ovf}
     V^{*}[\mathsf{Fr \land Gq}](x) \le \min\left\{V^{*}[\mathsf{Fr}](x), V^{*}[\mathsf{Gq}](x) \right\}.
\end{equation}

By contrast with the equality in \ref{apx-eqn:canoe-qs}, the inequality in \ref{apx-eqn:canoe-ovf} may indeed by strict.
For example, suppose that I begin in a state $x$ for which I can either (a) stay still indefinitely in my current state or (b) get across the river while necessarily hitting a rock on the way. 
In this case $V^{*}[\mathsf{Fr}](x) \ge 0$ and $V^{*}[\mathsf{Gq}](x) \ge 0$, but $V^{*}[\mathsf{Fr \land Gq}](x) < 0$.

To summarize, even when an algebraic relation holds for the quantitative semantics of some predicate (as in \ref{apx-eqn:canoe-qs}), the corresponding algebraic relation for the optimal value functions may not hold.
Instead, the two expressions may at best be related by an inequality (as in \ref{apx-eqn:canoe-ovf}).
This observation motivates our work on algebraically rules for decomposing optimal value functions.

\newpage
\section{Proofs} \label{apx:proofs}
\subsection{Agreeable Algebra} \label{apx:aa-proof}
In this section, we certify the algebraic properties of Bellman Value functions that match those of logic, corresponding to Thm.~\ref{lem:equiv} from the main text, restated here for clarity. These will prove fundamental to the later derivations.
\begin{tcolorbox}[DefFrame]
    \lemmaequiv*
\end{tcolorbox}

\proof
We give a direct algebraic derivation of each property. Recall that we write $\rho(\xi_x^\alpha) := \rho(\xi_x^\alpha, 0)$ for brevity. 
We begin with the \texttt{right-Until} property using 
Lem.~\ref{lem:ctrlsplit},
\begin{eqnarray*}
\begin{aligned}
V^*&[\mathsf{a} \, \mathsf{U} \, \mathsf{b}](x) = \max_{\alpha} \max_{t} \min\{\rho[\mathsf{b}](\xi_x^{\alpha}, t), \min_{\kappa \in [0, t]} \rho[\mathsf{a}](\xi_x^{\alpha}, \kappa) \} \\
&= \max_{t} \max_{\alpha_{t^-}} \min\{\max_{\alpha_{t^+}} \rho[\mathsf{b}](\xi_{\xi^{\alpha_{t^-}}_x(t)}^{\alpha_{t^+}}, 0), \min_{\kappa \in [0, t]} \rho[\mathsf{a}](\xi_x^{\alpha_{t^-}}, \kappa) \} \\
&= \max_{t} \max_{\alpha_{t^-}} \min\{V^*[\mathsf{b}](\xi_x^{\alpha_{t^-}}(t)), \min_{\kappa \in [0, t]} \rho[\mathsf{a}](\xi_x^{\alpha_{t^-}}, \kappa) \} \\
&= \max_{t} \max_{\alpha} \min\{V^*[\mathsf{b}](\xi_x^{\alpha}(t)), \min_{\kappa \in [0, t]} \rho[\mathsf{a}](\xi_x^{\alpha}, \kappa) \} \\
&= V^*[\mathsf{a} \, \mathsf{U} \, \mathsf{v}_{\mathsf{b}}](x) \\
\end{aligned}
\end{eqnarray*}
Next, we prove the \texttt{OR} property, 
\begin{eqnarray*}
\begin{aligned}
V^*[\mathsf{a} \lor \mathsf{b}](x)
&= \max_{\alpha} \max\{ \rho[\mathsf{a}](\xi_x^{\alpha}),\, \rho[\mathsf{b}](\xi_x^{\alpha}) \}\\
&= \max\Big\{ \max_{\alpha} \rho[\mathsf{a}](\xi_x^{\alpha}),\, \max_{\alpha} \rho[\mathsf{b}](\xi_x^{\alpha}) \Big\}\\
&= \max\big\{ V^*[\mathsf{a}](x),\, V^*[\mathsf{b}](x) \big\} \\
&= \max_{\alpha} \max\big\{ V^*[\mathsf{a}](\xi_x^{\alpha}(0)),\, V^*[\mathsf{b}](\xi_x^{\alpha}(0)) \big\} \\
&= V^*[\mathsf{v}_\mathsf{a} \lor \mathsf{v}_\mathsf{b}](x). \\
\end{aligned}
\end{eqnarray*}
Third, we give the \texttt{weak-AND} property,
\begin{eqnarray*}
\begin{aligned}
V^*[\mathsf{c} \land \mathsf{p}](x)
&= \max_{\alpha} \min\{ \rho[\mathsf{c}](\xi_x^{\alpha}),\, \rho[\mathsf{p}](\xi_x^{\alpha}) \}\\
&= \min\Big\{ \rho[\mathsf{c}](\xi_x^{\beta}),\, \max_{\alpha} \rho[\mathsf{p}](\xi_x^{\alpha}) \Big\}, \quad \beta \in \mathcal{A}^\mathbb{N}\\
&= \min\big\{ \rho[\mathsf{c}](\xi_x^{\beta}),\, V^*[\mathsf{p}](x) \big\} \\
&= \max_{\alpha} \min\big\{ \rho[\mathsf{c}](\xi_x^{\alpha}),\, V^*[\mathsf{p}](\xi_x^{\alpha}(0)) \big\} \\
&= V^*[\mathsf{c} \land \mathsf{v}_\mathsf{p}](x). \\
\end{aligned}
\end{eqnarray*}
Lastly, we give the \texttt{NEXT} property,
\begin{align*}
V^*[\mathsf{X}\mathsf{p}](x)
&= \max_{\alpha} \rho[\mathsf{p}](\xi_x^{\alpha}, 1) \\
&= \max_{a_1 \in \mathcal{A}} \max_{\alpha'} \rho[\mathsf{p}](\xi_{f(x, a_1)}^{\alpha'}, 0) \\
&= \max_{a \in \mathcal{A}} V^*[\mathsf{p}](f(x, a)) \\
&= \max_{\alpha} V^*[\mathsf{p}](\xi_x^{\alpha}(1))  \\
&= \max_{\alpha} \rho[\mathsf{v}_\mathsf{p}](\xi_x^{\alpha}, 1) \qedhere
\end{align*}
\endproof

Intuitively, these properties illustrate when the algebra of Bellman Value functions is equivalent to that of logic vis a vis the logical operators that ``align'' with the optimum over actions. Namely, these are the $\lor$ and right-side $\mathsf{U}$ which are quantitatively represented by maxima, and hence, commute with the maxima over action sequences (in the appropriate settings, e.g. finite state spaces).

With these equivalences, relevant classes of predicates are immediately decomposable, given by the following corollaries.

\begin{tcolorbox}[DefFrame]
    \begin{restatable}{coroll}{corountilchain} \label{cor:corountilchain}
        Let a predicate $\mathsf{p}_N$ be defined by the chain of N-Untils over predicates $\mathsf{a}_i$, s.t.
        \begin{equation*}
            \mathsf{p}_{N} = (\mathsf{a}_N \, \mathsf{U} \, \mathsf{p}_{N-1}), \qquad \mathsf{p}_{1} = \mathsf{a}_1.
        \end{equation*}
        Then then following property holds,
        \begin{equation*}
            V^*[\mathsf{p}](x) = V^*[\mathsf{a}_N \land \mathsf{v}_{\mathsf{p}_{N-1}}](x).
        \end{equation*}
    \end{restatable}
\end{tcolorbox}

This result, which is proved by simple iterative application of the first property of Thm.~\ref{lem:equiv}, shows that the Bellman value for a chain of Untils is equivalent to a chain of $\mathcal{RA}$ Bellman Values. Notably, another special case of this property is the eventually-always predicate $\mathsf{FGr}$, which corresponds to the reach-stay Value.

\begin{tcolorbox}[DefFrame]
    \begin{restatable}{coroll}{corfgreachstay} \label{cor:fgreachstay}
        For the eventually-always predicate $\mathsf{FGr}$, and corresponding reach-stabilize Value,
        \begin{equation*}
            V^*[\mathsf{FGr}](x) = V^*[\mathsf{F}\mathsf{v}_{\mathsf{Gr}}](x),
        \end{equation*}
        where $V_{\mathsf{Gr}}$ is the $\mathcal{A}$-Value for the region defined by $\lnot\mathsf{r}$.
    \end{restatable}
\end{tcolorbox}

\newpage
\section{$N$-$\mathcal{RA}_\ell$ Results} \label{apx:raloop-proof}

In this section, we give several properties surrounding the $\mathsf{GF}$ operation, including the $\mathsf{RA}_\ell$ Bellman equation that may be used in this context and the extension to $\mathsf{G}$ of multi-eventually and Until predicates.

Note, by definition we have the following property.
$$\rho[\mathsf{GFr}](\xi_x, t) = \inf_{t' \ge t} \sup_{t'' \ge t'} \rho[\mathsf{r}](\xi_x, t'') = \limsup_{s \to \infty} \rho[\mathsf{r}](\xi_x, s).$$
This is, ofcourse, a special case of the $\mathsf{G(q \, U \, r)}$ Bellman equation, which itself satisfies
\begin{align*}
    \rho[\mathsf{G(q \, U \, r)}]&(\xi_x, t) = \inf_{t' \ge t} \sup_{t'' \ge t'} \min\{ \rho[\mathsf{r}](\xi_x, t''), \min_{\kappa \le t''} \rho[\mathsf{q}](\xi_x, \kappa) \} \\
    &= \limsup_{s \to \infty} \min\{ \rho[\mathsf{r}](\xi_x, s), \min_{\kappa \le s} \rho[\mathsf{q}](\xi_x, \kappa) \} \\
    &= \min\{ \limsup_{s \to \infty} \rho[\mathsf{r}](\xi_x, s), \min_{\kappa \ge t} \rho[\mathsf{q}](\xi_x, \kappa) \} \\
    &= \rho[\mathsf{GFr} \land \mathsf{Gq}](\xi_x, t).
\end{align*}
In either case, the infinite-horizon nature of the $\mathsf{G}$ composition immediately yields several qualities regarding the temporal-independence of the $\mathsf{G}$ compositions.
\begin{lemma}
    The following properties hold:
    \begin{itemize}
        \item $\rho[\mathsf{G(q \, U \, r)}](\xi_x, t) = \rho[\mathsf{G(q \, U \, r)}](\xi_x, s), \quad \forall s \ge t$.
        \item $\mathsf{G(q \, U \, r)} = \mathsf{X}^n\mathsf{G(q \, U \, r)}, \quad \forall n \in \mathbb{N}$
        \item $V^*[\mathsf{G(q \, U \, r)}](x) = V^*[\mathsf{G(q \, U \, r)}](\xi^\alpha_x(s)), \quad \forall s \ge 0$.
    \end{itemize}
\end{lemma}

By logical rearrangement and application of the algebraic results, we may immediately have Lem.~\ref{thm:raloop} restated here for clarity.
\begin{tcolorbox}[DefFrame]
    \thmraloop*
\end{tcolorbox}

Notably, we may generalize this result to handle a composition of $\mathsf{G}$ with multiple eventually and Until predicates, by considering a loop of Bellman Values of the previous form. This corresponds to Thm.~\ref{thm:n-ra-loop} from the main text, restated as follows.
\begin{tcolorbox}[DefFrame]
    \thmnraloop*
\end{tcolorbox}

Although, these results appear like the previous decompositions, it is important to note that they are fundamentally different due to the implicit definition of the Value. Moreover, they do not guarantee the uniqueness or existence of the solution. To certify these properties, we may consider the $\mathsf{G}$ composition as the limit of the finite iterations. This is given in Sec.~\ref{apx:fxd-proof}.

With the Value iteration results, we may know conditions under which this Value exists (e.g. finite state spaces), and proceed to solve this Value. While the Value iteration is a nice theoretical procedure, it may not be practical for large state spaces and certain specifications. To address these challenges, we propose the $\mathcal{RA}_\ell$ Bellman Equation in the main text, given here for clarity,
\begin{tcolorbox}[DefFrame]
    \lemraloopbe*
\end{tcolorbox}
\proof
    We first prove the existence of the fixed point by showing that the operator is contractive and then show that in the limit of discounting, the fixed point achieves the desired solution. Note, in this context, $V \in \mathbb{R}^{|\mathcal{J}|}$ is a vector of Values.

    \,

    \noindent{{{1. Contraction}}}:

    \noindent Consider two vectors $V, W \in \mathbb{R}^{|\mathcal{J}|}$, and let $\| \cdot \|_\infty$ be the infinity norm.  Here, we write $r=\tilde{r}_j$ and $q=\tilde{q}_j$ for brevity. Note for each component $j$ we have,
    \begin{align*}
        \|\mathcal{B}^\gamma_{\mathcal{RA}_\ell}[V_j] - \mathcal{B}^\gamma_{\mathcal{RA}_\ell}[W_j] \|
        &= \gamma \| \min\{\max\{ \min\{r_j, V^+_{j+1}\}, V^+_{j} \}, q\} \, - \min\{\max\{ \min\{r, W^+_{j+1}\}, W^+_{j} \}, q\}\| \\
        &\le \gamma \left\| \max\{ \min\{r, V^+_{j+1}\}, V^+_{j} \} - \max\{ \min\{r, W^+_{j+1}\}, W^+_{j} \} \right\| \\
        &\le \gamma \max\{ \left\| \min\{r, V^+_{j+1}\} - \min\{r, W^+_{j+1}\} \right\|, \left\| V^+_j - W^+_j \right\| \} \\
        &\le \gamma \max\{ \left\| V^+_{j+1} - W^+_{j+1} \right\|, \left\| V^+_j - W^+_j \right\| \} \\
        &\le \gamma L \max\{ \left\| V_{j+1} - W_{j+1} \right\|, \left\| V_j - W_j \right\| \},
    \end{align*}
    where the last line follows from the lipschitz continuity of $V(x)$, $W(x)$ and $f(x,a)$, given the definition $V^+_j(x) := \max_{a \in \mathcal{A}} V_j(f(x,a))$.
    Taking the maximum over all components $j$, we have then
    \begin{align*}
        \|\mathcal{B}^\gamma_{\mathcal{RA}_\ell}[V] - \mathcal{B}^\gamma_{\mathcal{RA}_\ell}[W] \|_\infty &\le \gamma L \max_j\{ \left\| V_j - W_j \right\| \} \\
        &= \gamma L \left\| V - W \right\|_\infty,
    \end{align*}
    demonstrating that the operator $\mathcal{B}^\gamma_{\mathcal{RA}_\ell}$ is a contraction mapping.

    \,

    \noindent{2. Convergence in the limit of $\gamma \to 1$}:
    
    \noindent Let $V^\gamma$ be the vector-valued fixed point defined by $V^\gamma = \mathcal{B}^\gamma_{\mathcal{RA}_\ell}[V^\gamma]$, s.t. for each component $j$ we have
    \begin{align*}
        V^\gamma_j(x) &= (1-\gamma)\min\{\tilde{r}_j,\tilde{q}_j\} \, + \\
        & \qquad \qquad \gamma \min\{\max\{ \min\{\tilde{r}_j, V^{\gamma +}_{j+1}\}, V^{\gamma +}_j \}, \tilde{q}_j\}.
    \end{align*}
    Note, each component is just a special case of the proof of Proposition 3 in \cite{hsu2021safety}, hence we may know,
    \begin{align*}
        \lim_{\gamma \to 1} V^\gamma_j(x) &= \max_\alpha \max_t \min \{ \min\{\tilde{r}_j(x), V_{j+1}^{*,+}(x)\}, \max_{\kappa \in [0, t]} \tilde{q}_j(x)\} \\
        &= V_j^*[\tilde{\mathsf{q}}_j \, \mathsf{U} \, (\tilde{\mathsf{r}}_j \land \mathsf{X}\mathsf{v}_{j+1})](x) \\
        &= V^*\left[\mathsf{G} \left( \bigwedge_{j \in \mathcal{J}} (\mathsf{q}_j\; \mathsf{U}\; \mathsf{r}_j) \right)\right](x),
    \end{align*}
    where the last line follows from Thm.~\ref{thm:n-ra-loop}.
\endproof

\newpage
\section{Recurrence Specification Discounting}
\label{apx:nral-disct}

Here, we show an interesting property of Values for recurrence specifications when
approximated by the \ref{def:raloopbe}. Specifically, we show that the presence
of discounting, i.e. $\gamma < 1$, gives higher Value to trajectories that
achieve high rewards more frequently. It follows that optimal behaviors will
prefer to revisit the reward regions with shorter period (when infinite goal
attainment is impossible), yielding state trajectory loops of shorter length. This effect on the Value can have distinct consequences on greater task solutions, for example, in Fig.~\ref{fig:alwaysdemofig}, we can see that \texttt{VDPPO} yields policies which travel to the farther workspace from the gear (because the loop in the distant workspace is shorter). We will show this for the $N=2$ case with an informal derivation, but the logic is the same for higher $N$. 

Consider the coupled recurrence Values
\begin{align}
V_1(s)
&=
(1-\gamma) r_1(s)
+
\gamma
\max\left\{
    \min\{r_1(s), V_2(s^+)\},
    V_1(s^+)
\right\},
\label{eq:coupled-v1-short-loop}
\\
V_2(s)
&=
(1-\gamma) r_2(s)
+
\gamma
\max\left\{
    \min\{r_2(s), V_1(s^+)\},
    V_2(s^+)
\right\},
\label{eq:coupled-v2-short-loop}
\end{align}

where \(
V_i(s^+) := \max_{a\in\mathcal A} V_i(f(s,a))
\).
We first expand these equations along a fixed controlled trajectory and write \(r_i^t := r_i(s_t)\). The optimal-control expression is then obtained by taking the supremum over $\alpha$.

Starting from \(V_1(s_0)\), suppose the trajectory remains in the \(V_1\)-mode
until time \(\tau_0\), at which point it switches to the \(V_2\)-mode. Repeated
substitution of the \(V_1\)-continuation branch gives

\begin{align}
V_1(s_0)
&=
(1-\gamma)\sum_{t=0}^{\tau_0}\gamma^t r_1^t
+
\gamma^{\tau_0+1}
\min\{r_1^{\tau_0},V_2(s_{\tau_0+1})\}
\\
&=
(1-\gamma)\sum_{t=0}^{\tau_0}\gamma^t r_1^t
+
\min\left\{
    \gamma^{\tau_0+1}r_1^{\tau_0},
    \gamma^{\tau_0+1}V_2(s_{\tau_0+1})
\right\}.
\label{eq:first-switch-short-loop}
\end{align}

Now suppose the trajectory remains in the \(V_2\)-mode until time
\(\tau_1>\tau_0\), at which point it switches back to the \(V_1\)-mode. Expanding
the second term in \eqref{eq:first-switch-short-loop} using
\eqref{eq:coupled-v2-short-loop} gives
\begin{align}
\gamma^{\tau_0+1}V_2(s_{\tau_0+1})
&=
(1-\gamma)\sum_{t=\tau_0+1}^{\tau_1}\gamma^t r_2^t
+
\gamma^{\tau_1+1}
\min\{r_2^{\tau_1},V_1(s_{\tau_1+1})\}
\\
&=
(1-\gamma)\sum_{t=\tau_0+1}^{\tau_1}\gamma^t r_2^t
+
\min\left\{
    \gamma^{\tau_1+1}r_2^{\tau_1},
    \gamma^{\tau_1+1}V_1(s_{\tau_1+1})
\right\}.
\label{eq:second-switch-short-loop}
\end{align}
Substituting \eqref{eq:second-switch-short-loop} into
\eqref{eq:first-switch-short-loop} yields
\begin{align}
V_1(s_0)
=
&(1-\gamma)\sum_{t=0}^{\tau_0}\gamma^t r_1^t
\nonumber\\
&+
\min\Bigg\{
    \gamma^{\tau_0+1}r_1^{\tau_0},
    \,
    (1-\gamma)\sum_{t=\tau_0+1}^{\tau_1}\gamma^t r_2^t
    +
    \min\left\{
        \gamma^{\tau_1+1}r_2^{\tau_1},
        \gamma^{\tau_1+1}V_1(s_{\tau_1+1})
    \right\}
\Bigg\}.
\label{eq:two-switch-short-loop}
\end{align}
Continuing this expansion gives an alternating recurrence structure. Let
\[
0 \leq \tau_0 < \tau_1 < \tau_2 < \cdots
\]
denote the switching times, and define
\[
m_k =
\begin{cases}
1, & k \text{ even},\\
2, & k \text{ odd}.
\end{cases}
\]
Also define the discounted reward accumulated in mode \(i\) over the interval
\([p,q]\) by
\[
R_i[p,q]
:=
(1-\gamma)\sum_{t=p}^{q}\gamma^t r_i^t.
\]
Then the telescoped expression for \(V_1\) is
\begin{align}
V_1(s_0)
=
\sup_{\alpha}
\sup_{0\leq \tau_0<\tau_1<\cdots}
\Bigg[
&
R_{m_0}[0,\tau_0]
+
\min\Bigg\{
    \gamma^{\tau_0+1}r_{m_0}^{\tau_0},
    \,
    R_{m_1}[\tau_0+1,\tau_1]
\nonumber\\
&\qquad\qquad\qquad
+
\min\Bigg\{
    \gamma^{\tau_1+1}r_{m_1}^{\tau_1},
    \,
    R_{m_2}[\tau_1+1,\tau_2]
\nonumber\\
&\qquad\qquad\qquad\qquad
+
\min\Bigg\{
    \gamma^{\tau_2+1}r_{m_2}^{\tau_2},
    \,
    R_{m_3}[\tau_2+1,\tau_3]
    + \cdots
\Bigg\}
\Bigg\}
\Bigg\}
\Bigg].
\label{eq:coupled-telescope-short-loop}
\end{align}
Note, the expression \eqref{eq:coupled-telescope-short-loop} contains two types of
discounted terms:
\(
R_{m_k}[\tau_{k-1}+1,\tau_k]
=
(1-\gamma)\sum_{t=\tau_{k-1}+1}^{\tau_k}\gamma^t r_{m_k}^t,
\)
and the switching bottleneck terms
\(
\gamma^{\tau_k+1}r_{m_k}^{\tau_k}.
\)
Both decrease when the same reward is attained later in time.

We now make the preference for shorter loops explicit. Consider two trajectories
\(A\) and \(B\) that realize the same alternating high-reward sequence, but with
different loop periods. Suppose trajectory \(A\) attains the high rewards every
\(T_A\) steps, while trajectory \(B\) attains the same high rewards every
\(T_B\) steps, with
\[
T_A < T_B.
\]
For simplicity, assume both trajectories attain the same switching reward
level \(R>0\) at each switching event, so that
\[
r_{m_k}(s_{\tau_k^A}^A)
=
r_{m_k}(s_{\tau_k^B}^B)
=
R
\qquad
\text{for all } k.
\]
Assume also that the switching times are periodic:
\[
\tau_k^A = kT_A,
\qquad
\tau_k^B = kT_B.
\]
Then the \(k\)-th switching bottleneck contribution for trajectory \(A\) is
\[
\gamma^{\tau_k^A+1}R
=
\gamma^{kT_A+1}R,
\]
whereas for trajectory \(B\) it is
\[
\gamma^{\tau_k^B+1}R
=
\gamma^{kT_B+1}R.
\]
Since \(\gamma\in(0,1)\) and \(T_A<T_B\), we have
\[
kT_A+1 < kT_B+1
\qquad
\text{for all } k\geq 1,
\]
and therefore
\[
\gamma^{kT_A+1}R
>
\gamma^{kT_B+1}R.
\]
Thus every corresponding recurrent switching reward after the first is larger
for the shorter-period trajectory. Equivalently,
\begin{align}
\sum_{k=0}^{\infty}\gamma^{kT_A+1}R
&=
\gamma R\sum_{k=0}^{\infty}\gamma^{kT_A}
=
\frac{\gamma R}{1-\gamma^{T_A}},
\\
\sum_{k=0}^{\infty}\gamma^{kT_B+1}R
&=
\gamma R\sum_{k=0}^{\infty}\gamma^{kT_B}
=
\frac{\gamma R}{1-\gamma^{T_B}}.
\end{align}
Because \(T_A<T_B\), we have
\[
\gamma^{T_A}>\gamma^{T_B},
\]
and hence
\[
1-\gamma^{T_A}<1-\gamma^{T_B}.
\]
Therefore,
\[
\frac{\gamma R}{1-\gamma^{T_A}}
>
\frac{\gamma R}{1-\gamma^{T_B}}.
\]

Hence, for the same recurrent high-reward level \(R\), the
discounted contribution of the switching bottleneck terms is larger for the
trajectory with shorter loop period.
The same conclusion holds for the running-reward blocks. If trajectory \(A\)
and trajectory \(B\) receive comparable rewards within each mode, but trajectory
\(A\) reaches the next high-reward region earlier, then the corresponding
discounted running rewards in \(A\) are weighted by larger powers of \(\gamma\).
That is, earlier terms receive weights closer to one. Thus, under the standard
comparison assumption that trajectory \(A\)'s intermediate rewards are not worse
than trajectory \(B\)'s, the running-reward blocks in the telescoped expression
are also no smaller for \(A\).
Finally, the nested expression in \eqref{eq:coupled-telescope-short-loop} is
monotone nondecreasing in each of its scalar arguments. This follows because the
maps
\[
x\mapsto c+x,
\qquad
x\mapsto \min\{c,x\},
\qquad
x\mapsto \max\{c,x\}
\]
are all monotone nondecreasing. Therefore, if every discounted running-reward
block and every discounted switching bottleneck term for trajectory \(A\) is at
least as large as the corresponding term for trajectory \(B\), then the full
telescoped Value of trajectory \(A\) is at least as large as that of trajectory
\(B\). If at least one active term is strictly larger, then trajectory \(A\)
receives strictly larger Value.
Consequently, among trajectories that attain the same recurrent high rewards
and have comparable intermediate rewards, the trajectory with shorter recurrence
period obtains larger Value. 

\newpage
\section{$\mathsf{G}(\dots)$ Fixed Point Iteration} \label{apx:fxd-proof}

In this section, we present an alternate perspective on the Bellman Value corresponding to the $\mathsf{G}(\dots)$ compositions based on finite iterations of recursion. Indeed, one may use this approach to solve the Value, however, for large state spaces or complicated specifications, this may be expensive. We principally employ this approach to guarantee the uniqueness and existence of the corresponding Bellman Values (which in general may be ill defined) in order to accompany the $\mathcal{RA}_\ell$-BE. 

\subsection{Single-Predicate Recurrence}
For clarity, we begin by considering the case involving the recurrence of a single predicate (target to reach), given by $\mathsf{p} := \mathsf{GFr}$ and Value
$$
V[\mathsf{GFr}](x) = \max_\alpha \max_{t \ge 0} \min \Big\{r(\xi_x^\alpha(t)), V[\mathsf{GFr}](\xi_x^\alpha(t+1)) \Big\}
$$
per Thm~\ref{thm:raloop}.

We now consider the following other value function:
\begin{align*}
    V_{k+1}(x) &\coloneqq V^* [ \mathsf{F} (\mathsf{r \land X} \mathsf{v}_{k}) ](x) \\
    &= \max_{\alpha} \max_{t \geq 0} \min\Big( r(\xi_x^\alpha(t)),\; V_{k}(\xi_x^\alpha(t+1)) \Big),
\end{align*}
where $V_{0}(x) \coloneqq \infty$ for all $x$ i.e. $\mathsf{v}_0 := \top$.

\begin{tcolorbox}[DefFrame]
\begin{lemma}
    The sequence $V^k$ converges to $V[\mathsf{GFr}]$ pointwise, i.e., for all $x$,
    \begin{equation*}
        \lim_{k \to \infty} V_{k}(x) = V[\mathsf{GFr}](x).
    \end{equation*}
\end{lemma}
\end{tcolorbox}
\proof
    First, for an arbitrary threshold $\lambda$, construct the superlevel sets $R$, $W^*$ and $W_k$ as
    \begin{align*}
        R &\coloneqq \{ x : r(x) \geq \lambda \},\\
        W^* &\coloneqq \{ x : V^*[\mathsf{GFr}](x) \geq \lambda \}, \\
        W_k &\coloneqq \{ x : V_{k}(x) \geq \lambda \}.
    \end{align*}
    Note that $W_k$ is exactly the set of states from which it is possible to reach $R$ at least $k$ times.

    Since $V_{0}(x) = \infty$ for all $x$, we have $W_{0} = \mathcal{X}$.
    Let $\mathcal{T}$ denote the operator that maps $V_{k}$ to $V_{k+1}$, i.e., $V_{k+1} = \mathcal{T} V_{k}$. By Lem.~\ref{lem:gf_monotone}, $\mathcal{T}$ is monotone, i.e., $U(x) \leq V(x) \implies \mathcal{T} U(x) \leq \mathcal{T} V(x)$ for all $x$. Moreover, since $V_{1} \leq V_{0}$, we have $V_{k+1} \leq V_{k}$ for all $k$ by induction, and thus $W_{k+1} \subseteq W_{k}$ for all $k$.

    Since $W_k$ is a decreasing sequence of sets, the limit $W_\infty = \bigcap_{k=0}^\infty W_k$ exists, and also that $\lim_{k \to \infty} V_{k}(x) = V^{\infty}(x)$ exists for all $x$.


    \paragraph{1. ($W^* \subseteq W_\infty$)}
    Let $x \in W^*$.
    Then, by definition of $V^*[\mathsf{GFr}]$, there exists an action sequence $\alpha$ such that the system visits $R$ infinitely often.
    In particular, for any $k \in \mathbb{N}$, the system can reach $R$ at least $k$ times under $\alpha$.
    Hence, $x \in W_k$ for all $k$, and thus $x \in W_\infty$.

    \paragraph{2. ($W^* \supseteq W_\infty$)}
    We apply either Lem.~\ref{lem:Wstar_in_Winfty:finite_states}, \ref{lem:Wstar_in_Winfty:finite_actions}, or \ref{lem:Wstar_in_Winfty:compact_actions} depending on the assumptions on the state and action spaces to conclude that $W_\infty \subseteq W^*$.

    Since we have shown both inclusions, we conclude that $W^* = W_\infty$.
    Since this holds for any threshold $\lambda$, we have $\lim_{k \to \infty} V_{k}(x) = V^*[\mathsf{GFr}](x)$ for all $x$, i.e., $V_{k}$ converges pointwise to $V^*[\mathsf{GFr}]$.
\endproof

\begin{tcolorbox}[DefFrame]
\begin{lemma} \label{lem:gf_monotone}
    The operator $\mathcal{T}$ defined as
    \begin{equation*}
        \mathcal{T} V(x) = \max_{\alpha} \max_{t \geq 0} \min\Big( r(\xi_x^\alpha(t)),\; V(\xi_x^\alpha(t+1)) \Big)
    \end{equation*}
    is monotone, i.e., for any two functions $U$ and $V$ such that $U(x) \leq V(x)$ for all $x$, we have $\mathcal{T} U(x) \leq \mathcal{T} V(x)$ for all $x$.
\end{lemma}
\end{tcolorbox}
\proof
    Let $U$ and $V$ be two functions such that $U(x) \leq V(x)$ for all $x$.
    Then, for any action sequence $\alpha$ and any time $t$,
    \begin{align*}
        \min&\Big( r(\xi_x^\alpha(t)),\; U(\xi_x^\alpha(t+1)) \Big) \leq \\ &\qquad \qquad \qquad \min\Big( r(\xi_x^\alpha(t)),\; V(\xi_x^\alpha(t+1)) \Big).
    \end{align*}
    Taking $\max$ over $t$ and $\alpha$ on both sides yields
    \begin{equation*}
        \mathcal{T} U(x) \leq \mathcal{T} V(x).
    \end{equation*}
\endproof

\begin{lemma} \label{lem:Wstar_in_Winfty:finite_states}
    Suppose the set of states $\mathcal{X}$ is finite.
    Then, $W_\infty \subseteq W^*$.
\end{lemma}
\proof
    First, since $\mathcal{X}$ is finite, $W_k \subseteq \mathcal{X}$ is finite for all $k$.
    Moreover, since $W_{k+1} \subseteq W_{k}$ for all $k$, the sequence $W_k$ must stabilize at some finite $K$, i.e., $W_K = W_\infty$ for some $K$.
    Hence, $W_\infty$ is a fixed point of the operator that maps $W_k$ to $W_{k+1}$.

    Now, let $x \in W_\infty$. Since $W_\infty$ is a fixed point, there exists some action sequence $\alpha$ and time $t$ such that $\xi_x^\alpha(t) \in R$, and $\xi_x^\alpha(t) \in W_\infty$.
    We can repeat this argument to construct an infinite action sequence $\alpha$ under which the system visits $R$ infinitely often.
    Thus, $x \in W^*$, and $W_\infty \subseteq W^*$.
\endproof

\begin{lemma} \label{lem:Wstar_in_Winfty:finite_actions}
    Suppose the set of actions $\mathcal{A}$ is finite.
    Then, $W_\infty \subseteq W^*$.
\end{lemma}
\proof
    Let $x \in W_\infty$.
    Then, for any $k \in \mathbb{N}$, there exists an action sequence $\alpha^k$ such that the system can reach $R$ at least $k$ times under $\alpha^k$.
    We now construct a ``success tree'' where, from every node, we create a branch for each action in $\mathcal{A}$, and we remove all nodes that are not in $W_\infty$.
    Since $\mathcal{A}$ is finite, this tree has a finite branching factor.
    Moreover, since $x \in W_\infty$, for any depth $k$, there exists a path from the root to a node at depth $k$.
    By K\"onig's lemma \citep{diestel2025graph}, there exists an infinite path from the root. Since all nodes in the tree are in $W_\infty$, this infinite path corresponds to an action sequence under which the system visits $R$ infinitely often. Thus, $x \in W^*$, and $W_\infty \subseteq W^*$.
\endproof

\begin{lemma} \label{lem:Wstar_in_Winfty:compact_actions}
    Suppose the set of actions $\mathcal{A}$ is a compact space, and the dynamics $f$ is continuous in $a$.
    Then, $W_\infty \subseteq W^*$.
\end{lemma}
\proof
    Let $x \in W_\infty$.
    Then, for any $k \in \mathbb{N}$, there exists an action sequence $\alpha^k$ such that the system can reach $R$ at least $k$ times under $\alpha^k$.
    We now construct a sequence of non-empty compact sets $C_n$ as follows.
    Let $C_0 = \mathcal{A}$.
    For each $n \geq 1$, let
    \begin{align*}
        &C_n = \{ a \in C_{n-1} : \exists a_{1:\infty} \text{ s.t. } \\ & \qquad \text{the system reaches } R \text{ at least } n \text{ times under } (a, a_{1:\infty}) \}.
    \end{align*}
    Note that $C_n$ is non-empty since $x \in W_\infty$.
    Moreover, $C_n$ is closed since the dynamics $f$ is continuous in $a$, and thus $C_n$ is compact as a closed subset of the compact set $C_{n-1}$.
    Since $C_{n+1} \subseteq C_n$ for all $n$, by Cantor's intersection theorem \citep{rudin1976principles}, the intersection $\bigcap_{n=0}^\infty C_n$ is non-empty.
    Let $a_0$ be an element in this intersection.
    By construction of $C_n$, there exists an action sequence $a_{1:\infty}$ such that the system reaches $R$ at least $n$ times under $(a_0, a_{1:\infty})$ for all $n$.
    Hence, the system visits $R$ infinitely often under the action sequence $(a_0, a_{1:\infty})$, and thus $x \in W^*$.
    Therefore, $W_\infty \subseteq W^*$.
\endproof

\subsection{Multi-Predicate Recurrence}

Here we give a generalization of the previous finite recurrence approach to compositions of $\mathsf{G}$ with multi-Until predicates. We give the proofs for the case with $N=2$ but the generalization to $N > 2$ follows similarly.

\newcommand{\Vstar}{%
  V^*[\mathsf{G}(\land_j \mathsf{q}_j
  \, \mathsf{U} \, \mathsf{r}_j)]}

Let the globally-(until and until) value function
be defined as
\begin{align*}
    &\Vstar(x_0) \\
    &\coloneqq \max_{\alpha}
      \rho\bigl[
        \mathsf{G}(
          \mathsf{q}_1 \, \mathsf{U} \, \mathsf{r}_1
          \land
          \mathsf{q}_2 \, \mathsf{U} \, \mathsf{r}_2
        )
      \bigr](x_0, 0) \\
    &= \max_{\alpha} \min_{t \geq 0}
      \min\Bigl\{
        \max_{s \geq t} \min\bigl\{
          r_1\bigl(\xi_{x_0}^\alpha(s)\bigr),
          \\
    &\qquad\quad
          \min_{0 \leq \ell < s}
            q_1\bigl(\xi_{x_0}^\alpha(\ell)\bigr)
        \bigr\}, \\
    &\qquad
        \max_{u \geq t} \min\bigl\{
          r_2\bigl(\xi_{x_0}^\alpha(u)\bigr),
          \\
    &\qquad\quad
          \min_{0 \leq \ell < u}
            q_2\bigl(\xi_{x_0}^\alpha(\ell)\bigr)
        \bigr\}
      \Bigr\}.
\end{align*}

Let $w_1 \coloneqq q_1 \lor r_1$ and
$w_2 \coloneqq q_2 \lor r_2$, and define the
``until'' objective function $U_i$ as
\begin{align*}
    U_i\bigl(\xi_{x}^{\alpha_{t:\infty}}\bigr)
    &\coloneqq \sup_{s \geq t} \min\bigl\{
      r_i\bigl(\xi_{x}^{\alpha_{t:\infty}}(s)\bigr), \\
    &\qquad
      \min_{t \leq \ell < s}
        q_i\bigl(\xi_{x}^\alpha(\ell)\bigr)
    \bigr\}.
\end{align*}

We now consider the following coupled system of
value functions:
\begin{align*}
    &V_{1,k+1}(x_0) \\
    &\coloneqq \max_{\alpha}
      \rho\bigl[
        (\mathsf{q}_1 \land \mathsf{w}_2)
        \, \mathsf{U} \,
        (\mathsf{r}_1 \land \mathsf{w}_2
          \land \mathsf{X} \, V_{2,k})
      \bigr](x_0) \\
    &= \max_{\alpha} \max_{t \geq 0}
      \min\Bigl\{
        \min\bigl(
          r_1\bigl(\xi_{x_0}^\alpha(t)\bigr), \\
    &\qquad
          w_2\bigl(\xi_{x_0}^\alpha(t)\bigr),
          V_{2,k}\bigl(
            \xi_{x_0}^\alpha(t\!+\!1)\bigr)
        \bigr), \\
    &\qquad
        \min_{0 \leq \ell < t} \min\Bigl(
          q_1\bigl(\xi_{x_0}^\alpha(\ell)\bigr),
          \\
    &\qquad\quad
          w_2\bigl(\xi_{x_0}^\alpha(\ell)\bigr)
        \Bigr)
      \Bigr\}, \\[6pt]
    &V_{2,k+1}(x_0) \\
    &\coloneqq \max_{\alpha}
      \rho\bigl[
        (\mathsf{q}_2 \land \mathsf{w}_1)
        \, \mathsf{U} \,
        (\mathsf{r}_2 \land \mathsf{w}_1
          \land \mathsf{X} \, V_{1,k})
      \bigr](x_0) \\
    &= \max_{\alpha} \max_{t \geq 0}
      \min\Bigl\{
        \min\bigl(
          r_2\bigl(\xi_{x_0}^\alpha(t)\bigr), \\
    &\qquad
          w_1\bigl(\xi_{x_0}^\alpha(t)\bigr),
          V_{1,k}\bigl(
            \xi_{x_0}^\alpha(t\!+\!1)\bigr)
        \bigr), \\
    &\qquad
        \min_{0 \leq \ell < t} \min\Bigl(
          q_2\bigl(\xi_{x_0}^\alpha(\ell)\bigr),
          \\
    &\qquad\quad
          w_1\bigl(\xi_{x_0}^\alpha(\ell)\bigr)
        \Bigr)
      \Bigr\},
\end{align*}
where $V_{1,0}(x) \coloneqq \infty$ and
$V_{2,0}(x) \coloneqq \infty$ for all $x$.

\begin{tcolorbox}[DefFrame]
\begin{lemma}
    For any $k > 0$, let $\xi_{x_0}^\alpha$
    be the trajectory generated by the policy
    achieving the supremum in $V_{i,k}(x_0)$.
    Then,
    \begin{equation}
        V_{i,k}(x_0)
          \leq U_i\bigl(
            \xi_{x_0}^{{\alpha}_{0:\infty}}\bigr)
    \end{equation}
\end{lemma}
\end{tcolorbox}
\begin{proof}
    \begin{align*}
        &V_{i,k}(x_0) \\
        &= \max_{t \geq 0} \min\Bigl\{
          \min\bigl(
            r_i\bigl(\xi_{x_0}^\alpha(t)\bigr),
            \\
        &\qquad
            w_{\neg i}\bigl(
              \xi_{x_0}^\alpha(t)\bigr),
            V_{\neg i,k-1}\bigl(
              \xi_{x_0}^\alpha(t\!+\!1)\bigr)
          \bigr), \\
        &\qquad
          \min_{0 \leq \ell < t} \min\Bigl(
            q_i\bigl(\xi_{x_0}^\alpha(\ell)\bigr),
            \\
        &\qquad\quad
            w_{\neg i}\bigl(
              \xi_{x_0}^\alpha(\ell)\bigr)
          \Bigr)
        \Bigr\} \\
        &\leq \max_{t \geq 0} \min\Bigl\{
          r_i\bigl(\xi_{x_0}^\alpha(t)\bigr),
          \\
        &\qquad
          \min_{0 \leq \ell < t}
            q_i\bigl(\xi_{x_0}^\alpha(\ell)\bigr)
        \Bigr\} \\
        &= U_i\bigl(
          \xi_{x_0}^{{\alpha}_{0:\infty}}\bigr).
    \end{align*}
\end{proof}

\begin{tcolorbox}[DefFrame]
\begin{lemma} \label{lem:GUU:fixed_point_convergence}
    Both sequences $V_{1,k}$ and $V_{2,k}$
    converge to $\Vstar$ pointwise, i.e.,
    for all $x$,
    \begin{align*}
        \lim_{k \to \infty} V_{1,k}(x)
        &= \lim_{k \to \infty} V_{2,k}(x) \\
        &= \Vstar(x).
    \end{align*}
\end{lemma}
\end{tcolorbox}

Before we prove
Lem.~\ref{lem:GUU:fixed_point_convergence},
we set up a few useful definitions and lemmas.

Define the operator $\mathcal{T}$ mapping
$(J_1, J_2)$ to $(J_1', J_2')$ as
\begin{align*}
    &J_1'(x_0) \\
    &\coloneqq \sup_{\alpha} \sup_{t \geq 0}
      \min\Bigl\{
        \min\bigl(
          r_1\bigl(\xi_{x_0}^\alpha(t)\bigr), \\
    &\qquad
          w_2\bigl(\xi_{x_0}^\alpha(t)\bigr),
          J_2\bigl(
            \xi_{x_0}^\alpha(t\!+\!1)\bigr)
        \bigr), \\
    &\qquad
        \min_{0 \leq \ell < t} \min\Bigl(
          q_1\bigl(\xi_{x_0}^\alpha(\ell)\bigr),
          \\
    &\qquad\quad
          w_2\bigl(\xi_{x_0}^\alpha(\ell)\bigr)
        \Bigr)
      \Bigr\}, \\[6pt]
    &J_2'(x_0) \\
    &\coloneqq \sup_{\alpha} \sup_{t \geq 0}
      \min\Bigl\{
        \min\bigl(
          r_2\bigl(\xi_{x_0}^\alpha(t)\bigr), \\
    &\qquad
          w_1\bigl(\xi_{x_0}^\alpha(t)\bigr),
          J_1\bigl(
            \xi_{x_0}^\alpha(t\!+\!1)\bigr)
        \bigr), \\
    &\qquad
        \min_{0 \leq \ell < t} \min\Bigl(
          q_2\bigl(\xi_{x_0}^\alpha(\ell)\bigr),
          \\
    &\qquad\quad
          w_1\bigl(\xi_{x_0}^\alpha(\ell)\bigr)
        \Bigr)
      \Bigr\}.
\end{align*}

\begin{tcolorbox}[DefFrame]
\begin{lemma}
    The operator $\mathcal{T}$ is monotone.
\end{lemma}
\end{tcolorbox}
\begin{proof}
    It follows immediately from the monotonicity
    of the $\sup$ and $\min$ operators.
\end{proof}

\begin{tcolorbox}[DefFrame]
\begin{lemma}
    Both sequences converge pointwise, i.e.,
    $V_{1,\infty}$ and $V_{2,\infty}$ exist.
\end{lemma}
\end{tcolorbox}
\begin{proof}
    Since $V_{1,0}(x) = \infty$ and $V_{1,1}(x)$
    is finite,
    $V_{1,1}(x) \leq V_{1,0}(x)$ for all $x$.
    By monotonicity of $\mathcal{T}$, the sequence
    $V_{1,k}$ is non-increasing.
    Moreover, $V_{1,0}(x)$ is bounded below by
    $\min( \inf_x r_1(x), \inf_x r_2(x) )$.
    Thus, by the monotone convergence theorem,
    $V_{1,\infty}(x)
      = \lim_{k \to \infty} V_{1,k}(x)$
    exists for all $x$.
    The same reasoning applies to $V_{2,k}$
    to show that
    $V_{2,\infty}(x)
      = \lim_{k \to \infty} V_{2,k}(x)$
    exists for all $x$.
\end{proof}

We now show that $V_{1,\infty}$ and $V_{2,\infty}$
both equal $\Vstar$ via double inequality.

\begin{tcolorbox}[DefFrame]
\begin{lemma}
    \begin{equation} \label{lem:GUU:v_smaller_Vinf}
        \Vstar(x) \leq V_{i,\infty}(x)
        \quad \text{for } i = 1, 2.
    \end{equation}
\end{lemma}
\end{tcolorbox}
\begin{proof}
    Let $\Vstar(x_0) = \lambda$.
    By definition of the $\sup$ in $\Vstar$,
    for any $\epsilon > 0$, there exists a policy
    $\alpha$ such that for all $t \geq 0$,
    \begin{align} \label{eq:GUU:U_lower_bound}
        U_1\bigl(
          \xi_{x_0}^{{\alpha}_{t:\infty}}\bigr)
          &\geq \lambda - \epsilon,
          \notag \\
        U_2\bigl(
          \xi_{x_0}^{{\alpha}_{t:\infty}}\bigr)
          &\geq \lambda - \epsilon.
    \end{align}
    Using the recursive relation of $U_i$,
    \begin{align*}
        &U_i\bigl(
          \xi_{x_0}^{{\alpha}_{t:\infty}}\bigr)
        \\
        &= \max\bigl\{
          r_i\bigl(\xi_{x_0}^\alpha(t)\bigr),
          \\
        &\qquad
          \min\bigl(
            q_i\bigl(\xi_{x_0}^\alpha(t)\bigr),
            U_i\bigl(
              \xi_{x_0}^{{\alpha}_{t+1:\infty}}
            \bigr)
          \bigr) \bigr\} \\
        &\leq \max\bigl\{
          r_i\bigl(\xi_{x_0}^\alpha(t)\bigr),
          q_i\bigl(\xi_{x_0}^\alpha(t)\bigr)
        \bigr\} \\
        &= w_i\bigl(\xi_{x_0}^\alpha(t)\bigr).
    \end{align*}
    Hence, \eqref{eq:GUU:U_lower_bound} implies
    that under $\alpha$,
    $w_i\bigl(\xi_{x_0}^\alpha(t)\bigr)
      \geq \lambda - \epsilon$
    for all $t \geq 0$.

    We now show via induction on $k$ that
    $V_{1,k}\bigl(\xi_{x_0}^\alpha(t)\bigr)
      \geq \lambda - \epsilon$
    and
    $V_{2,k}\bigl(\xi_{x_0}^\alpha(t)\bigr)
      \geq \lambda - \epsilon$
    for all states visited by $\alpha$.

    \paragraph{Base Case ($k=0$):}
    By definition,
    $V_{1,0}(x) = V_{2,0}(x) = \infty
      \geq \lambda - \epsilon$.

    \paragraph{Inductive Step:}
    Assume the statement holds for some $k$,
    i.e., for all visited states,
    \begin{equation}
        V_{2,k}\bigl(\xi_{x_0}^\alpha(t)\bigr)
          \geq \lambda - \epsilon.
    \end{equation}
    Consider $V_{1,k+1}(x_0)$.
    Under $\alpha$, since
    $U_1\bigl(
      \xi_{x_0}^{{\alpha}_{0:\infty}}\bigr)
      \geq \lambda - \epsilon$,
    there exists some time $t$ where
    $r_1\bigl(\xi_{x_0}^\alpha(t)\bigr)
      \geq \lambda - \epsilon$
    and for all $0 \leq \ell < t$,
    $q_1\bigl(\xi_{x_0}^\alpha(\ell)\bigr)
      \geq \lambda - \epsilon$.
    By the inductive hypothesis,
    $V_{2,k}\bigl(\xi_{x_0}^\alpha(t\!+\!1)\bigr)
      \geq \lambda - \epsilon$.
    Thus,
    \begin{align*}
        &V_{1,k+1}(x_0) \\
        &\geq \min\Bigl\{
          r_1\bigl(\xi_{x_0}^\alpha(t)\bigr),
          w_2\bigl(\xi_{x_0}^\alpha(t)\bigr),
          \\
        &\qquad
          V_{2,k}\bigl(
            \xi_{x_0}^\alpha(t\!+\!1)\bigr), \\
        &\qquad
          \min_{0 \leq \ell < t} \min\bigl\{
            q_1\bigl(
              \xi_{x_0}^\alpha(\ell)\bigr),
            \\
        &\qquad\quad
            w_2\bigl(
              \xi_{x_0}^\alpha(\ell)\bigr)
          \bigr\}
        \Bigr\} \\
        &\geq \lambda - \epsilon.
    \end{align*}
    By symmetry, the same reasoning applies to
    $V_{2,k+1}(x_0)$.
    Since $\epsilon > 0$ was arbitrary, we have
    shown \eqref{lem:GUU:v_smaller_Vinf}.
\end{proof}

\begin{tcolorbox}[DefFrame]
\begin{lemma}
    \begin{equation} \label{lem:GUU:Vinf_smaller_v}
        V_{i,\infty}(x) \leq \Vstar(x)
        \quad \text{for } i = 1, 2.
    \end{equation}
\end{lemma}
\end{tcolorbox}
\begin{proof}
    We construct a policy $\alpha$ that achieves a
    value arbitrarily close to
    $V_{1,\infty}(x_0)$.

    Let $\lambda = V_{1,\infty}(x_0)$, and fix
    $\epsilon > 0$.
    Define ``slack'' variables
    $\delta_j = \epsilon / 2^{j+1}$
    for $j = 0, 1, \dots$,
    so that $\sum_{j=0}^\infty \delta_j = \epsilon$
    and $\sum_{j=0}^N \delta_j < \epsilon$
    for all finite $N$.

    We iteratively construct $\alpha$ by stitching
    together finite segments.
    Let $m = j \bmod 2 + 1$ denote the ``mode'' at
    switch $j$.
    We show that after $j$ switches, the state
    $x_{\mathrm{sw}}$ satisfies
    \begin{equation*}
        V_{m,\infty}(x_{\mathrm{sw}})
          \geq \lambda
            - \textstyle\sum_{i=0}^{j-1} \delta_i,
    \end{equation*}
    and for all times $t$ between switches,
    \begin{align*}
        U_1\bigl(
          \xi_{x_0}^{{\alpha}_{t:\infty}}\bigr)
          &\geq \lambda - \epsilon, \\
        U_2\bigl(
          \xi_{x_0}^{{\alpha}_{t:\infty}}\bigr)
          &\geq \lambda - \epsilon.
    \end{align*}

    \paragraph{Base Case.}
    At $j=0$, we begin at $x_0$ with
    $V_{1,\infty}(x_0) = \lambda$.

    \paragraph{Inductive Step.}
    Suppose after $j$ switches we are at state
    $\xi_{x_0}^\alpha(t)$ with $m = 1$
    (the case $m=2$ follows by symmetry).
    Suppose
    $V_{1,\infty}\bigl(\xi_{x_0}^\alpha(t)\bigr)
      \geq \lambda
        - \sum_{i=0}^{2j-1} \delta_i$.
    By definition of $V_{1,\infty}$, there exists
    a finite time $t_1$ and policy segment
    $\alpha_{t:t_1-1}$ such that
    \begin{itemize}
        \item $r_1\bigl(
          \xi_{x_0}^\alpha(t_1)\bigr)
          \geq \lambda
            - \sum_{i=1}^{2j} \delta_i$
        \item $w_2\bigl(
          \xi_{x_0}^\alpha(t_1)\bigr)
          \geq \lambda
            - \sum_{i=1}^{2j} \delta_i$
        \item $V_{2,\infty}\bigl(
          \xi_{x_0}^\alpha(t_1\!+\!1)\bigr)
          \geq \lambda
            - \sum_{i=1}^{2j} \delta_i$
        \item $q_1\bigl(
          \xi_{x_0}^\alpha(s)\bigr)
          \geq \lambda
            - \sum_{i=1}^{2j} \delta_i$
          for all $t \leq s < t_1$
        \item $w_2\bigl(
          \xi_{x_0}^\alpha(s)\bigr)
          \geq \lambda
            - \sum_{i=1}^{2j} \delta_i$
          for all $t \leq s < t_1$
    \end{itemize}
    Hence, for all $\tau$ with
    $t \leq \tau < t_1$,
    \begin{align*}
        &U_1\bigl(
          \xi_{x_0}^{{\alpha}_{\tau:\infty}}\bigr)
        \\
        &\geq \min\Bigl(
          r_1\bigl(\xi_{x_0}^\alpha(t_1)\bigr),
          \\
        &\qquad
          \min_{\tau \leq s < t_1}
            q_1\bigl(\xi_{x_0}^\alpha(s)\bigr)
        \Bigr) \\
        &\geq \lambda - \epsilon.
    \end{align*}
    For $U_2$, let $\tau \in [t, t_1 - 1]$.
    We consider two cases.
    \begin{enumerate}
        \item \emph{There exists $t'$ with
          $\tau \leq t' < t_1$ and
          $r_2\bigl(\xi_{x_0}^\alpha(t')\bigr)
            \geq \lambda - \epsilon$.}
        Let $t'$ be the smallest such time.
        Since
        $w_2\bigl(\xi_{x_0}^\alpha(s)\bigr)
          \geq \lambda - \epsilon$
        and $t'$ is minimal, we have
        $q_2\bigl(\xi_{x_0}^\alpha(s)\bigr)
          \geq \lambda - \epsilon$
        for all $\tau \leq s < t'$.
        Hence,
        \begin{align*}
            &U_2\bigl(
              \xi_{x_0}^{{\alpha}_{\tau:\infty}}
            \bigr) \\
            &\geq \min\Bigl(
              r_2\bigl(
                \xi_{x_0}^\alpha(t')\bigr),
              \\
            &\qquad
              \min_{\tau \leq s < t'}
                q_2\bigl(
                  \xi_{x_0}^\alpha(s)\bigr)
            \Bigr) \\
            &\geq \lambda - \epsilon.
        \end{align*}

        \item \emph{No such $t'$ exists.}
        Since
        $U_2\bigl(
          \xi_{x_0}^{{\alpha}_{t+1:\infty}}\bigr)
          \geq V_{2,\infty}\bigl(
            \xi_{x_0}^\alpha(t\!+\!1)\bigr)
          \geq \lambda - \epsilon$,
        there exists $t'' \geq t_1$ with
        $r_2\bigl(\xi_{x_0}^\alpha(t'')\bigr)
          \geq \lambda - \epsilon$
        and
        $q_2\bigl(\xi_{x_0}^\alpha(s)\bigr)
          \geq \lambda - \epsilon$
        for all $t+1 \leq s < t''$.
        Since no $t'$ exists,
        $q_2\bigl(\xi_{x_0}^\alpha(s)\bigr)
          \geq \lambda - \epsilon$
        for all $\tau \leq s \leq t_1$.
        Thus,
        \begin{align*}
            &U_2\bigl(
              \xi_{x_0}^{{\alpha}_{\tau:\infty}}
            \bigr) \\
            &= \sup_{s \geq \tau} \min\bigl\{
              r_2\bigl(
                \xi_{x_0}^\alpha(s)\bigr),
              \\
            &\qquad
              \min_{\tau \leq \ell < s}
                q_2\bigl(
                  \xi_{x_0}^\alpha(\ell)\bigr)
            \bigr\} \\
            &\geq \min\bigl\{
              r_2\bigl(
                \xi_{x_0}^\alpha(t'')\bigr),
              \\
            &\qquad
              \min_{\tau \leq \ell < t''}
                q_2\bigl(
                  \xi_{x_0}^\alpha(\ell)\bigr)
            \bigr\} \\
            &\geq \lambda - \epsilon.
        \end{align*}
    \end{enumerate}
    Hence, for all $\tau$ with
    $t \leq \tau < t_1$, both $U_1$ and $U_2$
    are at least $\lambda - \epsilon$.
    We extend $\alpha$ with the segment
    $\alpha_{t:t_1-1}$ and transition to
    $\xi_{x_0}^\alpha(t_1\!+\!1)$, where
    \begin{equation*}
        V_{2,\infty}\bigl(
          \xi_{x_0}^\alpha(t_1\!+\!1)\bigr)
          \geq \lambda
            - \textstyle\sum_{i=1}^{2j} \delta_i.
    \end{equation*}

    By symmetry, the same holds when $m=2$.
    Thus, the inductive step holds.

    By induction, at all times $t$,
    \begin{align*}
        U_1\bigl(
          \xi_{x_0}^{{\alpha}_{t:\infty}}\bigr)
          &\geq \lambda - \epsilon, \\
        U_2\bigl(
          \xi_{x_0}^{{\alpha}_{t:\infty}}\bigr)
          &\geq \lambda - \epsilon.
    \end{align*}
    Hence,
    \begin{align*}
        &\Vstar(x_0) \\
        &= \sup_{\alpha}
          \min\bigl\{
            U_1\bigl(
              \xi_{x_0}^{{\alpha}_{0:\infty}}
            \bigr), \\
        &\qquad
            U_2\bigl(
              \xi_{x_0}^{{\alpha}_{0:\infty}}
            \bigr)
          \bigr\} \\
        &\geq \lambda - \epsilon
          = V_{1,\infty}(x_0) - \epsilon.
    \end{align*}
    Since $\epsilon > 0$ was arbitrary,
    $\Vstar(x_0) \geq V_{1,\infty}(x_0)$.
    By symmetry,
    $\Vstar(x_0) \geq V_{2,\infty}(x_0)$.
    This shows \eqref{lem:GUU:Vinf_smaller_v}.
\end{proof}

We are now ready to prove
Lem.~\ref{lem:GUU:fixed_point_convergence}.
\begin{proof}
    The proof follows directly from
    \eqref{lem:GUU:v_smaller_Vinf} and
    \eqref{lem:GUU:Vinf_smaller_v}.
\end{proof}

\newpage
\section{General Decomposition} \label{apx:gen-proof}

Here, we offer an example of greater compositions then the major two presented in the text, in Thms.~\ref{thm:n-ra} and \ref{thm:n-ra-loop}. Namely, we give here a combination of the two cases, and show that it follows from a similar proof strategy.

\begin{tcolorbox}[DefFrame]
    \begin{restatable}{thm}{thmmaster} \label{thm:master}
        For the predicate $$\mathsf{p} := \left( \bigwedge_{i \in \mathcal{I}} \left( \mathsf{q}_i\, \mathsf{U}\, \mathsf{r}_i \right)\right) \land \mathsf{G} \left( \bigwedge_{j \in \mathcal{J}} \left( \mathsf{q}_j\, \mathsf{U}\, \mathsf{r}_j \right)\right) \land \mathsf{G} \mathsf{q} $$
        the corresponding optimal Value satisfies $V^*[\mathsf{p}](x) = V^* \big[\tilde{\mathsf{q}} \,\mathsf{U}\, \tilde{\mathsf{r}}\big](x)$
        where
        \begin{align*}
              \tilde{\mathsf{r}} := \bigvee\nolimits_i \left(\mathsf{r}_i \land \mathsf{v}_{\mathsf{p}_{-i}} \right), \qquad \tilde{\mathsf{q}} := \bigwedge_{k \in \mathcal{I} \times \mathcal{J}} \tilde{\mathsf{q}}_k \land \mathsf{q},
        \end{align*}
        $$\mathsf{p}_{-i} := \left(\bigwedge_{k \in \mathcal{I} \setminus \{i\}} (\mathsf{q}_k\, \mathsf{U}\, \mathsf{r}_k)\right) \land \mathsf{G} \left( \bigwedge_{j \in \mathcal{J}} \left( \mathsf{q}_j\, \mathsf{U}\, \mathsf{r}_j \right)\right) \land \mathsf{Gq}.$$
    \end{restatable}
\end{tcolorbox}

\proof
    The proof simply follows from the same reasoning as in the previous sections, utilizing the established relationships between the various Value functions and their decompositions. Namely, this result follows from a combination of logical rearrangement and then a usage of the algebraic properties of the Bellman equations.

    First, we may have by Lem.~\ref{lem:symbolic:master_GU_U_G} that $\mathsf{p}$ may be rewritten in one of two ways, depending on the remaining index set of Until predicates $\mathcal{J}$. Hence, the proof follows from either case.

    \textit{1. Non-empty $\mathcal{J}$} 

    In this case we have by Lem.~\ref{lem:symbolic:master_GU_U_G}, $\mathsf{p} = \tilde{\mathsf{q}} \, \mathsf{U} \, \tilde{\mathsf{r}}$, where $\tilde{\mathsf{r}}$ is given by,
    $$
    \tilde{\mathsf{r}}_{\IGUSet, \IUSet} \coloneqq \bigvee_{j \in \IUSet} \ru_j \land \Phi_{\IGUSet, \IUSet \setminus \{j \} }.
    $$
    Notably, this case is algebraically equivalent to the previous proofs (e.g. Thm.~\ref{thm:n-ra}), and hence, by Thm.~\ref{lem:equiv}, we have the given result.

    \textit{2. $\mathcal{J}= \emptyset$} 

    In this case we have by Lem.~\ref{lem:symbolic:master_GU_U_G},
    $$
    \mathsf{p} \equiv \mathsf{G}\Bigl( \bigwedge_{i \in \IGUSet} \bigl( \qgu_i \land \qg \bigr)\, \mathsf{U} \, \bigl( \rgu_i \land \qg \bigr) \Bigr)
    $$
    On the other hand, this is a special case of the $N$-$\mathcal{RA}_\ell$ problem, and thus by Thm.~\ref{thm:n-ra-loop}, we may decompose this into $N$ coupled Until-decompositions.
\endproof

\newpage
\section{Policy Synthesis} \label{apx:policy}

In this section, we extend the previous results involving the optimal action sequence $\alpha$ to a state-feedback policy $\pi: \mathcal{X} \to \mathcal{A}$.
For general TL predicates, the synthesis of a policy that matches open-loop action sequence performance requires state-augmentation \cite{sharpless2025dual,meng2025tgpo}.
The nature of temporal logic is to score satisfaction over the entire trajectory.
Hence, to play optimally, the running performance is required.
In \cite{sharpless2025dual}, the authors show that for a reduced set of dual-predicates, the optimal policy may be derived as a function of the augmented-state and each decomposed Value.
Here, we generalize these results to the decomposed Value graph that arises in the decomposition of the general predicates considered in this work.
We show how the optimal policy can be composed from the optimal policies corresponding to each of the nodes in the DVG.

Let $\predA{p}$, $\predB{p'}$, and $q$, $q'$ be arbitrary predicates.
Let $V^*[\cdot]$ be the optimal Value for the given predicates,
and let $\pi^*[\cdot]$ be an optimal policy that achieves $V^*[\cdot]$ given some (potentially augmented) state space.

\subsection{OR}
For $\predA{p} \lor \predB{p'}$, we consider the policy $\hat{\pi}$ defined as
\begin{equation}
    \hat{\pi}[\predA{p} \lor \predB{p'}](x) = \begin{dcases}
            \pi^*[\predA{p}](x) & \text{if } V^*[\predA{p}](x) \geq V^*[\predB{p'}](x), \\
            \pi^*[\predB{p'}](x) & \text{otherwise}.
        \end{dcases}
\end{equation}
This can be derived by noting that
\begin{align}
    \argmax_{\alpha} \rho[\predA{p} \lor \predB{p'}](\xi_x^{\alpha}) 
    &= \argmax_{\alpha} \max\Big\{
        \rho[\predA{p}](\xi_x^{\alpha}),\,
        \rho[\predB{p'}](\xi_x^{\alpha}) \Big\}, \\
    &= \begin{dcases}
        \argmax_{\alpha} \rho[\predA{p}](\xi_x^{\alpha}) & \text{if } V^*[\predA{p}](x) \geq V^*[\predB{p'}](x), \\
        \argmax_{\alpha} \rho[\predB{p'}](\xi_x^{\alpha}) & \text{otherwise}.
    \end{dcases}, \\
    &= \hat{\pi}[\predA{p} \lor \predB{p'}](x).
\end{align}

\subsection{Weak AND}
For $\predA{p} \land \predB{p'}$ where $\forall\alpha,\beta\in \mathcal{A}^\mathbb{N},\,\rho[\predA{p}](\xi_x^\alpha)=\rho[\predA{p}](\xi_x^\beta)$, we consider the policy $\hat{\pi}$ defined as
\begin{equation}
    \hat{\pi}\big[ \predA{p} \land \predB{p'} \big](x) = \pi^*[\predB{p'}](x).
\end{equation}
To derive this, first note that $V^*[\predA{p} \land \predB{p'}](x) = V^* [{\predV{v}[{\predA{p}}]} \land \predB{p'}](x)$ by \textcolor[HTML]{A64D79}{{\texttt{VDR}}}-\ref{vdr:wand}.
Then, we can derive the optimal policy as follows:
\begin{align}
    V^*[\predA{p} \land \predB{p'}](x)
    &= V^* [{\predV{v}[{\predA{p}}]} \land \predB{p'}](x), \\
    &= \max_{\alpha} \rho[\predA{p} \land \predB{p'}](\xi_x^{\alpha}), \\
    &=  \max_{\alpha} \min\{ \rho[\predB{p'}](\xi_x^{\alpha}), \rho[\predA{p}](\xi_x^{\alpha}) \}, \\
    &=  \max_{\alpha} \min\{ \rho[\predB{p'}](\xi_x^{\alpha}), c \}, \quad c = \rho[\predA{p}](\xi_x^\alpha) \text{ for any $\alpha$}, \\
    &=  \max_{\alpha}  \rho[\predB{p'}](\xi_x^{\alpha}),
\end{align}
where the last equality follows from the fact that $c$ is a constant independent of $\alpha$.

\subsection{NEXT}
Constructing a policy for $\mathsf{X} \predA{p}$ is slightly more tricky because
the time at which the predicate $\predA{p}$ is evaluated now matters.
We first consider a time-dependent version of $\hat{\pi}$, defined as
\begin{equation} \label{eq:policy:next}
    \hat{\pi}[\mathsf{X} \predA{p}](x, t) = \begin{dcases}
        \max_{\alpha_0} V^*[\predA{p}](f(x, \alpha_0)), & \text{if } t = 0, \\
        \pi^*[\predA{p}](x), & \text{if } t > 0.
    \end{dcases}
\end{equation}
To derive this, first note that $V^*[\mathsf{X}\predA{p}](x) = V^*[\mathsf{X}{\predV{v}[{\predA{p}}]}](x)$ by \textcolor[HTML]{A64D79}{{\texttt{VDR}}}-\ref{vdr:next}.
Then, we can derive the optimal policy as follows: when $t=0$,
\begin{align}
    V^*[\mathsf{X} \predA{p}](x)
    &= V^*[\mathsf{X}{\predV{v}[{\predA{p}}]}](x), \\
    &= \max_{\alpha} \rho[\mathsf{X} \predA{p}](\xi_x^{\alpha}), \\
    &= \max_{\alpha} \rho[\predA{p}](\xi_x^{\alpha}, 1), \\
    &= \max_{\alpha_0} V^*[\predA{p}](f(x, \alpha_0)).
\end{align}
On the other hand, when $t > 0$,
\begin{equation}
    \max_\alpha \rho[\mathsf{X} \predA{p}](\xi_x^{\alpha}, t)
    = \max_\alpha \rho[\predA{p}](\xi_x^{\alpha}, t+1).
\end{equation}
Hence,
\begin{align}
    \argmax_{\alpha_0} \max_{\alpha_{1:\infty}} \rho[\mathsf{X} \predA{p}](\xi_x^{\alpha})
    &= \argmax_{\alpha_0} \rho[\predA{p}](\xi_x^{\alpha}, 1), \\
    &= \hat{\pi}[\mathsf{X} \predA{p}](x, 0),
\end{align}
and, for $t > 0$,
\begin{align}
    \argmax_{\alpha_t} \max_{\alpha_{0:t-1}, \alpha_{t+1:\infty}} \rho[\mathsf{X} \predA{p}](\xi_x^{\alpha})
    &= \argmax_{\alpha_t} \rho[\predA{p}](\xi_x^{\alpha}, t+1), \\
    &= \pi^*[\predA{p}](x) = \hat{\pi}[\mathsf{X} \predA{p}](x, t).
\end{align}
Since we only need to remember a single bit of information to know whether $t=0$ or $t>0$,
it is sufficient to augment the state space to remember whether we are ``at the NEXT node'' or ``at the child of the NEXT node''.

\subsection{Right Until}
We now consider $\predA{p} \, \mathsf{U} \, \predB{p'}$ where $\forall\alpha,\beta\in \mathcal{A}^\mathbb{N},\,\rho[\predA{p}](\xi_x^\alpha)=\rho[\predA{p}](\xi_x^\beta)$.
This is similar to the NEXT case, where the timing becomes important. 
Let $t^*$ be the first time step at which the Value of $\predA{p} \, \mathsf{U} \, \predB{p'}$ is achieved, i.e.,
\begin{equation}
    V^*[\predA{p} \, \mathsf{U} \, \predB{p'}](x)
    = \max_\alpha \min\Big\{
        \rho[\predB{p'}](\xi_x^{\alpha}, t^*),\,
        \min_{\kappa \in [0, t^*]} \rho[\predA{p}](\xi_x^{\alpha}, \kappa)
    \Big\}.
\end{equation}
Namely, note that $t^*$ is also the first timestep where
\begin{equation} \label{eq:policy:until:comparison}
    \min\{ \rho[\predB{p'}](\xi_x^{\alpha}, t^*), \rho[\predA{p}](\xi_x^{\alpha}, t^*) \}
    \geq V^*[\predA{p} \, \mathsf{U} \, \predB{p'}](x)
    = V^*[\predA{p} \, \mathsf{U} \, {\predV{v}[{\predB{p'}}]}](x).
\end{equation}
We consider the policy $\hat{\pi}$ defined as
\begin{equation} \label{eq:policy:until}
    \hat{\pi}[\predA{p} \,\mathsf{U}\, \predB{p'}](x) = \begin{dcases}
            \pi^*\big[ \predA{p} \, \mathsf{U} \, \predV{v}[ \predB{p'} ] \big](x) & \text{if } V^*[\predA{p} \land \predB{p'}](x) < V^*[\predA{p} \, \mathsf{U} \, \predB{p'}](x), \\
            \pi^*[\predA{p} \land \predB{p'}](x) & \text{if } V^*[\predA{p} \land \predB{p'}](x) \geq V^*[\predA{p} \, \mathsf{U} \, \predB{p'}](x).
        \end{dcases}
\end{equation}
To derive this, we again first consider a time-dependent version of $\hat{\pi}$, defined as 
\begin{equation}
    \hat{\pi}[\predA{p} \,\mathsf{U}\, \predB{p'}](x, t) = \begin{dcases}
            \pi^*\big[ \predA{p} \, \mathsf{U} \, \predV{v}[ \predB{p'} ] \big](x) & \text{if } t < t^*, \\
            \pi^*[\predA{p} \land \predB{p'}](x) & \text{if } t \geq t^*.
        \end{dcases}
\end{equation}
The optimal actions for both $\predA{p} \, \mathsf{U} \, \predB{p'}$ and $\predA{p} \, \mathsf{U} \, {\predV{v}[{\predB{p'}}]}$ are the same for $t < t^*$ \cite{so2026value}.
On the other hand,
\begin{align}
    \argmax_{\alpha_{t^*:\infty}} \max_{\alpha_{0:t^*-1}} \rho[\predA{p} \, \mathsf{U} \, \predB{p'}](\xi_x^{\alpha})
    &= \argmax_{\alpha_{t^*:\infty}} \max_{\alpha_{0:t^*-1}} \min\Big\{
        \rho[\predB{p'}](\xi_x^{\alpha}, t^*),\,
        \min_{\kappa \in [0, t^*]} \rho[\predA{p}](\xi_x^{\alpha}, \kappa)
    \Big\}, \\
    &= \argmax_{\alpha_{t^*:\infty}} \min\Big\{
        \rho[\predB{p'}](\xi_x^{\alpha}, t^*),\,
        \rho[\predA{p}](\xi_x^{\alpha}, t^*) \}.
\end{align}
Again, similar to the case of NEXT, we only need to remember a single bit of information to know whether $t < t^*$ or $t \geq t^*$.
Moreover, by \eqref{eq:policy:until:comparison}, it is sufficient to augment the state space to remember whether we are ``at the UNTIL node'' or ``at the child of the UNTIL node''.

\subsection{Practical Implementation of Policy Synthesis via DVG Traversal}
As shown above, the NEXT and right-UNTIL operators require augmenting the state-space to remember whether we are at the node or its child.
To generalize this to repeated compositions of these operators (which would require remembering multiple bits of information), we augment the state space to remember which node of the DVG is currently being evaluated, and move to the child node as dictated by the piecewise definitions of the optimal policy in either \eqref{eq:policy:next} or \eqref{eq:policy:until}.
We provide a pseudocode implementation of this in \Cref{alg:policy_composition}, where we use \textit{purely propositional} (PP) to denote a node / subformula in the DVG that does not contain any temporal operators.
To handle Recurrent Reach-Avoid specifications with $\mathsf{GU}$,
the DVG's rewrites it into the form \eqref{def:predloop}, which can be handled by the above policy composition rules.

\begin{algorithm}
\caption{Policy Composition}
\label{alg:policy_composition}
\begin{algorithmic}[1]
\Require DVG root $r$, initial state $x_0$, dynamics $f$, Values $V^*[\cdot]$, policies $\pi^*[\cdot]$
\State $n \gets r$
\For{$t = 0,1,2,\ldots$}
    \State $a_t \gets \bot$
    \State $n_{\mathrm{post}} \gets n$
    \While{$a_t = \bot$}
        \If{$n = \bigvee_i \predA{p}_i$}
            \State $i^* \gets \argmax_i V^*[\predA{p}_i](x_t)$
            \State $n \gets \predA{p}_{i^*}$
        \ElsIf{$n = \predA{p} \land \predB{p'}$ with $\predA{p} \in \mathrm{PP}$}
            \State $n \gets \predB{p'}$
        \ElsIf{$n = \predA{p} \, \mathsf{U} \, \predB{p'}$ with $\predA{p} \in \mathrm{PP}$ }
            \If{$V^*[\predA{p} \land \predB{p'}](x_t) \geq V^*[n](x_t)$}
                \State $n \gets \predB{p'}$ \Comment{trigger reached; switch instantaneously}
            \Else
                \State $a_t \gets \pi^*[n](x_t)$
                \State $n_{\mathrm{post}} \gets n$
            \EndIf
        \ElsIf{$n = \mathsf{X} \predA{p}$}
            \State $a_t \in \argmax_{a \in \mathcal{A}} V^*[\predA{p}]\big(f(x_t,a)\big)$
            \State $n_{\mathrm{post}} \gets n_c$ \Comment{switch after one step}
        \Else
            \State $a_t \gets \pi^*[n](x_t)$
            \State $n_{\mathrm{post}} \gets n$
        \EndIf
    \EndWhile
    \State apply $a_t$ and observe $x_{t+1} = f(x_t,a_t)$
    \State $n \gets n_{\mathrm{post}}$
\EndFor
\end{algorithmic}
\end{algorithm}

\newpage
\section{\texttt{valtr} Details} \label{apx:valtr}

In this section, we describe our tool \texttt{valtr}, that (1.) converts temporal logic predicates into a suitable form for decomposition, and (2.) applies the main results recursively to generate the decomposed Value graph. 

To decompose the Value for a user-input predicate, the predicate must first be organized into the form given in Thm.~\ref{thm:master}. This is accomplished by lexing the temporal logic string into relevant tokens, such as atomic propositions and temporal operators, which may then be parsed to generate an abstract syntax tree (AST), which is thus a type of TL Tree (TLT). Over this AST, several passes are made to rearrange the tree into an intermediate representation. This rearrangement is accomplished by first applying well-known logical equivalences and then followed by cleaning (e.g. aggregating redundancies). The ultimate product is a TLT with structure that is amenable to the decompositional results.

To apply the main results recursively and generate the decomposed Value graph, we traverse the TLT and for each node, we apply the decomposition procedure outlined in Thm.~\ref{thm:master}. This involves identifying the relevant substructures, including constants (atomic predicates), negations, minima, maxima, and nodes which represent Value functions. After final cleaning passes, the resulting decomposed Value graph (DVG) is outputted, defining a topological order of nodes, which may be queried to assess a trajectory as well as identify dependencies, and thus suffices for dynamic programming and \texttt{VDPPO}.

\newpage
\section{\texttt{VDPPO} Details} \label{apx:vdppo}

In this section we further describe our algorithm, \texttt{VDPPO}. \texttt{VDPPO} is a specialized form of \texttt{PPO} \cite{schulman2015high}, designed to leverage the decomposed Value graph (DVG). We outline the two augmentations that distinguish it from standard PPO here.

\textbf{1. The advantage and targets are solved with $\mathcal{A}$, $\mathcal{RA}$, and $\mathcal{RA}_\ell$ Bellman eqns. and bootstrapped Values.} As given by the main results, the Bellman Value for a complex TL predicate may be decomposed into a graph of Bellman Values, connected by these atomic BEs. Hence, the Value at each node in the DVG may be approximated in the limit of discounting by the appropriate BE as a function of its dependencies: its decomposed sub-Values and the relevant predicates. To avoid topographically sequential approximation, we use the current Value approximations of the critic to solve these updates. This is denoted by the feedback loop in Fig.~\ref{fig:algplot0}.

\textbf{2. Nodes are embedded, allowing for a unified representation for each actor and critic} We hypothesize that different Values in the DVG may share some similarity, implying the policies do as well, and thus may be jointly approximated by a single representation. Namely, we augment the states with a current Value node and - with a one-hot encoding - condition the MLP for each actor and critic on mixed-node batches. We validate this hypothesis and design choice in the ablations in Sec.~\ref{apx:ablations}, demonstrating this yields equivalent performance while vastly improving the scaling ability compared to previous approaches \cite{sharpless2025dual}.

Additionally, for live roll-outs and evaluation, we define the policy such that upon satisfying the trigger condition given in Sec.\ref{apx:policy}, the current Value node switches to the triggered node in the current augmented state.

\newpage
\section{Environments} \label{apx:sim}

We give here additional details on the environments tested in this work.
The reader may refer to the main text for graphics and specs.
We will publish all code after the anonymous stage of review is complete.

All atomic propositions are defined to be within $[-1, 1]$.
By default, they are \textbf{sparse} and are $\{-1, 1\}$ valued, i.e., take the form
\begin{equation}
    r(x) = \begin{dcases}
        1, \text{if satisfied},
        -1, \text{otherwise}.
    \end{dcases}
\end{equation}
Unless specifically noted, atomic propositions are taken to be sparse.

\textbf{\texttt{DoubleInt}}: The \texttt{DoubleInt} env is defined by up to $N$ agents with 2-dimensional double integrator dynamics and velocity-tracking control. Namely, for each agent, the discrete action sets a desired velocity which is then tracked by a proportional controller in the acceleration (with $k_p=1$). The possible discrete actions correspond to $\pm1$ per dimension, multiplied by the max acceleration. Velocity and acceleration limits are set per-agent. In the three sub-envs, \texttt{Breadth}, \texttt{Depth}, \texttt{Agents (dim.)}, we vary the number of targets to reach (any order), the number of targets to reach sequentially, and the number of agents and number of targets to reach (any order) respectively. In all cases, we define a set of obstacles for which all specifications involve avoid predicates.
The atomic predicates for distance to targets is dense, linearly scaling with distance and clipped to within $[-1, 1]$.

\textbf{\texttt{Herding}}: The \texttt{Herding} env is an augmentation of the \texttt{DoubleInt} env, where we have a team of two agents (the herders) and multiple sheep agents (the herd). The sheep agents are defined by their own fixed policy which samples an action which maximizes the weighted soft-min of their distance to the herders, the walls and each another. The herders are defined such that one is twice as fast as the other, while the herders move at a maximum speed equivalent to the slow herder. A narrow gap divides the herders from the herd initially, as well as the target location of the herd and their initial position. The goal of the task is defined by moving the herd through the narrow passage toward the target region on the otherside and contain them there, while avoiding obstacles and collision. This additionally two intermediate goals to have the herd before the passage, and then to have the herd after the passage, which must be achieved sequentially. The full specification is given in the main text.
The atomic predicates for the distance to the final herding target and the collision between herd and herders is dense, linearly scaling with distance and clipped to within $[-1, 1]$.

\textbf{\texttt{Delivery}}: The \texttt{Delivery} env is an augmentation of the \texttt{DoubleInt} env, where we have a team of three agents -- two small, fast agents (the delivery robots) and one, big slow agent (the resupply truck) -- and randomly spawning targets (delivery locations). The goal of the task is for the agents to recurrently reach the target locations and then recurrently visit the resupply truck. After a delivery target is reached by the corresponding agent, the location jumps to a new random location. Additionally, the domain is defined with the same obstacles used in the \texttt{DoubleInt} env, and the team must avoid collision with the obstacles and one another, despite both needing to resupply at the mobile agent. All agents are mobile and hence the truck agent may dynamically adjust its location to suit the current positions. Note, this simulated env differs from the hardware version, which includes a different obstacle layout as well as an additional aerial obstacle (no fly zone).
The atomic predicates for distance to targets is dense, linearly scaling with distance and clipped to within $[-1, 1]$.

\textbf{\texttt{Manipulator}}: The \texttt{Manipulator} env is taken from \cite{ogbench_park2025}, and involves a manipulator which must grasp and interact with objects in the environment. The specification for this task is to place the cube inside the drawer and eventually always have the drawer closed. Additional objects exist in the environment but have no relevance to task completion.
All atomic predicates are sparse.


\newpage
\section{Baselines} \label{apx:baselines}

In this section we discuss the baselines employed in this work.

\textbf{LCRL}: This baseline \cite{hasanbeig2022lcrl} is a deep RL method that augments the MDP with an automata for learning TL solutions. Specifically, an actor-critic variation of PPO is designed such that they are conditioned on the automaton and the current state of an augmented trajectory. As this is just another variation of PPO, we employ the same parameter set as used in \texttt{VDPPO} for a fair comparison.

\textbf{TL-MPPI}: This baseline is an extension of Model Predictive Path Integral (MPPI) \cite{williams2016aggressive} to tackle TL problems \cite{halder2025trajectory}, which we denote TL-MPPI. Namely, this method plans a trajectory based on MPPI sample-based optimization of the TL robustness metric. The method in the work does not function adaptively as the controller has no memory without state-augmentation or automaton, however, we employ it as a trajectory optimization method which the agent then tracks. The parameters that worked best in the given environments included: $1000$ samples per step, a horizon of $100$ steps, $20$ iterations per step, an initial standard deviation of $50$, $\lambda=1$, and an iteration temperature (shrink) parameter of $0.6$.

\newpage
\section{Ablations} \label{apx:ablations}

Here, we provide additional ablation experiments to analyze the design of our algorithm, \texttt{VDPPO}. In \cite{sharpless2025dual}, authors similarly derived decompositional Value results, although for a greatly reduced set of predicates, and then faced the practical question of how to employ these results to learn the critics (Value estimates) effectively, deciding to use a different actor and critic for each decomposition. While this performed well for the dual-specifications that were considered, this approach scales poorly to tasks with complex logic, as the number of required actors and critics can grow combinatorially (see Thm.\ref{thm:n-ra}).

Moreover, while Values can vary significantly for different rewards and specifications, in many practical cases, tasks often involve different sub-tasks which themselves differ only by translation (e.g. identical configuration goals in different locations), order (e.g. iteratively unlock doors with keys) or other simple transformation or symmetry. 
Under certain variations, the resulting Bellman Value may indeed differ only by the same transformation. In such cases, a partial consolidation of the representations may accomplish sufficient approximation while greatly reducing the learning challenge.

In \texttt{VDPPO}, we employed this idea, by embedding all Values into a shared space with the one-hot encoding to allow the actor and critic to to each use a shared MLP trunk (see Section \ref{apx:vdppo} for details). To analyze the importance of this design choice, we compare against a version of \texttt{VDPPO} where each critic and actor has its own separate MLP trunk (i.e. no shared parameters). Moreover, we scan this comparison over an increasing range in the number of layers in the shared trunk (or each independent body, when not shared), to analyze the importance of the depth of the shared representation. The results are plotted in Fig.~\ref{fig:ablations_share}.

From a performance-only perspective, we find that sharing parameters for the value function alone erodes success rate while sharing parameters for the actor boosts success rate, and when combined, we observe performance that is nearly identical to performance without sharing. This result is inspiring as the shared architectures train nearly $N$-times faster than the standard approach employed in \cite{sharpless2025dual}, where $N$ is the quantity of decompositions.

\begin{figure}[h]
    \centering
    \includegraphics[width=0.6\linewidth]{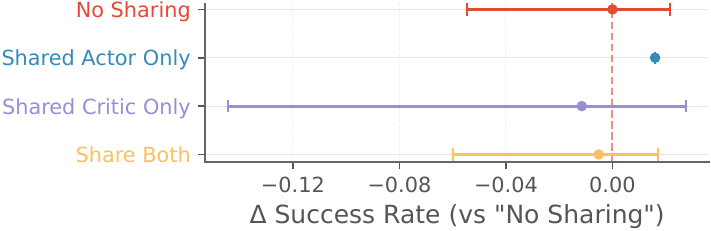}
    \caption{\textbf{Effect of parameter sharing.} Sharing parameters for the actor only improves performance while reduce the variance. }
    \label{fig:ablations_share}
\end{figure}


\newpage
\section{Hardware} \label{apx:hardware}

In the hardware experiments, we evaluate \texttt{VDPPO} performance in the Herding and Delivery tasks.
In both tasks, the state position is reported by HTC Vive base stations in communication with the an attached Lighthouse deck to each Crazyflie. The Go2 quadruped's location is integrated into the same framework by attaching a propeller-less Crazyflie to its chassis, which transmits its position data to a single computer. 
The state of each agent is concatenated to form the full state used by the \texttt{VDPPO} policy, which is inferred on the local CPU of the coordinating laptop. 
The output action velocity commands are broadcasted to each agent's onboard controller, which tracks the transmitted velocity setpoint. 


\newpage


\end{document}